\documentclass[11pt]{article}
\usepackage[cmex10]{amsmath}
\usepackage{amsthm}
\usepackage{amssymb,amsfonts,textcomp}
\usepackage{latexsym}
\usepackage[nospace, compress]{cite}
\usepackage[noend, boxed, noline]{algorithm2e}

\usepackage{graphicx} 
\usepackage{latexsym}
\usepackage{multirow}
\usepackage{rotating}
\usepackage{url}
\usepackage{color}

\newcommand{\ideal}{{{\mathrm{ideal}}}}

\newcommand{\fixlist}{\addtolength{\itemsep}{-5pt}}

\newcommand{\capacity}{{{{\mathrm{cap}}}}}

\newcommand{\pos}{{{{\mathrm{pos}}}}}
\newcommand{\bordainc}{{{{\mathrm{B, inc}}}}}
\newcommand{\bordadec}{{{{\mathrm{B, dec}}}}}
\newcommand{\OPT}{{{{\mathrm{OPT}}}}}
\newcommand{\opt}{{{{\mathrm{opt}}}}}
\newcommand{\sat}{{{{\mathrm{sat}}}}}
\newcommand{\E}{\mathop{\mathbb E}}
\newcommand{\w}{{{{\mathrm{w}}}}}

\title{Achieving Fully Proportional Representation: Approximability
  Results\footnote{This paper combines and extends results presented
    at IJCAI-2013 (paper titled ``Fully Proportional Representation as
    Resource Allocation: Approximability Results''; the paper
    contained most of our theoretical results) and at AAMAS-2012
    (paper titled ``Achieving Fully Proportional Representation is
    Easy in Practice''; the paper contained most of our experimental
    results).}}

\date{}

\author{
Piotr Skowron\\
       {University of Warsaw}\\
       {Warsaw, Poland}\\
\and
Piotr Faliszewski\\
       {AGH University}\\
       {Krakow, Poland}\\
\and
Arkadii Slinko\\
       {University of Auckland}\\
       {Auckland, New Zealand}\\
}

\newtheorem{theorem}{Theorem}
\newtheorem{definition}{Definition}
\newtheorem{proposition}[theorem]{Proposition}

\newtheorem{lemma}[theorem]{Lemma}

\newcommand{\np}{{\mathrm{NP}}}
\newcommand{\fpt}{{\mathrm{FPT}}}
\newcommand{\wone}{{\mathrm{W[1]}}}
\newcommand{\wtwo}{{\mathrm{W[2]}}}
\newcommand{\p}{{\mathrm{P}}}

\newcommand{\naturals}{{{\mathbb{N}}}}

\newcommand{\calA}{{{\mathcal{A}}}}
\newcommand{\calF}{{{\mathcal{F}}}}
\newcommand{\calS}{{{\mathcal{S}}}}

\begin{document}

\maketitle

\begin{abstract}
  We study the complexity of (approximate) winner determination under
  the Monroe and Chamberlin--Courant multiwinner voting rules,
  which determine the set of representatives by optimizing the total (dis)satisfaction of the voters with their representatives. 
 The total (dis)satisfaction  is calculated either as the sum of individual (dis)satisfactions  (the
  utilitarian case) or as the (dis)satisfaction of the worst off voter
  (the egalitarian case). We provide good approximation algorithms for
  the satisfaction-based utilitarian versions  of the Monroe and Chamberlin--Courant 
  rules, and inapproximability results for the dissatisfaction-based utilitarian versions of them and also for all egalitarian cases.
  Our algorithms are applicable and particularly appealing when voters submit truncated
  ballots.
  We provide experimental evaluation of the algorithms both on
  real-life preference-aggregation data and on synthetic data. These experiments show
  that our simple and fast algorithms can in many cases find  near-perfect solutions.
\end{abstract}

\section{Introduction}\label{sec::introduction}
We study the complexity of (approximate) winner determination under
the Monroe~\cite{monroeElection} and
Chamberlin--Courant~\cite{ccElection} multiwinner voting rules,
which aim at selecting a group of candidates that best represent
the voters. Multiwinner elections are important both for human
societies (e.g., in indirect democracies for electing committees of
representatives like parliaments) and for software multiagent systems (e.g., for
recommendation systems~\cite{budgetSocialChoice}), and thus it is
important to have good multiwinner rules and good algorithms for them.
The Monroe and Chamberlin--Courant rules are particularly appealing
because they create an explicit (and, in some sense, optimal)
connection between the elected committee members and the voters; each
voter knows his or her representative and each committee member knows
to whom he or she is accountable. In the context of recommendation systems this means that every selected item is personalized, i.e., recommended to a particular user. Moreover, the Monroe rule ensures the proportionality of the representation.
 We assume that $m$ candidates participate in the election and that the society consists of $n$ voters, who each rank the
candidates, expressing their preferences about who they would like to see as their representative.

When choosing a $K$-member committee, the Monroe and
Chamberlin--Courant rules work as follows.    For each voter they assign a single candidate as their representative, 
respecting the following rules: 
\begin{itemize}
\item[(a)] 
altogether exactly $K$ candidates are assigned to the voters. For the
Monroe rule, each candidate is assigned either to about
$\frac{n}{K}$ voters or to none; for the Chamberlin--Courant rule there
is no such restriction and each committee member might be representing a different number
  of voters. The committee should take this into account in its
  operation, i.e., by means of weighted voting. 
  \item[(b)] the candidates are selected and assigned to the
voters optimally minimizing the total (societal) dissatisfaction or maximizing the total (societal) satisfaction.
\end{itemize}
The total (dis)satisfaction is calculated on the basis of individual (dis)satisfactions. 
We assume that there is a function $\alpha \colon \naturals \to \naturals$
such that $\alpha(i)$ measures how well a voter is represented by the
candidate that this voter ranks as $i$'th best. The
function $\alpha$ is the same for each voter.
We can view $\alpha$ either as a \emph{satisfaction function} (then it
should be a decreasing one) or as a \emph{dissatisfaction function}
(then it should be an increasing one).  For example, it is typical to
use the Borda count scoring function whose $m$-candidate
dissatisfaction variant is defined as $\alpha^m_{\bordainc} = i-1$, and
whose satisfaction variant is $\alpha^m_{\bordadec} = m-i$.
In the utilitarian variants of the rules, the assignment should
maximize (minimize) the total satisfaction (dissatisfaction) calculated as the sum of the voters' individual satisfactions
(dissatisfactions) with their representatives. In the egalitarian
variants, the assignment should maximize (minimize) the total satisfaction
(dissatisfaction) calculated as the  satisfaction
(dissatisfaction) of the worst-off voter.

The Monroe and Chamberlin--Courant rules create a useful
connection between the voters and their representatives  that
makes it possible to achieve both candidates' accountability to the voters, and proportional representation of
voters' views. Among common voting rules, the Monroe and Chamberlin--Courant
rules seem to be unique in having \emph{both} the accountability and
the proportionality properties simultaneously. 
%
%
For example, First Past the Post system (where the voters are
partitioned into districts with a separate single-winner Plurality
election in each) can give very disproportionate results (forcing some
of the voters to be represented by candidates they dislike). On the
other side of the spectrum are the party-list systems, which achieve
perfect proportionality.  In those systems the voters vote for the
parties, based on these votes each party receives some number of seats
in the parliament, and then each party distributes the seats among its
members (usually following a publicly available list of the party's
candidates).  This makes the elected candidates feel more accountable
to apparatchiks of their parties than to the voters.  Somewhere
between the First Past the Post system and the party-list systems, we
have the single transferable vote rule (STV), but for STV it is
difficult to tell which candidate represents which voters.

Unfortunately, the Monroe and
Chamberlin--Courant rules have one crucial drawback that makes them
impractical. It is $\np$-hard to tell who the winners are!
Specifically, $\np$-hardness of winner determination under the Monroe
and Chamberlin--Courant rules was shown by Procaccia et
al.~\cite{complexityProportionalRepr} and by Lu and
Boutilier~\cite{budgetSocialChoice}.
Worse yet, the hardness holds even if various natural parameters of
the election are small~\cite{fullyProportionalRepr}.  Rare easy cases
include those, where the committee to be elected is small, or we
consider the Chamberlin--Courant rule and the
voters have single-peaked~\cite{fullyProportionalRepr} or
single-crossing preferences~\cite{sko-yu-fal-elk:c:single-crossing-monroe-cc}. 
%
%
%

Lu and Boutilier~\cite{budgetSocialChoice} proposed to use approximation
algorithms
and have given the first such algorithm for the Chamberlin--Courant
system. Their procedure outputs an assignment that achieves no less
than $1 - \frac{1}{e} \approx 0.63$ fraction of the optimal voter
satisfaction.  
However, the approximation ratio $0.63$ here means that it is possible that, on  average, each agent is
represented by a candidate that this agent prefers to only about
63\% of the candidates, even if there is a perfect solution that
assigns each agent to their most preferred candidate. Such issues, however,
would not occurr if we had a constant-factor approximation algorithm
minimizing the total dissatisfaction. Indeed, if a perfect solution exists, 
then the optimal dissatisfaction is zero 
and a constant-factor approximation algorithm  must also output this perfect solution.




The use of approximation algorithms in real-life applications requires some discussion. 
For example, their use is naturally justified in the context of recommendation systems. Here the strive for optimality is not crucial since a good but not optimal recommendation still has useful information and nobody would object if we replaced the exact recommendation with an approximate one (given that the exact one is hard to calculate). For example, Amazon.com may recommend you a book on gardening which may not be the best book for you on this topic, but still full of useful advice. For such situations, Herbert Simon \cite{sim:j:satisficing} used the term `satisficing,' instead of optimizing, to explain the behavior of decision makers under circumstances in which an optimal solution cannot be easily determined.  On page 129 he wrote: ``Evidently, organisms adapt well enough to ÔsatisficeÕ; they do not, in general, `optimize'.''  Effectively, what Simon says is that the use of approximation algorithms fits well with the human nature. 

Still, the use of approximation algorithms in elections requires some care. It is conceivable that the electoral commission finds an allocation of voters to candidates with a certain value of (dis)satisfaction and one of the parties participating in the election finds an allocation with a better value. This can lead to a political deadlock. There are two ways of avoiding this. Firstly, an approximation algorithm can be fixed by law. In such a case, it becomes an acting voting rule and a new way to measure fairness in the society. Secondly, an electoral commission may calculate the allocation, but also publish the raw data and issue a call for submissions. If,  within the period specified by law, nobody can produce a better allocation, then the committee goes ahead and announces the result. If someone produces a better allocation, then the electoral commission uses the latter one. 

 The use of approximation algorithms is even more natural in elections with partial ballots. Indeed, even if we use an exact algorithm to calculate the winners, the results will be approximate anyway since the voters provide us with approximations of their real preferences and not with their exact preferences. 

\subsection{Our Results}\label{sec:our}

In this paper we focus on approximation algorithms for
winner determination under the Monroe and Chamberlin--Courant rules. Our
first goal is to seek algorithms that find assignments for which the 
dissatisfaction of voters is within a fixed bound of the optimal one.
Unfortunately, we have shown that under standard complexity-theoretic
assumptions such algorithms do not exist.  Nonetheless, we found good algorithms that maximize voter's satisfaction.  Specifically, we have obtained the following results:
\begin{enumerate}
\fixlist
\item The Monroe and Chamberlin--Courant rules are hard to
  approximate up to any constant factor for the dissatisfaction-based
  cases (both utilitarian and egalitarian ones; see
  Theorems~\ref{theorem:noApprox1},~\ref{theorem:noApprox2},~\ref{theorem:noApprox3}~and~\ref{theorem:noApprox4})
  and for the satisfaction-based egalitarian cases (see
  Theorems~\ref{theorem:noApprox5} and~\ref{theorem:noApprox6}).

\item For the satisfaction-based utilitarian framework we show the following.  For the Monroe rule with the Borda
  scoring function we give a  $(0.715-\epsilon)$-approximation algorithm (often, the ratio is much
  better; see Section~\ref{sec:algorithms}). In case of an arbitrary positional scoring
  function we give a ($1 - \frac{1}{e}$)-approximation algorithm (Theorem~\ref{thm:gmMonroe}).
  For the Chamberlin--Courant rule with the Borda scoring function
we give a polynomial-time approximation scheme
  (that is, for each $\epsilon$, $0 < \epsilon < 1$, we have a
  polynomial-time $(1-\epsilon)$-approximation algorithm; see
  Theorem~\ref{theorem:ptas}).

\item We provide empirical evaluation of our algorithms for the
  satisfaction-based utilitarian framework, both on synthetic  and 
  real-life data. This evaluation shows that in practice our best
  algorithms achieve at least $0.9$ approximation ratios, and even
  better results are typical (see Section~\ref{sec:experiments}).

\item We show that our algorithms work very well in the setting where
  voters do not necessarily rank all the candidates, but only provide the so-called
  truncated ballots, in which they rank  several most preferred
  candidates (usually at least three). We provide theoretical guarantees on the performance of
  our algorithms (Propositions~\ref{prop:monTruncated}
  and~\ref{lemma:greedyCCTruncated}) as well as empirical evaluation (see
  Section~\ref{sec:truncated}).

\end{enumerate}

Our results show that, as long as one is willing to accept approximate
solutions, it is possible to use the utilitarian variants of the Monroe
and Chamberlin--Courant rules in practice. This view is justified
both from the theoretical and from the empirical point of view. Due to
our negative results, we did not perform empirical evaluation for the
egalitarian variants of the rules, but we believe that this is an
interesting future research direction.

\subsection{Related Work}\label{sec:related}

A large number of papers are related to our research in
terms of methodology (the study of computational complexity and
approximation algorithms for winner determination under various
$\np$-hard election rules), in terms of perspective and motivation
(e.g., due to the resource allocation view of Monroe and
Chamberlin--Courant rules that we take), and in terms of formal
similarity (e.g., winner determination under the Chamberlin--Courant
rule can be seen as a form of the facility location problem). Below we
review this related literature.

There are several single-winner voting rules for which winner
determination is known to be $\np$-hard. These rules include, for
example, Dodgson's
rule~\cite{bar-tov-tri:j:who-won,hem-hem-rot:j:dodgson,bet-guo-nie:j:dodgson-parametrized},
Young's
rule~\cite{rot-spa-vog:j:young,bet-guo-nie:j:dodgson-parametrized},
and Kemeny's
rule~\cite{bar-tov-tri:j:who-won,hem-spa-vog:j:kemeny,bet-fel-guo-nie-ros:j:fpt-kemeny-aaim}.
For the single-transferable
vote rule (STV), the winner determination problem becomes $\np$-hard if we use the
so-called parallel-universes tie-breaking~\cite{con-rog-xia:c:mle}.
Many of these hardness results hold even in the sense of parameterized
complexity theory (however, there also is a number of fixed-parameter
tractability results; see the references above for details).

These hardness results motivated the search for approximation
algorithms. There are now very good approximation algorithms for
Kemeny's rule~\cite{ail-cha-new:j:kemeny-approx,cop-fla-rud:j:kemeny-approx,ken-sch:c:kemeny-few-errors}
and for Dodgson's
rule~\cite{mcc-pri-sli:j:dodgson,hem-hom:j:dodgson-greedy,car-cov-fel-hom-kak-kar-pro-ros:j:dodgson,fal-hem-hem:j:multimode,car-kak-kar-pro:c:dodgson-acceptable}.
In both cases the results are, in essence, optimal. For Kemeny's rule
there is a polynomial-time approximation
scheme~\cite{ken-sch:c:kemeny-few-errors} and for Dodgson's rule the
achieved approximation ratio is optimal under standard
complexity-theoretic
assumptions~\cite{car-cov-fel-hom-kak-kar-pro-ros:j:dodgson}
(unfortunately, the approximation ratio is not constant but depends
logarithmically on the number of candidates). On the other hand, for
Young's rule it is known that no good approximation algorithms
exist~\cite{car-cov-fel-hom-kak-kar-pro-ros:j:dodgson}.

The work of Caragiannis et
al.~\cite{car-kak-kar-pro:c:dodgson-acceptable} and of Faliszewski et
al.~\cite{fal-hem-hem:j:multimode} on approximate winner determination
for Dodgson's rule is particularly interesting from our perspective. In
the former, the authors advocate treating approximation algorithms for
Dodgson's rule  as voting rules in their own right and design them to have
desirable properties.  
In the latter, the authors show that a
well-established voting rule (so-called Maximin rule) is a reasonable
(though not optimal) approximation of Dodgson's rule. This perspective
is important for anyone interested in using approximation algorithms
for winner determination in elections (as might be the case for our
algorithms for the Monroe and Chamberlin--Courant rules).

The hardness of the winner determination problem for the Monroe and
Chamberlin--Courant rules have been considered in several papers.
Procaccia, Rosenschein and Zohar~\cite{complexityProportionalRepr} were the first to show the hardness of these two rules for the case of a particular approval-style dissatisfaction function.  Their results were
complemented by Lu and Boutilier~\cite{budgetSocialChoice}, Betzler,
Slinko and Uhlmann~\cite{fullyProportionalRepr}, Yu, Chan, and
Elkind~\cite{yu-cha-elk:c:cc-sp-trees}, Skowron et
al.~\cite{sko-yu-fal-elk:c:single-crossing-monroe-cc}, and Skowron and
Faliszewski~\cite{sko-fal:t:max-cover}. These are showing the hardness in case of the Borda
dissatisfaction function,  obtain results on parameterized hardness
of the two rules, and results on hardness (or easiness)
for the cases where the profiles are single-peaked or single-crossing.
Further, Lu and Boutilier~\cite{budgetSocialChoice} initiated the
study of approximability for the Chamberlin--Courant rule (and were the
first to use satisfaction-based framework). Specifically, they gave the
$(1-\frac{1}{e})$-approximation algorithm for the Chamberlin--Courant
rule. The motivation of Lu and Boutilier was coming from the point of
view of recommendation systems and, in that sense, our view of the
rules is quite similar to theirs.

In this paper we take the view that the Monroe and Chamberlin--Courant rules are 
special cases of the following resource allocation problem. The
alternatives are shareable resources, each with a certain capacity defined as the
maximal number of agents that may share this resource. Each agent has preferences
over the resources and is interested in getting exactly one. The goal
is to select a predetermined number $K$ of resources and to find an optimal allocation of these resources (see
Section~\ref{sec:prelim} for details). This provides a unified framework for the two rules and
reveals the connection of proportional representation problem to other resource allocation problems.
In particular, it closely resembles multi-unit resource allocation
with single-unit demand~\cite[Chapter
11]{ley-sho:b:multiagent-systems} (see also the work of Chevaleyre et
al.~\cite{Chevaleyre06issuesin} for a survey of the most
fundamental issues in the multiagent resource allocation theory) and
resource allocation with sharable indivisible
goods~\cite{Chevaleyre06issuesin,AiriauEndrissAAMAS2010}.
Below, we point out other connections of the Monroe and
Chamberlin--Courant rules to several other problems. 

\begin{description}

\item[Facility Location Problems.]  In the facility location problem,
  there are $n$ customers located in some area and an authority, say a
  city council, that wants to establish a fixed number $k$ of
  facilities to serve those customers.  Customers incur certain costs
  (say transportation costs) of using the facilities. Further, setting
  up a facility costs as well (and this cost may depend on the
  facility's location).  The problem is to find $k$ locations for the
  facilities that would minimize the total (societal) cost. If these
  facilities have infinite capacities and can serve any number of
  customers, then each customer would use his/her most preferred
  (i.e., closest) facility and the problem is similar to finding the
  Chamberlin--Courant assignment. If the capacities of the facilities
  are finite and equal, the problem looks like finding an assignment
  in the Monroe rule. An essential difference between the two problems
  are the setup costs and the distance metric.  The parameterized
  complexity of the Facility Location Problem was investigated in
  Fellows and Fornau~\cite{FF11}.  The papers of Procaccia et
  al.~\cite{complexityProportionalRepr} and of Betzler et
  al.~\cite{fullyProportionalRepr} contain a brief discussion of the
  connection between the Facility Location Problem and the winner
  determination problem under the Chamberlin--Courant rule.\medskip

\item[Group Activity Selection Problem.]  In the group activity
  selection problem~\cite{dar-elk-kur-lan-sch-woe:t:group-activity} we
  have a group of agents (say, conference attendees) and a set of
  activities (say, options that they have for a free afternoon such as
  a bus city tour or wine tasting).  The agents express
  preferences  regarding the activities and organisers try to allocate agents to activities 
  to maximise their total satisfaction. If there are $m$ possible activities but only $k$ must 
  be chosen by organisers, then we are in the Chamberline-Courant framework, if all activities can take all agents, 
  and in the Monroe framework, if all activities have the same capacities. The difference is that those 
  capacities may be different and also that in the Group Activity Selection Problem we may allow expression of more complicated preferences. 
For example, an agent may express the following preference
  ``I like wine-tasting best provided that at most
  $10$ people participate in it, and otherwise I prefer a bus city tour provided
  that at least $15$ people participate, and otherwise I prefer to not
  take part in any activity''. The Group Activity Selection Problem is more general than the winner determination in the Monroe and Chamberline-Courant rules. Some hardness and easiness results for this problem were obtained in~\cite{dar-elk-kur-lan-sch-woe:t:group-activity}, but the investigation of this problem has only started.

\end{description}

The above connections show that, indeed, the complexity of winner
determination under the Monroe and Chamberlin--Courant voting rules are
interesting, can lead to progress in several other directions, and may
have impact on other applications of artificial intelligence.

\section{Preliminaries}\label{sec:prelim}

We first define basic notions such as preference orders and
positional scoring rules. Then we present our Resource Allocation
Problem in full generality and discuss which restrictions of it correspond to the winner determination problem for the Monroe
and Chamberlin--Courant voting rules. Finally, we briefly
recall relevant notions regarding computational complexity. \medskip

\noindent
\textbf{Preferences.}\quad
For each $n \in \naturals$, by $[n]$ we mean $\{1, \ldots, n\}$.
We assume that there is a set $N = [n]$ of \emph{agents} and a set $A
= \{a_{1}, \dots a_{m}\}$ of \emph{alternatives}. Each alternative $a\in A$
has the \emph{capacity} $\capacity_a \in \naturals$, which gives the total number of 
agents that can be assigned to it.  Further, each agent $i$ has a
\emph{preference order} $\succ_i$ over $A$, i.e., a strict linear
order of the form $a_{\pi(1)} \succ_{i} a_{\pi(2)} \succ_{i} \dots
\succ_{i} a_{\pi(m)}$ for some permutation $\pi$ of $[m]$. For an
alternative $a$, by $\pos_i(a)$ we mean the position of $a$ in the $i$'th
agent's preference order. For example, if $a$ is the most preferred
alternative for $i$ then $\pos_i(a) = 1$, and if $a$ is the least
preferred one then $\pos_i(a) = m$.  A collection $V = (\succ_1,
\ldots, \succ_n)$ of agents' preference orders is called a
\emph{preference profile}.


We will often include subsets of the alternatives in the descriptions
of preference orders.  For example, if $A$ is the set of alternatives
and $B$ is some nonempty strict subset of $A$, then by $B \succ A-B$
we mean that for the preference order $\succ$ all alternatives in $B$
are preferred to those outside of $B$.


A \emph{positional scoring function} (PSF) is a function $\alpha^m \colon
[m] \rightarrow \naturals$. A PSF $\alpha^m$ is an \emph{increasing
  positional scoring function} (IPSF) if for each $i,j \in [m]$, if $i
< j$ then $\alpha^m(i) < \alpha^m(j)$. Analogously, a PSF $\alpha^m$ is
a \emph{decreasing positional scoring function} (DPSF) if for each
$i,j \in [m]$, if $i < j$ then $\alpha^m(i) > \alpha^m(j)$. 

Intuitively, if $\beta^m$ is an IPSF then $\beta^m(i)$ can represent the
\emph{dissatisfaction} that an agent suffers when assigned to an
alternative that is ranked $i$'th in his or her preference
order. Thus, we assume that for each IPSF $\beta^m$ it holds that
$\beta^m(1) = 0$ (an agent is not dissatisfied by her top
alternative).  Similarly, a DPSF $\gamma^m$ measures an agent's
satisfaction and we assume that for each DPSF $\gamma^m$ it holds that
$\gamma^m(m)=0$ (an agent is completely not satisfied being assigned his or her least desired alternative).
%
Sometimes we write $\alpha$ instead of $\alpha^m$, when it cannot lead to a confusion.

We will often speak of families $\alpha$ of IPSFs (DPSFs) of the form 
$\alpha=(\alpha^m)_{m \in \naturals}$, where $\alpha^m \mbox{ is a PSF on }[m]$,
such that:
\begin{enumerate}
\fixlist
\item For a family of IPSFs  it holds that $ \alpha^{m+1}(i) = \alpha^m(i)$ for all
  $m \in \naturals$ and $i\in [m]$.
\item For a family of DPSFs  it holds that $ \alpha^{m+1}(i+1) = \alpha^m(i)$ for all
  $m \in \naturals$ and $i\in [m]$.
\end{enumerate}
In other words, we build our families of IPSFs (DPSFs) by appending
(prepending) values to functions with smaller domains.   To simplify notation, we will refer to such
families of IPSFs (DPSFs) as \emph{normal} IPSFs (normal
DPSFs). We assume that each function $\alpha^m$ from a family can be computed in polynomial
time with respect to $m$. 
Indeed, we are particularly interested in the Borda families of IPSFs
and DPSFs defined by $\alpha^{m}_{\bordainc}(i) = i-1$ and 
$\alpha^{m}_{\bordadec}(i) = m - i$, respectively.\medskip

\noindent
\textbf{Assignment functions.}\quad
A $K$-\emph{assignment function} is any function
$\Phi \colon N \rightarrow A$, such that $\|\Phi(N)\| \leq K$ (that
is, it matches agents to at most $K$ alternatives), and such that 
for every alternative $a \in A$ we have that $\|\Phi^{-1}(a)\| \leq \capacity_a$ (i.e.,
the number of agents assigned to $a$ does not exceed $a$'s capacity $\capacity_a$).

We will also consider partial assignment functions. A partial
$K$-assignment function is defined in the same way as a regular one,
except that it may assign a null alternative, $\bot$, to some of the
agents.  It is convenient to think that for each agent $i$ we have
$\pos_i(\bot) = m$.  In general, it might be the case that a partial
$K$-assignment function cannot be extended to a regular one. This may
happen, for example, if the partial assignment function uses $K$
alternatives whose capacities sum to less than the total number of
voters. However, in the context of Chamberlin--Courant and Monroe
rules it is always possible to extend a partial $K$-assignment
function to a regular one.

Given a normal IPSF (DPSF) $\alpha$, we may consider the following three functions, each assigning a positive integer to any assignment $\Phi$:
\begin{align*}
\ell_{1}^{\alpha}(\Phi) &= \sum_{i=1}^{n}\alpha(\pos_{i}(\Phi(i))), \\
\ell_{\infty}^{\alpha}(\Phi) &= \mathrm{max}_{i = 1}^{n}\alpha(\pos_{i}(\Phi(i))), \\
\ell_{\min}^{\alpha}(\Phi) &= \mathrm{min}_{i = 1}^{n}\alpha(\pos_{i}(\Phi(i))) \textrm{.}
\end{align*}
These functions are built from individual dissatisfaction (satisfaction) functions, so that they can measure the quality of the assignment for the whole society. 
In the utilitarian framework the first one can be viewed as a \emph{total (societal) dissatisfaction function} in the IPSF case and a 
\emph{total (societal) satisfaction function} in the DPSF case. The second and the third can be used, respectively, as a total dissatisfaction and satisfaction functions for IPSF and DPSF cases  in the egalitarian framework. We will omit the word total if no confusion may arise.\medskip

%


For each subset of the alternatives $S \subseteq A$ such that $\|S\|
\leq K$, we denote as $\Phi^S_{\alpha}$ the partial $K$-assignment
that assigns agents only to the alternatives from $S$ and such that
$\Phi^S_{\alpha}$ maximizes the utilitarian satisfaction
$\ell_{1}^{\alpha}(\Phi^S_{\alpha})$.  (We introduce this notation
only for the utilitarian satisfaction-based setting because it is
useful to express appropriate algorithms for this case; for other
settings we have hardness results only and this notation would not be
useful.)\medskip

\noindent
\textbf{The Resource Allocation Problem.}\quad
Let us now define the resource allocation problem that forms the base of our study.
This problem stipulates finding an optimal $K$-assignment function, where the optimality is relative to one of the total dissatisfaction or satisfaction functions that we have just introduced. The former is to be minimized and the latter is to be maximized.

\begin{definition}\label{def:assignment}
  Let $\alpha$ be a normal IPSF. An instance of
  $\alpha$-\textsc{DU-Assignment} problem (i.e., of the
  disatisfaction-based utilitarian assignment problem) consists of a
  set of agents $N = [n]$, a set of alternatives $A = \{a_1, \ldots
  a_m\}$, a preference profile $V$ of the agents, and a sequence
  $(\capacity_{a_1}, \ldots, \capacity_{a_m})$ of alternatives' capacities. We ask for an assignment function $\Phi$
  such that:
(1) $\|\Phi(N)\| \leq K$;
(2) $ \|\Phi^{-1}(a) \| \leq \capacity_{a}$ for all ${a \in A}$; and
%
(3) $\ell_{1}^\alpha(\Phi)$ is minimized.
\end{definition}


Problem $\alpha$-\textsc{SU-Assignment} (the satisfaction-based
utilitarian assignment problem) is defined identically except that
$\alpha$ is a normal DPSF and condition (3) is replaced with ``($3'$) $\ell_{1}^\alpha(\Phi)$ is maximal.''
%
%
If we replace $\ell^\alpha_1$ with $\ell_\infty^\alpha$ in
$\alpha$-\textsc{DU-Assignment} then we obtain problem
$\alpha$-\textsc{DE-Assignment}, i.e., the dissatisfaction-based
egalitarian variant.  If we replace $\ell^\alpha_1$ with $\ell_{\min}^\alpha$
in $\alpha$-\textsc{SU-Assignment} then we obtain problem
$\alpha$-\textsc{SE-Assign\-ment}, i.e., the satisfaction-based
egalitarian variant.

Our four problems can be viewed as generalizations of the winner determination problem for the
Monroe~\cite{monroeElection} and
Chamberlin--Courant~\cite{ccElection} multiwinner voting systems
(see the introduction for their definitions). To model the Monroe
system, it suffices to set the capacity of each alternative to be $
\frac{\|N\|}{K}$ (for simplicity, throughout the paper we assume that
$K$ divides $\|N\|$\footnote{In general, this assumption is not as
  innocent as it may seem. Often dealing with cases there $K$ does not
  divide $\|N\|$ requires additional insights and care. However, for
  our algorithms and results, the assumption simiplifies notation and
  does not lead to obscuring any unexpected difficulties.}). We will
refer to thus restricted variants of our problems as the \textsc{Monroe}
variants.  To represent the Chamberlin--Courant system, we set
alternatives' capacities to $\|N\|$.  We will refer to the so-restricted 
variants of our problems as \textsc{CC} variants.\medskip


\noindent\textbf{Computational Issues}.\quad
For many normal IPSFs $\alpha$ and, in particular, for the Borda IPSF, even the
above-mentioned restricted versions of the Resource Allocation Problem, namely,
$\alpha$-\textsc{DU-Monroe}, $\alpha$-\textsc{DE-Monroe},
$\alpha$-\textsc{DU-CC}, and $\alpha$-\textsc{DE-CC} are
$\np$-complete~\cite{fullyProportionalRepr,complexityProportionalRepr}
(the same holds for the satisfaction-based variants of the
problems). Thus we seek approximate solutions.

\begin{definition}
  Let $r$ be a real number such that $r \geq 1$ ($0 < r  \leq 1$) and let $\alpha$ be a normal IPSF (a normal DPSF).  An
  algorithm is an $r$-approximation algorithm for
  $\alpha$-\textsc{DU-Assignment} problem (for
  $\alpha$-\textsc{SU-Assignment} problem) if on each instance $I$ it
  returns a feasible assignment $\Phi$ 
  such that $\ell_{1}^{\alpha}(\Phi) \leq r \cdot
  \OPT$ (such that $\ell_{1}^{\alpha}(\Phi) \geq r \cdot
  \OPT$), where $\OPT$ is the optimal total dissatisfaction
  (satisfaction) $\ell_{1}^{\alpha}(\Phi_\OPT)$. 
\end{definition}
We define $r$-approximation algorithms for the egalitarian
variants
analogously.  Lu and Boutilier~\cite{budgetSocialChoice} gave a $(1 -
\frac{1}{e})$-approximation algorithm for the \textsc{SU-CC} family of
problems.

Throughout this paper, we will consider each of the \textsc{Monroe}
and \textsc{CC} variants of the problem and for each we will either
prove inapproximability with respect to any constant $r$ (under
standard complexity-theoretic assumptions) or we will present an
approximation algorithm. In our
inapproximability proofs, we will use the following two classic
$\np$-complete problems~\cite{gar-joh:b:int}.

\begin{definition}
  An instance $I$ of \textsc{Set-Cover} consists of set $U = [n]$
  (called the ground set), family $\calF = \{F_{1}, F_{2}, \dots,
  F_{m}\}$ of subsets of $U$, and positive integer $K$. We ask if
  there exists a set $I \subseteq [m]$ such that $\|I\|
  \leq K$ and $\bigcup_{i\in I}F_i = U$.
\end{definition}

\begin{definition}
  \textsc{X3C} is a variant of \textsc{Set-Cover} where $\|U\|$
  is divisible by $3$, each member of $\calF$ has exactly three
  elements, and $K = \frac{\|U\|}{3}$.
\end{definition}

\textsc{Set-Cover} remains $\np$-complete even if we restrict each
member of $U$ to be contained in at most two sets from $\calF$ (it
suffices to note that this restriction is satisfied by
\textsc{Vertex-Cover}, which is a special case of \textsc{Set-Cover}).
\textsc{X3C} remains $\np$-complete even if we additionally assume
that $n$ is divisible by $2$ and each member of $U$ appears in at most
$3$ sets from $\calF$~\cite{gar-joh:b:int}.

We will also use results from the theory of parameterized complexity developed by Downey and Fellows~\cite{DF99}. 
This theory allows to single out a particular parameter of the problem, say $k$, and analyze its `contribution' to the overall 
complexity of the problem. An analogue of the class $\p$ here is the class $\fpt$ which is the 
class of problems that can be solved in time $f(k)n^{O(1)}$, where $n$ is the size of the input instance, and 
$f$ is some computable function (for a fixed $k$ everything gets polynomial).  Parameterized complexity theory also operates with classes $\wone
\subseteq \wtwo \subseteq \cdots$ which are believed to form a hierarchy of
classes of \emph{hard} problems (combined, they are analogous to the class
$\np$).  It holds that $\fpt \subseteq \wone$, but it seems unlikely
that $\fpt = \wone$, let alone $\fpt = \wtwo$.  We point the reader to
the books of Niedermeier~\cite{nie:b:invitation-fpt} and Flum and
Grohe~\cite{flu-gro:b:parameterized-complexity} for detailed overviews
of parametrized complexity theory.  Interestingly, while both
\textsc{Set-Cover} and \textsc{Vertex-Cover} are $\np$-complete, the
former is $\wtwo$-complete and the latter belongs to $\fpt$ (see,
e.g., the book of Niedermeier~\cite{nie:b:invitation-fpt} for these
now-standard results and their history).

\section{Hardness of Approximation}\label{sec:approximation}

We now present our inapproximability results for the Monroe and
Chamberlin--Courant rules. Specifically, we show that there are no constant-factor
approximation algorithms for the dissatisfaction-based variants of the
rules (both utilitarian and egalitarian) and for the
satisfaction-based egalitarian ones.

Naturally, these inapproximability results carry over to more general
settings.  For example, unless $\p = \np$, there are no
polynomial-time constant-factor approximation algorithms for the
general dissatisfaction-based Resource Allocation Problem. On the
other hand, our results do not preclude good satisfaction-based
approximation algorithms for the utilitarian case and, indeed, in
Section~\ref{sec:algorithms} we provide such algorithms.

\begin{theorem}\label{theorem:noApprox1}
  For each normal IPSF $\alpha$ and each constant factor 
  $r > 1$, there is no polynomial-time $r$-approximation
  algorithm for $\alpha$-\textsc{DU-Monroe} unless $\p = \np$.
\end{theorem}
\begin{proof}
  Let us fix a normal IPSF $\alpha$ and let us assume, aiming at getting a
  contradiction, that there is some constant $r > 1$ and
  a polynomial-time $r$-approximation algorithm $\calA$ for
  $\alpha$-\textsc{DU-Monroe}.

  Let $I$ be an instance of \textsc{X3C} with ground set $U = [n]$ and
  family $\calF = \{F_{1}, F_{2}, \dots, F_{m}\}$ of $3$-element
  subsets of $U$. Without loss of generality, we assume that $n$ is divisible by both
  $2$ and $3$ and that each member of $U$ appears in at most 3 sets
  from $\calF$.

  Using $I$, we build instance $I_{M}$ of
  $\alpha$-\textsc{DU-Monroe} as follows.  We set $N = U$
  (that is, the elements of the ground set are the agents) and we set $A
  = A_1 \cup A_2$, where $A_1 = \{a_1, \ldots, a_m\}$ is a set of
  alternatives corresponding to the sets from the family $\calF$ and
  $A_2$ is a
  set of dummy alternatives of cardinality $\|A_2\| = \frac{1}{2}{n^{2}r\cdot \alpha(3)}$, needed for the construction. We let $m' =
  \|A_2\|$ and rename the alternatives in $A_2$ so that $A_2 =
  \{b_1, \ldots, b_{m'}\}$.  We set $K = \frac{n}{3}$.

  We build agents' preference orders using the following algorithm.
  For each $j \in N$, set $M_f(j) = \{ a_i \mid j \in F_i\}$ and $M_l
  = \{ a_i \mid j \not\in F_i \}$. Set $m_f(j) = \|M_f(j)\|$ and
  $m_l(j) = \|M_l(j)\|$. As the frequency of the elements from $U$ is
  bounded by 3, we have $m_f(j) \leq 3$. For each agent $j$ we set his or her
  preference order to be of the form $M_f(j) \succ_j A_2 \succ_j
  M_l(j)$, where the alternatives in $M_f(j)$ and $M_l(j)$ are ranked
  in an arbitrary way and the alternatives from $A_2$ are placed at
  positions $m_{f}(j) + 1, \dots, m_{f}(j) + m'$ in the way described
  below (see Figure \ref{fig:diag1} for a high-level illustration of
  the construction).
  
  \begin{figure}[tb]
    \begin{center}
      \includegraphics[scale=0.75]{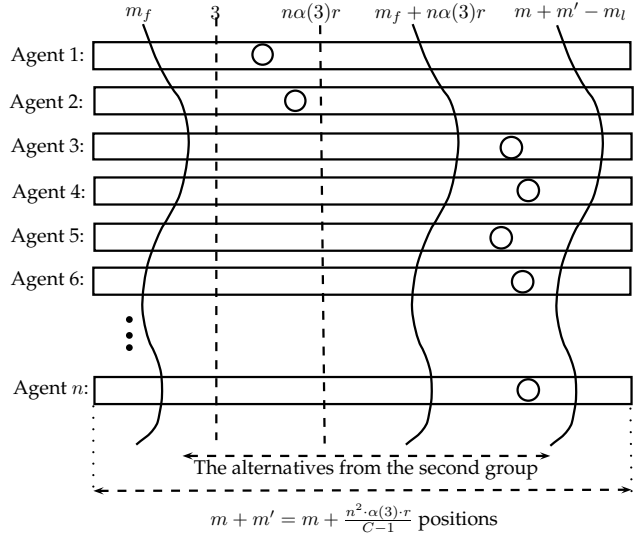}
    \end{center}
   \vspace{-0.75cm}
    \caption{The alignment of the positions in the preference orders
      of the agents. The positions are numbered from the left to the
      right. The left wavy line shows the positions $m_{f}(\cdot)$,
      each no greater than $3$. The right wavy line shows the
      positions $m_{l}(\cdot)$, each higher than $nr \cdot \alpha(3)$. The alternatives from $A_{2}$ (positions of one
      such an alternative is illustrated with the circle) are placed
      only between the peripheral wavy lines. Each alternative from
      $A_{2}$ is placed on the left from the middle wavy line exactly
      2 times, thus each such alternative is placed on the left from
      the right dashed line no more than $2$ times (exactly two times
      at the figure).}
    \label{fig:diag1}
  \end{figure}

  We place the alternatives from $A_{2}$ in the preference orders of
  the agents in such a way that for each alternative $b_i \in A_{2}$
  there are at most two agents that rank $b_i$ among their $nr \cdot
  \alpha(3)$ top alternatives.  The following construction
  achieves this effect.  If $(i + j) \,\mathrm{mod}\, n < 2$, then
  alternative $b_{i}$ is placed at one of the positions $m_{f}(j) + 1,
  \dots, m_{f}(j) + nr \cdot \alpha(3)$ in $j$'s preference
  order. Otherwise, $b_i$ is placed at a position with index higher
  than $m_{f}(j) + nr \cdot \alpha(3)$ (and, thus, at a
  position higher than $nr \cdot \alpha(3)$). This
  construction can be implemented because for each agent $j$ there are
  exactly $m' \cdot \frac{2}{n} = nr \cdot \alpha(3)$
  alternatives $b_{i_{1}}, b_{i_{2}}, b_{i_{n \alpha(3) r}}$ such
  that $(i_{k} + j) \,\mathrm{mod}\, n < 2$.

  Let $\Phi$ be an assignment computed by $\calA$ on $I_{M}$.  We
  will show that $\ell_{1}^{\alpha}(\Phi) \leq n \cdot \alpha(3) \cdot
  r$ if and only if $I$ is a \emph{yes}-instance of \textsc{X3C}.

  ($\Leftarrow$) If there exists a solution for $I$ (i.e., an exact
  cover of $U$ with $\frac{n}{3}$ sets from $\calF$), then we can
  easily show an assignment in which each agent $j$ is assigned to an
  alternative from the top $m_{f}(j)$ positions of his or her
  preference order (namely, one that assigns each agent $j$ to the
  alternative $a_i \in A_1$ that corresponds to the set $F_i$, from
  the exact cover of $U$, that contains $j$). Thus, for the optimal
  assignment $\Phi_\OPT$ it holds that $\ell_{1}^{\alpha}(\Phi_\OPT) \leq
  \alpha(3) \cdot n$. In consequence, $\mathcal{A}$ must return an
  assignment with the total dissatisfaction at most $nr \cdot \alpha(3)$.

  ($\Rightarrow$) Let us now consider the opposite direction. We
  assume that $\calA$ found an assignment $\Phi$ such that
  $\ell_{1}^{\alpha}(\Phi) \leq nr \cdot \alpha(3)$ and we
  will show that $I$ is a \emph{yes}-instance of \textsc{X3C}.  Since
  we require each alternative to be assigned to either $0$ or $3$
  agents, if some alternative $b_{i}$ from $A_2$ were assigned to some
  $3$ agents, at least one of them would rank $b_i$ at a position
  worse than $nr \cdot \alpha(3)$. This would mean that
  $\ell_{1}^{\alpha}(\Phi) \geq nr \cdot \alpha(3)  +
  1$. Analogously, no agent $j$ can be assigned to an alternative that
  is placed at one of the $m_l(j)$ bottom positions of $j$'s
  preference order. Thus, only the alternatives in $A_1$ have agents
  assigned to them and, further, if agents $x$, $y$, $z$, are assigned
  to some $a_i \in A_1$, then it holds that $F_i = \{x,y,z\}$ (we will
  call each set $F_i$ for which alternative $a_i$ is assigned to some
  agents $x,y,z$ \emph{selected}). Since each agent is assigned to
  exactly one alternative, the selected sets are disjoint. Since the
  number of selected sets is $K = \frac{n}{3}$, it must be the case
  that the selected sets form an exact cover of $U$. Thus, $I$ is a
  \emph{yes}-instance of \textsc{X3C}.
\end{proof}

One may wonder if hardness of approximation for
$\alpha$-\textsc{DU-Monroe} is not an artifact of the strict
requirements regarding the number of chosen candidates.  It turns out
that unless $\p = \np$, there is no $r$-$s$-approximation
algorithm that finds an assignment with the following properties: (1)
the aggregated dissatisfaction $\ell_{1}^{\alpha}(\Phi)$ is at most
$r$ times higher than the optimal one, (2) the number of
alternatives to which agents are assigned is at most $s K$ and
(3) each selected alternative (the alternative that has agents
assigned), is assigned to no more than $s \lceil \frac{n}{K}
\rceil$ and no less than $\frac{1}{s} \lceil \frac{n}{K}\rceil$
agents. (The proof is similar to the one used for
Theorem~\ref{theorem:noApprox1}.)  Thus, in our further study we do
not consider such relaxations of the problem.


\begin{theorem}\label{theorem:noApprox2}
  For each normal IPSF $\alpha$ and each constant $r > 1$, there is no polynomial-time $r$-approximation algorithm for
  $\alpha$-\textsc{DE-Monroe} unless $\p = \np$.
\end{theorem}

\begin{proof}
  The proof of Theorem~\ref{theorem:noApprox1} applies to this case as
  well. In fact, it even suffices to take $m' = \|A_2\| = \frac{1}{2}nr  \cdot \alpha(3)$.
\end{proof}

Results analogous to Theorems~\ref{theorem:noApprox1}
and~\ref{theorem:noApprox2} hold for the \textsc{DU-CC}
family of problems as well. 

\begin{theorem}\label{theorem:noApprox3}
  For each normal IPSF $\alpha$ and each constant factor 
  $r > 1$, there is no polynomial-time $r$-approximation
  algorithm for $\alpha$-\textsc{DU-CC} unless $\p = \np$.
\end{theorem}

\begin{proof}
  Let us fix a normal IPSF $\alpha$.  For the sake of contradiction,
  let us assume that there is some constant $r > 1$, and a
  polynomial-time $r$-approximation algorithm $\calA$ for
  $\alpha$-\textsc{DU-CC}. We will show that it is
  possible to use $\calA$ to solve the $\np$-complete
  \textsc{Vertex-Cover} problem.

  Let $I = (U, \calF, K)$ be an instance of \textsc{Vertex-Cover},
  where $U = [n]$ is the ground set, $\calF = \{F_1, \ldots, F_m\}$ is
  a family of subsets of $U$ (where each member of $U$ belongs to exactly
  two sets in $\calF$), and $K$ is a positive integer. 
  

  Given $I$, we construct an instance $I_{CC}$ of
  $\alpha$-\textsc{DU-CC} as follows. The set of
  agents is $N = U$ and the set of alternatives is $A =
  \bigcup_{j=1}^m A_j$, where each $A_j$ contains exactly $\alpha(2)
  \cdot r \cdot n$ (unique) alternatives. Intuitively, for each
  $j\in [m]$, the alternatives in $A_j$ correspond to the set
  $F_j$. For each $A_j$, $1 \leq j \leq m$, we pick one alternative,
  which we denote $a_j$.  For each agent $i \in N$, we set $i$'s
  preference order as follows: Let $F_j$ and $F_k$, $j < k$, be the
  two sets that contain $i$. Agent $i$'s preference order is of the
  form $a_j \succ_i a_k \succ_i A_k - \{a_k\} \succ_i A - (A_k \cup
  \{a_j,a_k\})$ (a particular order of alternatives in the sets
  $A_k-\{a_k\}$ and $A - (A_k \cup \{a_j,a_k\})$ is irrelevant for the
  construction). We ask for an assignment of the agents to at most $K$
  alternatives.

  %

  Let us consider a solution $\Phi$ returned by $\mathcal{A}$ on input
  $I_{CC}$. We claim that $\ell_{1}^{\alpha}(\Phi) \leq nr \cdot
  \alpha(2)$ if and only if $I$ is a \emph{yes}-instance
  of \textsc{Vertex-Cover}.

  ($\Leftarrow$) If $I$ is a \emph{yes}-instance then, clearly, each
  agent $i$ can be assigned to one of the top two alternatives in his
  or her preference order (if there is a size-$K$ cover, then this
  assignment selects at most $K$ candidates).  Thus the total
  dissatisfaction of an optimal assignment is at most $n\cdot \alpha(2)$. As a result, the solution $\Phi$ returned by $\calA$ has total
  dissatisfaction at most $nr\cdot \alpha(2)$.

  ($\Rightarrow$) If $\mathcal{A}$ returns an assignment with total
  dissatisfaction no greater than $nr\cdot \alpha(2)$,
  then, by the construction of agents preference orders, we see that
  each agent $i$ was assigned to an alternative from a set $A_j$ such
  that $i \in F_j$. Since the assignment can use at most $K$ alternatives,
  this directly implies that there is a size-$K$ cover of $U$ with sets
  from $\calF$.
\end{proof}

\begin{theorem}\label{theorem:noApprox4}
  For each normal IPSF $\alpha$ and each constant factor $r > 1$, there is no polynomial-time $r$-approximation
  algorithm for $\alpha$-\textsc{DE-CC} unless $\p  = \np$.
\end{theorem}

\begin{proof}
  The proof of Theorem~\ref{theorem:noApprox3} is applicable in this
  case as well. In fact, it even suffices to take the $m$ groups of
  alternatives, $A_1, \ldots, A_m$, to contain $\alpha(2) \cdot r$
  alternatives each.
\end{proof}

The above results show that approximating algorithms for finding the minimal dissatisfaction
of agents is difficult. On the other hand, if we focus on agents'
total satisfaction then constant-factor approximation exist in many
cases (see, e.g., the work of Lu and
Boutilier~\cite{budgetSocialChoice} and the next section). Yet, if we
focus on the satisfaction of the least satisfied voter, there are no
efficient constant-factor approximation algorithms for the Monroe and
Chamberlin--Courant systems. (However, note that our result for
the Monroe setting is more general than the result for the
Chamberlin--Courant setting; the latter is for the Borda DPSF only.)

\begin{theorem}\label{theorem:noApprox5}
  For each normal DPSF $\alpha$ (where each entry is polynomially
  bounded in the number of alternatives) and each constant factor
  $r$, with $0 < r \leq 1$, there is no $r$-approximation
  algorithm for $\alpha$-\textsc{SE-Monroe} unless
  $\p = \np$.
\end{theorem}

\begin{proof}
  Let us fix a DPSF $\alpha=(\alpha^m)_{m\in \naturals}$, where each entry $\alpha^m$ is polynomially bounded
  in the number of alternatives $m$. For the sake of contradiction, let us
  assume that for some $r$, $0 < r \leq 1$, there is a
  polynomial-time $r$-approximation algorithm $\mathcal{A}$ for
  $\alpha$-\textsc{SE-Monroe}. We will show that
  the existence of this algorithm implies that \textsc{X3C} is
  solvable in polynomial time.

  Let $I$ be an \textsc{X3C} instance with ground set $U = \{1, 2,
  \dots, n\}$ and collection $\calF = \{F_{1}, \dots, F_{m}\}$ of
  subsets of $U$. Each set in $\calF$ has cardinality three. Further,
  without loss of generality, we can assume that $n$ is divisible by three and that each
  $i \in U$ appears in at most three sets from $\calF$.  Given $I$, we
  form an instance $I_M$ of
  $\alpha$-\textsc{SE-Monroe} as follows.  Let $n'
  = 3 \cdot (\alpha^{m+1}(1) \cdot \lceil \frac{1 - r}{r}\rceil +
  3)$.  The set $N$ of agents is partitioned into two subsets, $N_1$
  and $N_2$. $N_1$ contains $n$ agents (intuitively, corresponding to
  the elements of the ground set $U$) and $N_2$ contains $n'$ agents
  (used to enforce certain properties of the solution).  The set of
  alternatives $A$ is partitioned into two subsets, $A_1$ and $A_2$.
  We set $A_1 = \{a_1, \ldots, a_m\}$ (members of $A_1$ correspond to
  the sets in $\calF$), and we set $A_2 = \{b_1, \ldots, b_{m'}\}$,
  where $m' = \frac{n'}{3}$.

  For each $j$, $1 \leq j \leq n$, we set $M_f(j) = \{ a_i \mid j \in
  F_i\}$.  For each $j$, $1 \leq j \leq n$, we set the preference
  order of the $j$'th agent in $N_1$ to be of the form 
  \[ 
     M_f(j) \succ A_2 \succ A_1 - M_f(j).\] 
  Note that by our assumptions, $\|M_f(j)\| \leq 3$.
  For each $j$, $1 \leq j \leq n'$, we set the preference order of the
  $j$'th agent in $N_2$ to be of the form
  \[ 
     b_{\left\lceil\frac{j}{3}\right\rceil} \succ A_2 -
     \{b_{\left\lceil\frac{j}{3}\right\rceil}\} \succ A_1.
  \]
  Note that each agent in $N_2$ ranks the alternatives from $A_1$ in
  positions $m'+1, \ldots, m'+m$.  Finally, we set the
  number of candidates that can be selected to be $K =
  \frac{n+n'}{3}$.

  Now, consider the solution $\Phi$ returned by $\mathcal{A}$ on
  $I_{M}$. We will show that $\ell_{\infty}^{\alpha^{m +
      m'}}(\Phi) \leq$ $r\alpha^{m + m'}(3)$ if and only
  if $I$ is a \emph{yes}-instance of \textsc{X3C}.

  ($\Leftarrow$) If there exists an exact set cover of $U$ with sets
  from $\calF$, then it is easy to construct a solution for $I_M$
  where the satisfaction of each agent is greater or equal to
  $r\cdot\alpha^{m + m'}(3)$. Let $I \subseteq \{1, \ldots, m\}$
  be a set such that $\bigcup_{i \in I}F_i = U$ and $\|I\| =
  \frac{n}{3}$. We assign each agent $j$ from $N_1$ to the alternative
  $a_i$ such that (a) $i \in I$ and (b) $j \in F_i$, and we assign
  each agent from $N_2$ to his or her most preferred alternative.
  Thus, Algorithm $\calA$ has to return an assignment with the minimal
  satisfaction greater or equal to $r\cdot\alpha^{m + m'}(3)$.

  ($\Rightarrow$) For the other direction, we first show that
  $r\cdot\alpha^{m + m'}(3) \geq \alpha^{m + m'}(m')$.
  Since DPSFs are strictly decreasing, it holds that:
  \begin{equation}
    \label{eq:1}
    r\cdot\alpha^{m + m'}(3) \geq r\cdot(\alpha^{m + m'}(m') + m' - 3).
  \end{equation}
  Then, by the definition of DPSFs, it holds that:
  \begin{equation}
    \label{eq:2}
    \alpha^{m + m'}(m') = \alpha^{m + 1}(1).
  \end{equation}
  Using the fact that $m' = (\alpha^{m+1}(1) \cdot \lceil
  \frac{1 - r}{r}\rceil + 3)$ and using~\eqref{eq:2},
  we can transform inequality~\eqref{eq:1} to obtain the following:
  \begin{align*}
    r\cdot\alpha^{m + m'}(3) &\geq r\cdot(\alpha^{m + m'}(m') + m' - 3) \\
       & = r\cdot\left(\alpha^{m + m'}(m') + (\alpha^{m+1}(1) \cdot \left\lceil\frac{1 - r}{r}\right\rceil + 3) - 3\right)\\
       & \geq r\cdot\alpha^{m + m'}(m') + (1-r)\cdot \alpha^{m+1}(1) \\
       &  = r\cdot\alpha^{m + m'}(m') + (1-r)\cdot \alpha^{m + m'}(m') = \alpha^{m + m'}(m').
  \end{align*}
%
  This means that if the minimal satisfaction of an agent is at least
  $r\cdot\alpha^{m + m'}(3)$, then no agent was assigned to
  an alternative that he or she ranked beyond position $m'$.  If some
  agent $j$ from $N_{1}$ were assigned to an alternative from $A_{2}$,
  then, by the pigeonhole principle, some agent from $N_{2}$ would be
  assigned to an alternative from $A_{1}$. However, each agent in
  $N_2$ ranks the alternatives from $A_1$ beyond position $m'$ and
  thus such an assignment is impossible.  In consequence, it must be
  that each agent in $j$ was assigned to an alternative that
  corresponds to a set $F_{i}$ in $\calF$ that contains $j$. Such an
  assignment directly leads to a solution for $I$.
\end{proof}

Let us now move on to the case of \textsc{SE-CC}
family of problems. Unfortunately, in this case our inapproximability
argument holds for the case of Borda DPSF only (though we believe that
it can be adapted to other DPSFs as well). Further, in our previous
theorems we were showing that existence of a respective
constant-factor approximation algorithm implies that $\np$ collapses
to $\p$. In the following theorem we will show a seemingly weaker
collapse of $\wtwo$ to $\fpt$.

To prove hardness of approximation for
$\alpha_\bordadec$-\textsc{SE-CC}, we first prove
the following simple lemma.

\begin{lemma}\label{lemma:coveringSubsets}
  Let $K, p, l$ be three positive integers and let $X$ be a set of
  cardinality $lpK$. There exists a family $\calS = \{S_1, \ldots,
  S_{\binom{lK}{K}} \}$ of $pK$-element subsets of $X$ such that for
  each $K$-element subset $B$ of $X$, there is a set $S_i \in \calS$
  such that $B \subseteq S_i$.
\end{lemma}
\begin{proof}
  Set $X' = [lK]$ and let $Y'$ be a family of all $K$-element subsets
  of $X'$. Replace each element $i$ of $X'$ with $p$ new elements (at
  the same time replacing $i$ with the same $p$ elements within each
  set in $Y'$ that contains $i$). As a result we obtain two new sets,
  $X$ and $Y$, that satisfy the statement of the theorem (up to the
  renaming of the elements).
\end{proof}

\begin{theorem}\label{theorem:noApprox6}
  Let $\alpha^{m}_{\bordadec}$ be the Borda DPSF
  ($\alpha^{m}_{\bordadec}(i) = m - i$). For each constant factor
  $r$, $0 < r \leq 1$, there is no $r$-approximation
  algorithm for
  $\alpha^{m}_{\bordadec}$-\textsc{SE-CC} unless
  $\fpt = \wtwo$.
\end{theorem}
\begin{proof}
  For the sake of contradiction, let us assume that there is some
  constant $r$, $0 < r \leq 1$, and a polynomial-time
  $r$-approximation algorithm $\calA$ for
  $\alpha^{m}_{\bordadec}$-\textsc{SE-CC}. We will
  show that the existence of this algorithm implies that
  \textsc{Set-Cover} is fixed-parameter tractable for the parameter
  $K$ (since \textsc{Set-Cover} is known to be $\wtwo$-complete for
  this parameter, this will imply $\fpt=\wtwo$).

  Let $I$ be an instance of \textsc{Set-Cover} with ground set $U =
  [n]$ and family $\calF = \{F_{1}, F_{2}, \dots, F_{m}\}$ of subsets
  of $U$.  Given $I$, we build an instance $I_{CC}$ of
  $\alpha^{m}_{\bordadec}$-\textsc{SE-CC} as
  follows. The set of agents $N$ consists of $n$ subsets of agents,
  $N_1, \ldots, N_n$, where each group $N_i$ contains exactly $n' =
  \binom{\left\lceil \frac{2}{r} \right\rceil K}{K}$ agents.
  Intuitively, for each $i$, $1 \leq i \leq n$, the agents in the set
  $N_{i}$ correspond to the element $i$ in $U$.  The set of
  alternatives $A$ is partitioned into two subsets, $A_1$ and $A_2$,
  such that: (1) $A_1 = \{a_1, \ldots, a_m\}$ is a set of alternatives
  corresponding to the sets from the family $\calF$, and (2) $A_2$,
  $\|A_2\| = \left\lceil \frac{2}{r}\right\rceil \left\lceil
    \frac{m(1 + r)}{K} \right\rceil K$, is a set of dummy
  alternatives needed for our construction.  We set $m' = \|A\| = m +
  \|A_2\|$.

  Before we describe the preference orders of the agents in $N$, we
  form a family $R = \{r_1, \ldots, r_{n'}\}$ of preference orders
  over $A_2$ that satisfies the following condition: For each
  $K$-element subset $B$ of $A_2$, there exists $r_j$ in $R$ such that
  all members of $B$ are ranked among the bottom $\left\lceil \frac{m(1 +
    r)}{K} \right\rceil K$ positions in $r_j$. By
  Lemma~\ref{lemma:coveringSubsets}, such a construction is possible
  (it suffices to take $l = \left\lceil \frac{2}{r}\right\rceil$ and $p =
  \left\lceil \frac{m(1 + r)}{K} \right\rceil$); further, the proof of the
  lemma provides an algorithmic way to construct $R$.

  We form the preference orders of the agents as follows.  For each
  $i$, $1 \leq i \leq n$, set $M_f(i) = \{ a_t \mid i \in F_t\}$. For
  each $i$, $1 \leq i \leq n$, and each $j$, $1 \leq j \leq n'$, the
  $j$'th agent from $N_i$ has preference order of the form:
  \[
     M_f(i) \succ r_j \succ A_1 - M_f(i)
  \]
  (we pick any arbitrary, polynomial-time computable order of
  candidates within $M_f(i)$ and $M_l(i)$).

  Let $\Phi$ be an assignment computed by $\calA$ on $I_{M}$.  We will
  show that $\ell_{\infty}^{\alpha^{m'}_{\bordadec}}(\Phi) \geq
  r\cdot(m' - m)$ if and only if $I$ is a \emph{yes}-instance of
  \textsc{Set-Cover}.

  ($\Leftarrow$) If there exists a solution for $I$ (i.e., a cover of
  $U$ with $K$ sets from $\calF$), then we can easily show an
  assignment where each agent is assigned to an alternative that he or
  she ranks among the top $m$ positions (namely, for each $j$, $1 \leq
  j \leq n$, we assign all the agents from the set $N_j$ to the
  alternative $a_i \in A_1$ such that $j \in F_i$ and $F_i$ belongs to
  the alleged $K$-element cover of $U$). Under this assignment, the
  least satisfied agent's satisfaction is at least $m'-m$ and, thus,
  $\calA$ has to return an assignment $\Phi$ where
  $\ell_{\infty}^{\alpha^{m'}_{\bordadec}}(\Phi) \geq r\cdot(m' -
  m)$.

  ($\Rightarrow$) Let us now consider the opposite direction. We
  assume that $\calA$ found an assignment $\Phi$ such that
  $\ell_{\infty}^{\alpha^{m}_{\bordadec}}(\Phi) \geq r\cdot(m' -
  m)$ and we will show that $I$ is a \emph{yes}-instance of
  \textsc{Set-Cover}.  We claim that for each $i$, $1 \leq i \leq n $,
  at least one agent $j$ in $N_i$ were assigned to an alternative from
  $A_{1}$. If all the agents in $N_i$ were assigned to alternatives
  from $A_2$, then, by the construction of $R$, at least one of them
  would have been assigned to an alternative that he or she ranks at a
  position greater than $\|A_2\| - \left\lceil \frac{m(1 + r)}{K}\right\rceil
  K = \left\lceil \frac{2}{r}\right\rceil \left\lceil \frac{m(1 + r)}{K}
  \right\rceil K - \left\lceil \frac{m(1 + r)}{K}\right\rceil K$.  For $x =
  \left\lceil \frac{m(1 + r)}{K} \right\rceil K$ we have:
  \begin{align*}
    \left\lceil \frac{2}{r}\right\rceil x - x \geq m' - m'r +
    mr
  \end{align*}
  (we skip the straightforward calculation)
  and, thus, this agent would have been assigned to an
  alternative that he or she ranks at a position greater than $m' -
  m'r + mr$.  As a consequence, this agent's satisfaction
  would be lower than $(m' - m)r$.
  Similarly, no agent from $N_{i}$ can be assigned to an alternative
  from $M_l(i)$. Thus, for each $i$, $1 \leq i \leq n$, there exists
  at least one agent $j \in N_{i}$ that is assigned to an alternative
  from $M_f(i)$.  In consequence, the covering subfamily of $\calF$
  consists simply of those sets $F_k$, for which some agent is
  assigned to alternative $a_k \in A_1$.

  The presented construction gives the exact algorithm for
  \textsc{Set-Cover} problem running in time $f(K)(n+m)^{O(1)}$, where
  $f(K)$ is polynomial in $\binom{\left\lceil \frac{2}{r}
    \right\rceil}{K}$. The existence of such an algorithm means that
  \textsc{Set-Cover} is in $\fpt$. On the other hand, we know that
  \textsc{Set-Cover} is $\wtwo$-complete, and thus if $\calA$ existed
  then $\fpt = \wtwo$ would hold.
\end{proof}


\section{Algorithms for the Utilitarian, Satisfaction-Based Cases}
\label{sec:algorithms}

We now turn to approximation algorithms for the Monroe and Chamberlin--Courant multiwinner voting rules in the satisfaction-based framework. Indeed, if one focuses on
agents' total satisfaction then it is possible to obtain high-quality
approximation results.  In particular, we show the first nontrivial
(randomized) approximation algorithm for
$\alpha_{\bordadec}$-\textsc{SU-Monroe}. We show that for each
$\epsilon > 0$ we can provide a randomized polynomial-time algorithm
that achieves $0.715 - \epsilon$ approximation ratio; the algorithm
usually gives even better approximation guarantees.  For the case of
arbitrarily selected DPSF we show a $(1 - e^{-1})$-approximation
algorithm.  Finally, we show the first polynomial-time approximation
scheme (PTAS) for $\alpha_{\bordadec}$-\textsc{SU-CC}. These results
stand in  sharp contrast to those from the previous section, where we
have shown that approximation is hard for essentially all remaining
variants of the problem.

The core difficulty in solving $\alpha$-\textsc{Monroe/CC-Assignment}
problems lays in selecting the alternatives that should be assigned to
the agents. Given a preference profile and a set $S$ of up to $K$
alternatives, using a standard network-flow argument, it is easy to
find a (possibly partial) optimal assignment $\Phi^S_{\alpha}$ of the agents to
the alternatives from $S$.

\begin{proposition}[\textbf{Implicit in the paper of Betzler et
    al.~\cite{fullyProportionalRepr}}]\label{prop:assignment}
  Let $\alpha$ be a normal DPSF, $N$ be a set of agents, $A$ be a set
  of alternatives (togehter with their capacities; perhaps represented
  implicitly as for the case of the Monroe and Chamberlin--Courant
  rules), $V$ be a preference profile of $N$ over $A$, and $S$ a
  $K$-element subset of $A$ (where $K$ divides $\|N\|$).  Then there is a
  polynomial-time algorithm that computes a (possibly partial)
%
%
  optimal assignment $\Phi^S_{\alpha}$ of the agents to the
  alternatives from $S$.
\end{proposition}
Note that for the case of the Chamberlin--Courant rule the algorithm
from the above proposition can be greatly simplified: To each voter we
assign the candidate that he or she ranks highest among those from
$S$. For the case of Monroe, unfortunately, we need the expensive
network-flow-based approach. Nonetheless,
Proposition~\ref{prop:assignment} allows us to focus on the issue of
selecting the winning alternatives and not on the issue of matching
them to the agents.


Below we describe our algorithms for
$\alpha_\bordadec$-\textsc{SU-Monroe} and for
$\alpha_\bordadec$-\textsc{SU-CC}.  Formally speaking, every
approximation algorithm for $\alpha_\bordadec$-\textsc{SU-Monroe} also
gives feasible results for $\alpha_\bordadec$-\textsc{SU-CC}. However,
some of our algorithms are particularly well-suited for both problems
and some are tailored to only one of them.  Thus, for each algorithm
we clearly indicate if it is meant only for the case of Monroe, only
for the case of CC, or if it naturally works for both systems.

\subsection{Algorithm A (Monroe)}

Perhaps the most natural approach to solve
$\alpha_\bordadec$-\textsc{SU-Monroe} is to build a solution
iteratively: In each step we pick some not-yet-assigned alternative
$a_i$ (using some criterion) and assign it to those $\lceil
\frac{N}{K} \rceil$ agents that (a) are not assigned to any other
alternative yet, and (b) whose satisfaction of being matched with
$a_i$ is maximal. It turns out that this idea, implemented formally as
Algorithm~A (see pseudo code in Figure~\ref{alg:greedy}), works very
well in many cases. We provide a lower bound on the total satisfaction
it guarantees in the next lemma.  We remind the reader that the so-called $k$'th \emph{harmonic
number} $H_k = \sum_{i=1}^k\frac{1}{i}$ has asymptotics $H_k = \Theta(\log k)$.

\SetKwInput{KwNotation}{Notation}
\begin{figure}[t]
\begin{algorithm}[H]
   \small
   \SetAlCapFnt{\small}
   \KwNotation{$\Phi \leftarrow$ a map defining a partial assignment, iteratively built by the algorithm. \\
          $\hspace{21pt}$ $\Phi^{\leftarrow} \leftarrow$ the set of agents for which the assignment is 
          already defined. \\
          $\hspace{21pt}$ $\Phi^{\rightarrow} \leftarrow$ the set of alternatives already used in the 
          assignment.}
   \If{$K \leq 2$}{compute the optimal solution using an algorithm of Betzler et al.~\cite{fullyProportionalRepr} and return.} 
   $\Phi = \{\}$ \\
   \For{$i\leftarrow 1$ \KwTo $K$}{
      $score \leftarrow \{\}$ \\
      $bests \leftarrow \{\}$ \\
      \ForEach{$a_{i} \in A \setminus \Phi^{\rightarrow}$}{
          $agents \leftarrow$ sort $N \setminus \Phi^{\leftarrow}$ so that if agent $j$ precedes agent $k$ then 
                           $pos_{j}(a_{i}) \leq pos_{k}(a_{i})$ \\
          $bests[a_{i}] \leftarrow$ chose first $\lceil \frac{N}{K} \rceil$ elements from $agents$ \\
          $score[a_{i}] \leftarrow \sum_{j \in bests[a_i]} (m - pos_{j}(a_{i}))$\\
      }
      $a_{best} \leftarrow \mathrm{argmax}_{a \in A \setminus \Phi^{\rightarrow}} score[a]$ \\
      \ForEach{$j \in bests[a_{best}]$}{
         $\Phi[j] \leftarrow a_{best}$ \\
      }
   }

\end{algorithm}
\caption{The pseudocode for Algorithm~A.}\label{alg:greedy}
\end{figure}

\begin{lemma}\label{lemma:greedy}
  Algorithm~A 
  is a polynomial-time $(1 -
  \frac{K-1}{2(m-1)} - \frac{H_K}{K})$-approximation algorithm for
  $\alpha_{\bordadec}$-\textsc{SU-Monroe}.
\end{lemma}
\begin{proof}
  Our algorithm explicitly computes an optimal solution when $K \leq 2$ so we
  assume that $K \geq 3$.  Let us consider the situation in the algorithm
  after the $i$'th iteration of the outer loop (we have $i=0$ if no
  iteration has been executed yet). So far, the algorithm has picked
  $i$ alternatives and assigned them to $i\frac{n}{K}$ agents (recall
  that for simplicity we assume that $K$ divides $n$ evenly). Hence,
  each agent has $\lceil \frac{m-i}{K-i} \rceil$ unassigned
  alternatives among his or her $i+ \lceil \frac{m-i}{K-i} \rceil$
  top-ranked alternatives.  By pigeonhole principle, this means that
  there is an unassigned alternative $a_{\ell}$ who is ranked among
  top $i+ \lceil \frac{m-i}{K-i} \rceil$ positions by at least
  $\frac{n}{K}$ agents. To see this, note that there are
  $(n-i\frac{n}{K})\lceil \frac{m-i}{K-i} \rceil$ slots for unassigned
  alternatives among the top $i+ \lceil \frac{m-i}{K-i} \rceil$
  positions in the preference orders of unassigned agents, and that
  there are $m-i$ unassigned alternatives. As a result, there must be
  an alternative $a_\ell$ for whom the number of agents that rank him
  or her among the top $i+ \lceil \frac{m-i}{K-i} \rceil$ positions is
  at least:
  \[
       \frac{1}{m-i}\left((n-i\frac{n}{K})\left\lceil \frac{m-i}{K-i} \right\rceil\right) \geq
       \frac{n}{m-i}\left(\frac{K-i}{K}\right)\left(\frac{m-i}{K-i}\right)  
       =\frac{n}{K}.
  \]
%
%
%
  In consequence, the $\lceil \frac{n}{K} \rceil$ agents assigned in
  the next step of the algorithm will have the total satisfaction at
  least $\lceil \frac{n}{K} \rceil \cdot (m - i - \lceil
  \frac{m-i}{K-i} \rceil)$. Thus, summing over the $K$ iterations, the
  total satisfaction guaranteed by the assignment $\Phi$ computed by
  Algorithm~A
  is at least the following value:
  (to derive the fifth line from the fourth one we
   note that $K(H_K-1) - H_K \geq 0$ when $K \geq 3$):
  \begin{align*}
    \ell_{1}^{\alpha_{b}}(\Phi) & \geq \sum_{i = 0}^{K-1} \frac{n}{K}  \cdot \left(m - i - \lceil \frac{m-i}{K-i} \rceil\right) \\
    & \geq \sum_{i = 0}^{K-1} \frac{n}{K}  \cdot \left( m - i - \frac{m-i}{K-i} -1 \right) \\
    & = \sum_{i = 1}^{K} \frac{n}{K} \cdot \left(m - i - \frac{m-1}{K-i+1} + \frac{i-2}{K-i+1} \right) \\
    & = \frac{n}{K}\left( \frac{K(2m-K-1)}{2} -(m-1) H_K + K(H_K-1) - H_K \right) \\
    & \geq \frac{n}{K}\left( \frac{K(2m-K-1)}{2} -(m-1) H_K \right) \\
    & \geq (m-1)n \left( 1 - \frac{K-1}{2(m-1)} - \frac{H_K}{K} \right)
  \end{align*}
  If each agent were assigned to his or her top alternative, the total
  satisfaction would be equal to $(m-1)n$. Thus we get the following
  bound:
\begin{align*}
  \frac{\ell_{1}^{\alpha_{\bordadec}}(\Phi)}{\OPT} \leq 1 - \frac{K-1}{2(m-1)} - \frac{H_K}{K}. 
\end{align*}
This completes the proof.
\end{proof}

Note that in the above proof we measure the quality of our assignment
against, a perhaps-impossible, perfect solution, where each agent is
assigned to his or her top alternative.  This means that for
relatively large $m$ and $K$, and small $\frac{K}{m}$ ratio, the
algorithm can achieve a close-to-ideal solution irrespective of the
voters' preference orders. We believe that this is an argument in
favor of using Monroe's system in multiwinner elections.  On the flip
side, to obtain a better approximation ratio, we would have to use a
more involved bound on the quality of the optimal solution. To see
that this is the case, form an instance $I$ of
$\alpha_\bordadec$-\textsc{SU-Monroe} with $n$ agents and $m$
alternatives, where all the agents have the same preference order, and
where we seek to elect $K$ candidates (and where $K$ divides $n$).  It
is easy to see that each solution that assigns the $K$ universally
top-ranked alternatives to the agents is optimal.  Thus the total
dissatisfaction of the agents in the optimal solution is:
\begin{align*}
  \frac{n}{K}\left( (m-1) + \cdots + (m-K) \right) &= \frac{n}{K}
  \left(\frac{K(2m-K-1)}{2}\right) 
   = n(m-1) \left( 1 - \frac{K-1}{2(m-1)} \right).
\end{align*}
By taking large enough $m$ and $K$ (even for a fixed value of
$\frac{m}{K}$), the fraction $1 - \frac{K-1}{2(m-1)}$ can be
arbitrarily close to the approximation ratio of our algorithm (the
reasoning here is somewhat in the spirit of the idea of identifying
maximally robust elections, as studied by Shiryaev, Yu, and
Elkind~\cite{shi-yu-elk:t:robust-winners}).

For small values of $K$, it is possible that the $\frac{H_K}{K}$ part
of our approximation ratio would dominate the $\frac{K-1}{2(m-1)}$
part.  In such cases we can use the result of Betzler et
al.~\cite{fullyProportionalRepr}, who showed that for each fixed
constant $K$, $\alpha_\bordadec$-\textsc{SU-Monroe} can be solved in
polynomial time. Thus, for the finite number of cases where
$\frac{H_K}{K}$ is too large, we can solve the problem optimally using
their algorithm.
%
%
In consequence, the quality of the solution produced by
Algorithm~A 
most strongly depends on the ratio
$\frac{K-1}{2(m-1)}$. In most cases we can expect it to be small (for
example, in Polish parliamentary elections $K = 460$ and $m \approx
6000$; in this case the greedy algorithm's approximation ratio is
about $0.96$).

Our algorithm has one more great advantage: Since it focuses on the
top parts of voters' preference orders, it can achieve very good
results even if the voters submit so-called truncated ballots (that
is, if they rank some of their top alternatives only). Below we
present the formal analysis of the algorithm's approximation ratio for
this case. Unfortunately, we did not obtain a closed form formula and,
instead, we present the guaranteed approximation ratio as a sum, in
Proposition~\ref{prop:monTruncated} below. We also present the
relation between the fraction of the top alternatives ranked by each
of the voters and the approximation ratio for few values of $m$ and
$K$ in Figure~\ref{fig:monroe_truncated}.

\begin{proposition}\label{prop:monTruncated}
  Let $P$ be the number of top positions in the agents' preference
  orders that are known by the algorithm. In this case
  Algorithm~A 
  is a polynomial-time
  $r$-approximation algorithm for
  $\alpha_{\bordadec}$-\textsc{SU-Monroe}, where:
  \begin{align*}
    & r = \sum_{i = 0}^{K-1} \frac{1}{n(m-1)} \max(r(i), 0) \\
  \text{and}\\
    & r(i) = \left\{
	\begin{array}{l l}
          \frac{n}{K}(m - i - \frac{m-i}{K-i}) & \quad \text{if $\left(i + \frac{m-i}{K-i}\right) \leq P$,}\\\\
          \frac{n}{K}\frac{(K-i)(m-i)}{4} & \quad \text{if $\left(i + \frac{m-i}{K-i}\right) > P$ and $(2P-m) \geq i \geq (K-2)$,} \\\\
          \frac{n}{K}\frac{(m-P)(K-i)(P-i)}{m-i} & \quad \text{otherwise.}
	\end{array} \right.
    \end{align*}
  \end{proposition}
  \begin{proof}
    We use the same approach as in the proof of
    Lemma~\ref{lemma:greedy}, except that we adjust our estimates of
    voters' satisfaction. Consider a situation after some $i$'th
    iteration of the algorithm's outer loop ($i=0$ if we are before the
    first iteration). If $i + \frac{m-i}{K-i}
    \leq P$, then we can use the same lower bound for the satisfaction 
    of the agents assigned in the $(i+1)$'th iteration as in the 
    proof of Lemma~\ref{lemma:greedy}. That is, the agents assigned in the
    $(i+1)$'th iteration will have satisfaction at least 
    $r_1(i) = \frac{n}{K} \cdot (m - i - \frac{m-i}{K-i})$.

    For the case where $i + \frac{m-i}{K-i} > P$, the bound from
    Lemma~\ref{lemma:greedy} does not hold, but we can use a similar
    approach to find a different one. Let $P_x \leq P$ be some
    positive integer. We are interested in the number $x$ of not-yet
    assigned agents who rank some not-yet-selected alternative among
    their top $P_x$ positions (after the $i$'th iteration). Similarly
    as in the proof of Lemma~\ref{lemma:greedy}, using the pigeonhole principle
    we note that:
    \begin{align*}
      x \geq \frac{1}{m-i} \left(n-i\frac{n}{K}\right)(P_x-i) =
      \frac{n}{K} \cdot \frac{(K-i)(P_x-i)}{m-i} \textrm{.}
    \end{align*}
    Thus, the satisfaction of the agents assigned in the $(i+1)$'th
    iteration is at least:
    \begin{equation}\label{equ:sat-bound}
      \min\left(x, \frac{n}{K}\right)(m - P_x) = \frac{n}{K} \cdot
      (m-P_x) \min\left(\frac{(K-i)(P_x-i)}{m-i}, 1\right) \textrm{.}
    \end{equation}
    The case $\frac{(K-i)(P_x-i)}{m-i} \geq 1$ (or, equivalently, $i +
    \frac{m-i}{K-i} \leq P_x$) implies that $i + \frac{m-i}{K-i} \leq
    P$ and for this case we lower-bound agents' satisfaction by $r_1(i)$.
    For the case where $\frac{(K-i)(P_x-i)}{m-i} \leq 1$,
    i.e. where $i + \frac{m-i}{K-i} \geq P_x$, equation \eqref{equ:sat-bound}
    simplifies to:
    \begin{align}\label{equ:sat-bound-2}
      \frac{n}{K} \cdot (m-P_x) \cdot \frac{(K-i)(P_x-i)}{m-i} \textrm{.}
    \end{align}
    We use this estimate for the satisfaction of the agents assigned in 
    the $(i+1)$'th iteration for the cases where
    (a) $i + \frac{m-i}{K-i} \geq \frac{m+i}{2}$ and (b) $\frac{m+i}{2} \leq P$
    (or, equivalently, $(2P-m) \geq i \geq (K-2)$). In this case we estimate \eqref{equ:sat-bound-2}
    as follows:
    \begin{align*}
      \frac{n}{K} \cdot (m-P_x) \cdot \frac{(K-i)(P_x-i)}{m-i} 
      & \geq \frac{n}{K} \cdot (m-\frac{m+i}{2}) \cdot \frac{(K-i)(\frac{m+i}{2}-i)}{m-i}\\
      & = \frac{n}{K} \cdot
      \frac{(K-i)(m-i)^2}{4(m-i)} = \frac{n}{K} \cdot
      \frac{(K-i)(m-i)}{4}\textrm{.}
    \end{align*}
    For the remaining cases, we set $P_x = P$ and \eqref{equ:sat-bound-2} becomes:
    \begin{align*}
      \frac{n}{K} \cdot \frac{(m-P)(K-i)(P-i)}{m-i} \textrm{.}  
    \end{align*} 
    Naturally, we replace our estimates by $0$ whenever they become negative.

    To complete the proof, it suffice to, as in the proof of Lemma~\ref{lemma:greedy},
    note that $(m-1)n$ is an upper bound on the satisfaction achieved by the
    optimal solution.
\end{proof}

\begin{figure*}[!t!h]
\begin{minipage}[h]{0.48\linewidth}
  \centering
  \includegraphics[width=\textwidth]{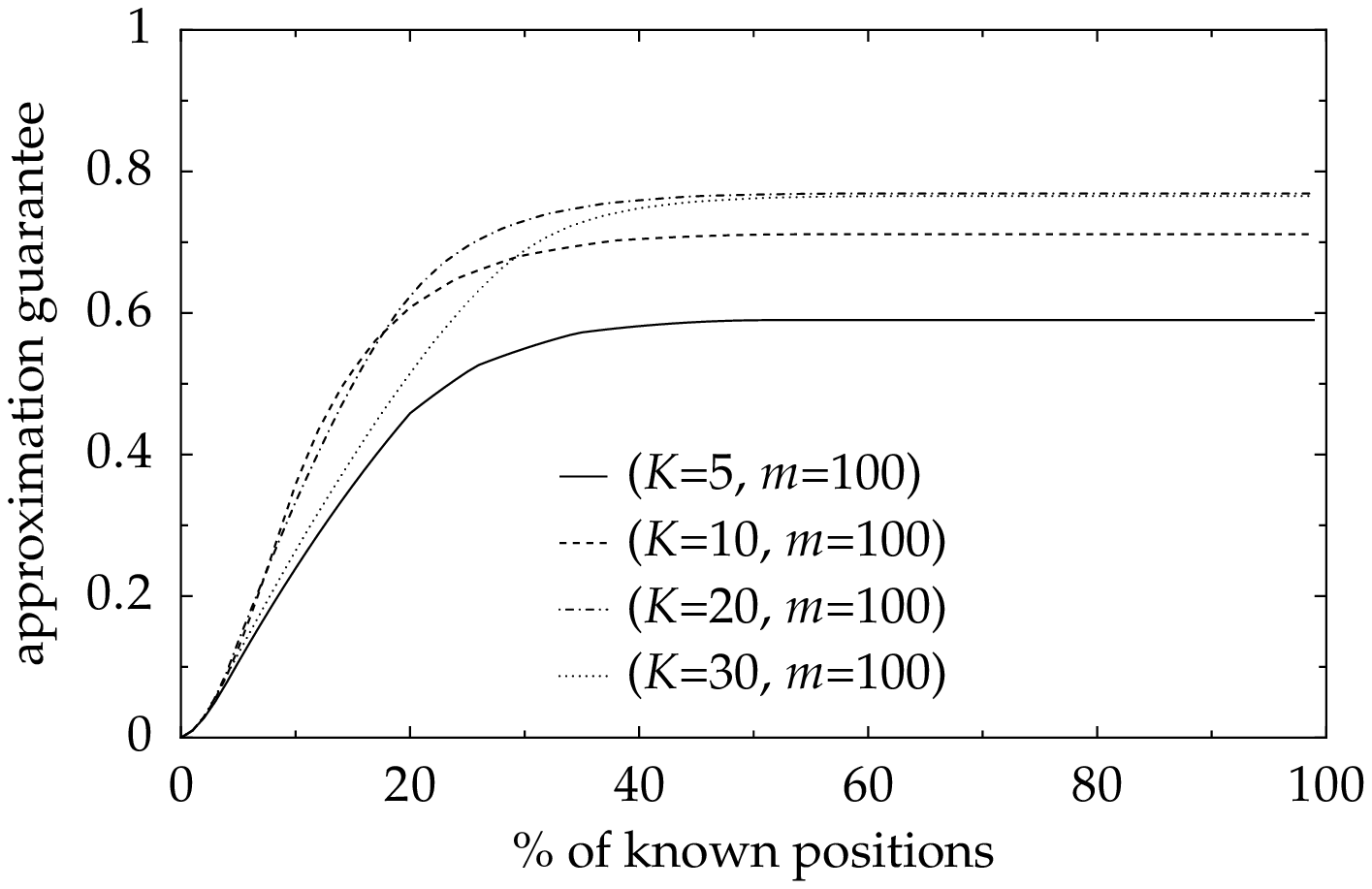}
\end{minipage}
\hspace{0.3cm}
\begin{minipage}[h]{0.48\linewidth}
  \centering
  \includegraphics[width=\textwidth]{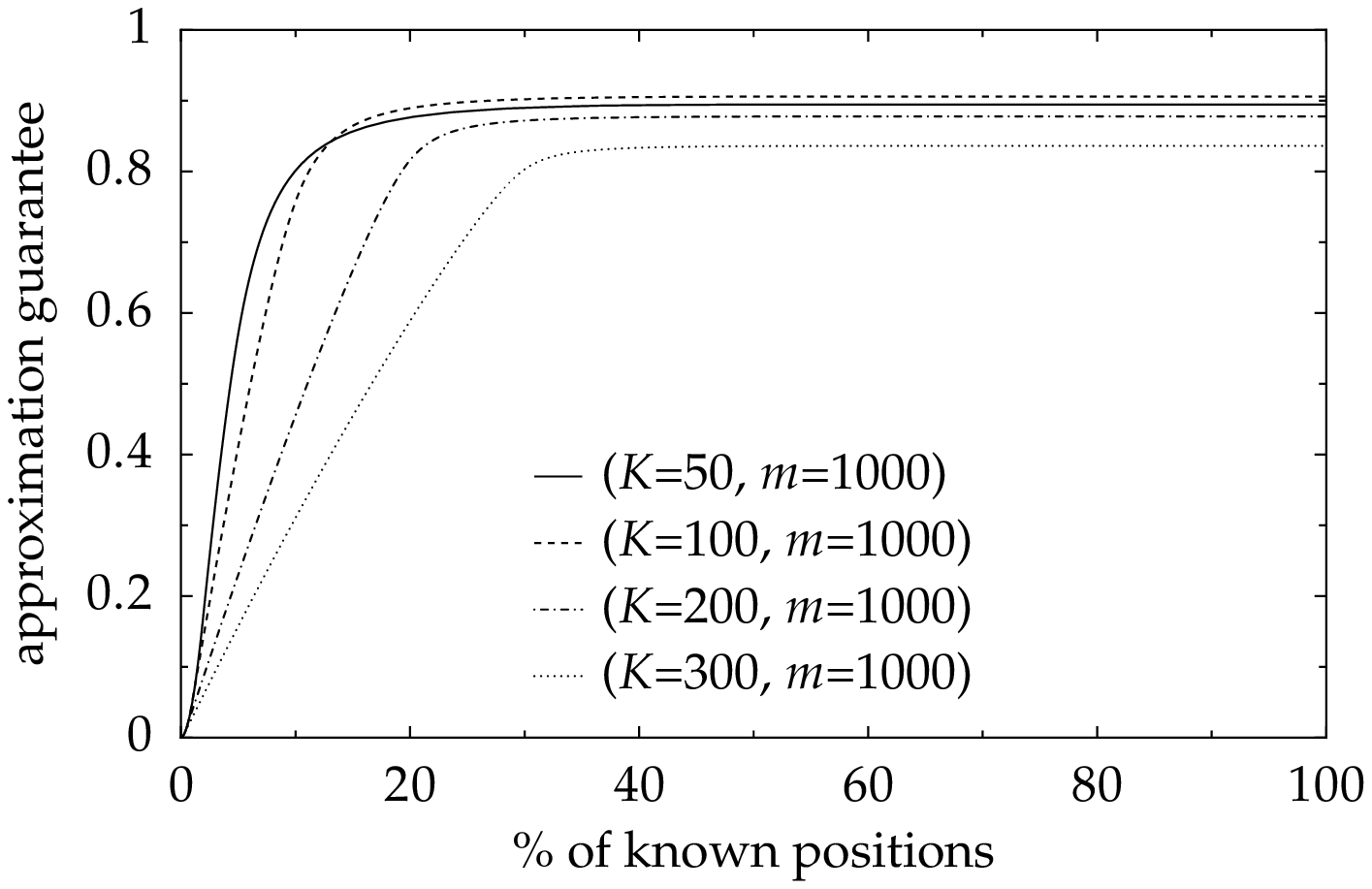}
\end{minipage}
\caption{The relation between the percentage of the known positions and the approximation ratio of Algorithm A for $\alpha_{\bordadec}$-\textsc{SU-Monroe}.}
\label{fig:monroe_truncated}
\end{figure*}

For example, for the case of Polish parliamentary elections ($K=460$
and $m=6000$), to achieve $90\%$ of voters' optimal satisfaction, each
voter would have to rank about $8.7\%$ of the candidates.

Our results show that for most settings there is very little reason to
ask the agents to rank all the alternatives. Using
Proposition~\ref{prop:monTruncated}, election designers can estimate
how many alternatives the agents should rank to obtain a particular
level of satisfaction and, since computing preference orders can be
expensive for the agents, this way can save a large amount of effort.

\subsection{Algorithm B (Monroe)}

There are simple ways in which we can improve the quality of the
assignments produced by Algorithm~A. For example, our Algorithm~B
first runs Algorithm~A and then, using
Proposition~\ref{prop:assignment}, optimally reassigns the
alternatives to the voters. As shown in Section~\ref{sec:experiments},
this very simple trick turns out to noticeably improve the results
of the algorithm in practice (and, of course, the theoretical
approximation guarantees of Algorithm~A carry over to Algorithm~B).

\subsection{Algorithm C (Monroe, CC)}

Algorithm~C is a further heuristic improvement over Algorithm~B.  This
time the idea is that instead of keeping only one partial function
$\Phi$ that is iteratively extended up to the full assignment, we keep
a list of up to $d$ partial assignment functions, where $d$ is a
parameter of the algorithm. At each iteration, for each assignment
function $\Phi$ among the $d$ stored ones and for each alternative $a$
that does not yet have agents assigned to by this $\Phi$, we compute
an optimal extension of this $\Phi$ that assigns agents to $a$. As a
result we obtain possibly more than $d$ (partial) assignment
functions. For the next iteration we keep those $d$ of them that give
highest satisfaction.

We provide pseudocode for Algorithm~C in
Figure~\ref{alg:greedyImpr}. If we take $d=1$, we obtain
Algorithm~B. If we also disregard the last two lines prior to
returning the solution, we obtain Algorithm~A.

Algorithm~C can also be adapted for the Chamberlin--Courant rule.
The only difference concerns creating the assignment functions:
we replace the contents of the first foreach loop with the 
following code:
\begin{flushleft}
  \small
  \hspace{37pt} \ForEach{$a_{i} \in A \setminus \Phi^{\rightarrow}$}{
  \hspace{37pt}   $\Phi' \leftarrow \Phi$\\
  \hspace{37pt}   \ForEach{$j \in N$}{    
  \hspace{37pt}     \If{agent $j$ prefers $a_i$ to $\Phi'(j)$}
    {\hspace{37pt} $\Phi'(j) \leftarrow a_i$}
    }
  \hspace{37pt} $newPar$.push($\Phi'$) \\
  }
\end{flushleft}

\noindent
Note that, for the case of the Chamberlin--Courant rule, Algorithm~C can
also be seen as a generalization of Algorithm GM that we will discuss
later in Section~\ref{alg:gm}.

\begin{figure}[t]
\begin{algorithm}[H]
   \small
   \SetAlCapFnt{\small}
   \KwNotation{We use the same notation as in Algorithm~A;\\
          $\hspace{31pt}$ $Par$ $\leftarrow$ a list of partial representation functions}

   $Par = []$ \\
   $Par$.push($\{\}$) \\
   
   \For{$i\leftarrow 1$ \KwTo $K$}{
      $newPar = []$ \\
      \For{$\Phi \in Par$}{
          $bests \leftarrow \{\}$ \\
          \ForEach{$a_{i} \in A \setminus \Phi^{\rightarrow}$}{ 
              $agents \leftarrow$ sort $N \setminus \Phi^{\leftarrow}$  (agent $j$ precedes agent $k$ implies that $pos_{j}(a_{i}) \leq pos_{k}(a_{i})$) \\
              $bests[a_{i}] \leftarrow$ chose first $\lceil \frac{N}{K} \rceil$ elements of $agents$ \\
              $\Phi' \leftarrow \Phi$ \\
              \ForEach{$j \in bests[a_{i}]$}{
                 $\Phi'[j] \leftarrow a_{i}$ \\
              }
              $newPar$.push($\Phi'$) \\
          }
          sort $newPar$ according to descending order of the total satisfaction of the assigned agents  \\
          $Par \leftarrow$ chose first $d$ elements of $newPar$ \\
      }
   }
   \For{$\Phi \in Par$}{
      $\Phi \leftarrow$ compute the optimal representative function using an algorithm of Betzler et al.~\cite{fullyProportionalRepr} for the set of winners $\Phi^{\rightarrow}$
   }
   \Return{the best representative function from $Par$}
\end{algorithm}
\caption{The pseudocode for Algorithm~C.}\label{alg:greedyImpr}
\end{figure}

\subsection{Algorithm R (Monroe, CC)}

Algorithms~A,~B~and~C achieve very high approximation ratios for the
cases where $K$ is small relative to $m$. For the remaining cases,
where $K$ and $m$ are comparable, we can use a sampling-based
randomized algorithm (denoted as Algorithm~R) described below.  We
focus on the case of Monroe and we will briefly mention the case of CC
at the end.

The idea of this algorithm is to randomly pick $K$ alternatives and
match them optimally to the agents, using
Proposition~\ref{prop:assignment}. Naturally, such an algorithm might
be very unlucky and pick $K$ alternatives that all of the agents rank
low. Yet, if $K$ is comparable to $m$ then it is likely that such a
random sample would include a large chunk of some optimal solution.
In the lemma below, we asses the expected satisfaction obtained with a
single sampling step (relative to the satisfaction given by the
optimal solution) and the probability that a single sampling step
gives satisfaction close to the expected one.  Naturally, in practice
one should try several sampling steps and pick the one with the
highest satisfaction.


\begin{lemma}\label{lemma:randMonroe}
  A single sampling step of the randomized algorithm for
  $\alpha_\bordadec$-\textsc{SU-Monroe} achieves expected
  approximation ratio of $\frac{1}{2}(1 + \frac{K}{m} - \frac{K^2}{m^2-m} +
  \frac{K^3}{m^3-m^2})$. Let $p_{\epsilon}$ denote the
  probability that the relative deviation between the obtained total
  satisfaction and the expected total satisfaction is higher than
  $\epsilon$. Then for $K \geq 8$ we have $p_{\epsilon} \leq \exp \left(-
    \frac{K\epsilon^2}{128} \right)$.
\end{lemma}
\begin{proof}
  Let $N = [n]$ be the set of agents, $A = \{a_1, \ldots, a_m\}$ be
  the set of alternatives, and $V$ be the preference profile of the
  agents.  Let us fix some optimal solution $\Phi_\opt$ and let
  $A_\opt$ be the set of alternatives assigned to the agents in this
  solution.  For each $a_{i} \in A_{\opt}$, we write $\sat(a_{i})$ to
  denote the total satisfaction of the agents assigned to $a_{i}$ in
  $\Phi_\opt$. Naturally, we have $\sum_{a \in A_{\opt}} \sat(a) =
  \OPT$.  In a single sampling step, we choose uniformly at random a
  $K$-element subset $B$ of $A$.  Then, we form a solution $\Phi_B$ by
  matching the alternatives in $B$ optimally to the agents (via
  Proposition~\ref{prop:assignment}).  We write $K_{\opt}$ to denote
  the random variable equal to $\|A_\opt \cap B\|$, the number of
  sampled alternatives that belong to $A_{\opt}$. We define $p_{i} =
  \Pr(K_{\opt} = i)$. For each $j \in \{1, \ldots, K\}$, we write
  $X_j$ to denote the random variable equal to the total satisfaction
  of the agents assigned to the $j$'th alternative from the sample.
  We claim that for each $i$, $0 \leq i \leq K$, it holds that:
  \begin{align*}
    \E\left(\sum_{j=1}^{K}X_{j} \,\Bigg|\, K_\opt = i\right) \geq \frac{i}{K}\OPT + \frac{m-i-1}{2}
    \cdot \left(n - i\frac{n}{K}\right).
  \end{align*}
  Why is this so? Given a sample $B$ that contains $i$ members of
  $A_\opt$, our algorithm's solution is at least as good as a solution
  that matches the alternatives from $B \cap A_\opt$ in the same way
  as $\Phi_\opt$, and the alternatives from $B - A_\opt$ in a random
  manner.  Since $K_\opt = i$ and each $a_j \in A_\opt$ has equal
  probability of being in the sample, it is easy to see that the
  expected value of $\sum_{a_j \in B \cap A_\opt}\sat(a_j)$ is $
  \frac{i}{K}\OPT$.
  %
  %
  After we allocate the agents from $B \cap A_\opt$, each of the
  remaining, unassigned agents has $m-i$ positions in his or her
  preference order where he ranks the agents from $A - A_\opt$.  For
  each unassigned agents, the average score value associated with
  these positions is at least $\frac{m-i-1}{2}$ (this is so, because
  in the worst case the agent could rank the alternatives from $B \cap
  A_\opt$ in the top $i$ positions). There are $(n - i\frac{n}{K})$
  such not yet assigned agents and so the expected total satisfaction
  from assigning them randomly to the alternatives is $\frac{m-i-1}{2}
  \cdot (n - i\frac{n}{K})$. This proves our bound on the expected
  satisfaction of a solution yielded by optimally matching a random
  sample of $K$ alternatives.


  Since $\OPT$ is upper bounded by $(m-1)n$ (consider a
  possibly-nonexistent solution where every agent is assigned to his
  or her top preference), we get that:
  \[
    \E\left(\sum_{j=1}^{K}X_{j} | K_\opt = i\right) \geq  \frac{i}{K}\OPT + \frac{m-i-1}{2(m-1)} \cdot \left(1 - \frac{i}{K}\right)\OPT.
  \]
  We can compute the unconditional expected satisfaction of $\Phi_B$
  as follows:
  \begin{align*}
    \E\left(\sum_{j=1}^{K}X_{j}\right) & = \sum_{i=0}^{K}p_{i}\E\left(\sum_{j=1}^{K}X_{j} | K_\opt = i\right) \\
    & \geq \sum_{i=0}^{K}p_{i}\left(\frac{i}{K}\OPT + \frac{m-i-1}{2(m-1)}
    \cdot \left(1 - \frac{i}{K}\right)\OPT\right).
  \end{align*}
  Since $\sum_{i=1}^{K}p_{i} \cdot i$ is the expected number of the
  alternatives in $A_{\opt}$, we have that $ \sum_{i=1}^{K}p_{i} \cdot
  i = \frac{K^2}{m}$ (one can think of summing the expected values
  of $K$ indicator random variables; one for each element of $A_\opt$,
  taking the value $1$ if a given alternative is selected and taking
  the value $0$ otherwise).  Further, from the generalized mean
  inequality we obtain $\sum_{i=1}^{K}p_{i} \cdot i^{2} \geq
  \left(\frac{K^2}{m}\right)^{2}.$ In consequence, through routine
  calculation, we get that:
  \begin{align*}
     \E\left(\sum_{j=1}^{K}X_{j}\right) 
     & \geq \left(\frac{K}{m}\OPT + \frac{m^2 - K^2 -m}{2m(m-1)} \cdot \left(1 - \frac{K}{m}\right)\OPT\right) \\
     & = \frac{\OPT}{2}\left(1 + \frac{K}{m} - \frac{K^2}{m^2-m} +
     \frac{K^3}{m^3-m^2}\right).
  \end{align*}

  It remains to assess the probability that the total satisfaction
  obtained through $\Phi_B$ is close to its expected value.  Since
  $X_{j} \in \langle 0, \frac{(m-1)n}{K} \rangle$, from Hoeffding's
  inequality we get:
  \begin{align*}
    p_{\epsilon} & = \Pr\left(\left|\sum_{j=1}^{K}X_{j} - \E(\sum_{j=1}^{K}X_{j})\right| \geq \epsilon \E(\sum_{j=1}^{K}X_{j})\right) \\
    & \leq \exp \left(- \frac{2\epsilon^2 (\E(\sum_{j=1}^{K}X_{j}))^{2}}{K(\frac{(m-1)n}{K})^2}
    \right) = \exp \left(- \frac{K\epsilon^2 (\E(\sum_{j=1}^{K}X_{j}))^{2}}{((m-1)n)^2} \right)
  \end{align*}
  We note that since $\frac{K}{m}-\frac{K^2}{m^2-m} \geq 0$, our
  previous calculations show that $\E(\sum_{j=1}^{K}X_{j}) \geq
  \frac{\OPT}{2}$. Further, for $K \geq 8$, Lemma~\ref{lemma:greedy}
  (and the fact that in its proof we upper-bound $\OPT$ to be $(m-1)n$)
  gives that $\OPT \geq \frac{mn}{8}$. Thus $p_{\epsilon} \leq \exp
  \left(- \frac{K\epsilon^2}{128} \right)$.  This completes the proof.
\end{proof}

In the next theorem we will see that to have a high chance of
obtaining a high quality assignment, we need to repeat the sampling
step many times. Thus, for practical purposes, by Algorithm~R we
mean an algorithm that repreats the sampling process a given number of
times (this parameter is given as input) and returns the best solution
found (the assignment is created using
Proposition~\ref{prop:assignment}).

The threshold for $\frac{K}{m}$, where the sampling step is (in
expectation) better than the greedy algorithm is about 0.57.  Thus, by
combining the two algorithms, we can guarantee an expected
approximation ratio of $0.715 - \epsilon$, for each fixed constant
$\epsilon$.  The pseudo-code of the combination of the two algorithms
(Algorithm~AR) is presented in Figure~\ref{alg:combination}.

\begin{theorem}\label{thm:combMonroe}
  For each fixed $\epsilon$, Algorithm~AR provides
  a $(0.715-\epsilon)$-approximate solution for the problem
  $\alpha_\bordadec$-\textsc{SU-Monroe} with probability
  $\lambda$ in time polynomial with respect to the input instance
  size and $-\log(1-\lambda)$.
\end{theorem}
\begin{proof}
  Let $\epsilon$ be a fixed constant. We are given an instance $I$ of
  $\alpha_\bordadec$-\textsc{SU-Monroe}. If $m \leq 1 +
  \frac{2}{\epsilon}$, we solve $I$ using a brute-force algorithm
  (note that in this case the number of alternatives is at most a
  fixed constant). Similarly, if $\frac{H_K}{K} \geq
  \frac{\epsilon}{2}$ then we use the exact algorithm of Betzler et
  al.~\cite{fullyProportionalRepr} for a fixed value of $K$ (note that
  in this case $K$ is no greater than a certain fixed constant). We do
  the same if $K \leq 8$.

  On the other hand, if neither of the above conditions hold, we try
  both Algorithm~A
  and a number of runs of the
  sampling-based algorithm. It is easy to check through routine
  calculation that if $\frac{H_K}{K} \leq \frac{\epsilon}{2}$ and $m >
  1 + \frac{2}{\epsilon}$ then Algorithm~A
  achieves
  approximation ratio no worse than $(1 -\frac{K}{2m} - \epsilon)$.
  %
  %
  %
  %
  We run the sampling-based algorithm $\frac{-512 \log (1 -
    \lambda)}{K\epsilon^2}$ times. The probability that a single run
  fails to find a solution with approximation ratio at least
  $\frac{1}{2}(1 + \frac{K}{m} - \frac{K^2}{m^2-m} +
  \frac{K^3}{m^3-m^2}) - \frac{\epsilon}{2}$ is
  $p_{\frac{\epsilon}{2}} \leq \exp \left(- \frac{K\epsilon^2}{4 \cdot
      128} \right)$. Thus, the probability that at least one run will
  find a solution with at least this approximation ratio is at least:
  \begin{align*}
    1 - p_{\frac{\epsilon}{2}}^{\frac{-512 \log (1 - \lambda)}{K\epsilon^2}} = 1
    - \exp \left(-\frac{K\epsilon^2}{4\cdot128} \cdot \frac{-512 \log
        (1 - \lambda)}{K\epsilon^2} \right) = \lambda.
  \end{align*}
  Since $m \leq 1 + \frac{2}{\epsilon}$, by routine calculation we see
  that the sampling-based algorithm with probability $\lambda$ finds a
  solution with approximation ratio at least $\frac{1}{2}(1 +
  \frac{K}{m} - \frac{K^2}{m^2} + \frac{K^3}{m^3}) - \epsilon$.  By
  solving the equality:
  \begin{align*}
    \frac{1}{2}\left(1 + \frac{K}{m} - \frac{K^2}{m^2} + \frac{K^3}{m^3}\right) =
    1 -\frac{K}{2m}
  \end{align*}
  we can find the value of $\frac{K}{m}$ for which the two algorithms
  give the same approximation ratio.  By substituting $x =
  \frac{K}{m}$ we get equality $1 + x - x^2 + x^3 = 2 - x$.  One can
  calculate that this equality has a single solution within $\langle
  0,1 \rangle$ and that this solution is $x \approx 0.57$. For this
  $x$ both algorithms guarantee approximation ratio of $0.715 -
  \epsilon$. For $x < 0.57$ the deterministic algorithm guarantees a
  better approximation ratio and for $x > 0.57$, the randomized
  algorithm does better.
\end{proof}

\begin{figure}[t]
\begin{algorithm}[H]
   \small
   \SetAlCapFnt{\small}
   \SetKwInOut{Parameters}{Parameters}
   \KwNotation{We use the same notation as in Algorithm~\ref{alg:greedy}; $\w(\cdot)$ denotes Lambert's W-Function.\\}
   \Parameters{$\lambda$ $\leftarrow$ required probability of achieving the approximation ratio equal $0.715 - \epsilon$}
   \If{$\frac{H_K}{K} \geq
  \frac{\epsilon}{2}$ 
  or $K \leq 8$}{compute the optimal solution using an algorithm of Betzler et al.~\cite{fullyProportionalRepr} and return.}
   \If{$m \leq 1 + \frac{2}{\epsilon}$}{compute the optimal solution using a simple brute force algorithm and return.} 
   $\Phi_1 \leftarrow$ solution returned by Algorithm~A\\ 
   $\Phi_2 \leftarrow$ run the sampling-based algorithm $\frac{-512 \log (1 - \lambda)}{K\epsilon^2}$ times; select the assignment of the best quality \\
   return the better assignment among $\Phi_1$ and $\Phi_2$
\end{algorithm}
\caption{Algorithm~AR---combination of Algorithms A and R.}\label{alg:combination}
\end{figure}

Let us now consider the case of CC. It is just as natural to try a
sampling-based approach for solving $\alpha_\bordadec$-\textsc{SU-CC},
as we did for the Monroe variant. Indeed, as recently (and
independently) observed by Oren~\cite{ore:p:cc}, this leads to a
randomized algorithm with expected approximation ratio of $(1 -
\frac{1}{K+1})(1+\frac{1}{m})$. However, since we will later see an
effective, deterministic, polynomial-time approximation scheme for
$\alpha_\bordadec$-\textsc{SU-CC}, there is little reason to explore
the sampling based approach.

\subsection{Algorithm GM (Monroe, CC)}\label{alg:gm}

Algorithm~GM (greedy marginal improvement) was introduced by Lu and
Boutilier for the case of the Chamberlin--Courant rule.  Here we
generalize it to apply to Monroe's rule as well, and we show that it
is a $1-\frac{1}{e}$ approximation algorithm for
$\alpha$-\textsc{SU-Monroe}.  We point out that this approximation
result for Monroe rule applies to all non-decreasing PSFs
$\alpha$. For the Monroe rule, the algorithm can be viewed as an
extension of Algorithm~B.

The algorithm proceeds as
follows. We start with an emtpy set $S$.  Then we execute $K$
iterations. In each iteration we find an alternative $a$ that is not
assigned to agents yet, and that maximizes the value $\Phi^{S \cup
  \{a\}}_{\alpha}$. (A certain disadvantage of this algorithm for the case of
Monroe is that it requires a large number of computations of
$\Phi^S_{\alpha}$; since in Monroe's rule each alternative can be assigned at most $\frac{n}{K}$ agents in the
partial assignment $\Phi^S_{\alpha}$, computation of $\Phi^S_{\alpha}$ is a slow process based on min-cost/max-flow
algorithm.) We provide the pseudocode for Algorithm~GM in
Figure~\ref{alg:greedyOptImpr}.

\begin{figure}[t]
\begin{algorithm}[H]
  \small \SetAlCapFnt{\small}
  \KwNotation{$\Phi^{S}_{\alpha}$---the partial assignement that assigns a single alternative to at most $\lceil \frac{n}{K} \rceil$ agents, that assigns to the agents only the alternatives from $S$, and that maximizes the utilitarian satisfaction $\ell_{1}^{\alpha}(\Phi^S_{\alpha})$.\\}
  $S \leftarrow \emptyset$ \\
  \For{$i\leftarrow 1$ \KwTo $K$}{
    $a \leftarrow \mathrm{argmax}_{a \in A \setminus S} \ell_1^{\alpha}(\Phi^{S \cup \{a\}}_{\alpha})$ \\
    $S \leftarrow S \cup \{a\}$ \\
  } \Return{$\Phi^{S}_{\alpha}$}
\end{algorithm}
\caption{Pseudocode for Algorithm~GM.}\label{alg:greedyOptImpr}
\end{figure}

\begin{theorem}\label{thm:gmMonroe}
  For any non-decreasing positional scoring function $\alpha$
  Algorithm~GM is an $(1 - \frac{1}{e})$-approximation algorithm for $\alpha$-\textsc{SU-Monroe}.
\end{theorem}
\begin{proof}
  The proof follows by applying the powerful result of Nemhauser et
  al.~\cite{submodular}, which says that greedy algorithms achieve
  $1-\frac{1}{e}$ approximation ratio when used to optimize
  nondecreasing submodular functions (we explain these notions
  formally below). The main challenge in the proof is to define a
  function that, on one hand, satisfies the conditions of Nemhauser et
  al.'s result, and, on the other, models solutions for
  $\alpha$-\textsc{SU-Monroe}.

  Let $A$ be a set of alternatives, $N = [n]$ be a set of agents with
  preferences over $A$, $\alpha$ be an $\|A\|$-candidate DPSF, and $K
  \leq \|A\|$ be the number of representatives that we want to elect.
  We consider function $z: 2^{A} \rightarrow \naturals$ defined, for
  each set $S$, $S \subseteq A$ and $\|S\| \leq K$, as $z(S) =
  \ell_{1}^{\alpha}(\Phi_{\alpha}^{S})$.  Clearly, $z(S)$ is
  nondecreasing (that is, for each two sets $A$ and $B$, if $A
  \subseteq B$ and $\|B\| \leq K$ then $z(A) \leq z(B)$. 
  Since $\mathrm{argmax}_{S \subset A, \|S\| = K} z(S)$ is the set of
  winners under $\alpha$-Monroe and since Algorithm~GM builds the
  solution iteratively by greedily extending initially empty set $S$
  so that each iteration increases the value of $z(S)$ maximally, if
  $z$ were submodular then by the results of Nemhauser et
  al.~\cite{submodular} we would get that Algorithm~GM is a
  $(1-\frac{1}{e})$-approximation algorithm. Thus, our goal is to show
  that $z$ is submodular.

  Formally, our goal is to show that for each two sets $S$ and $T$, $S
  \subset T$, and each alternative $a \notin T$ it holds that $z(S
  \cup \{a\}) - z(S) \geq z(T \cup \{a\}) - z(T)$ (this is the formal
  definition of submodularity). First, we introduce a notion that
  generalizes the notion of a partial set of winners $S$. Let $s: A
  \rightarrow \naturals$ denote a function that assigns a capacity to
  each alternative (i.e., $s$ gives a bound on the number of agents
  that a given alternative can represent). Intuitively, each set $S$,
  $S \subseteq A$, corresponds to the capacity function that assigns
  $\lceil \frac{n}{k} \rceil$ to each alternative $a \in S$ and 0 to
  each $a \notin S$.  Given a capacity function $s$, we define a
  partial solution $\Phi_{\alpha}^{s}$ to be one that maximizes the
  total satisfaction of the agents and that satisfies the new capacity
  constraints: $\forall_{a \in S} \|(\Phi_{\alpha}^{s})^{-1}(a)\| \leq
  s(a)$. To simplify notation, we write $s \cup \{a\}$ to denote the
  function such that $(s \cup \{a\})(a) = s(a) + 1$ and $\forall_{a'
    \in S \setminus\{a\}} (s \cup \{a\})(a') = s(a')$. (Analogously,
  we interpret $s \setminus \{a\}$ as subtracting one from the
  capacity for $a$; provided it is nonzero.)  Also, by $s \leq t$ we
  mean that $\forall_{a \in A} s(a) \leq t(a)$. We extend our function
  $z$ to allow us to consider a subset of the agents only.  For each
  subset $N'$ of the agents and each capacity function $s$, we define
  $z_{N'}(s)$ to be the satisfaction of the agents in $N'$ obtained
  under $\Phi_{\alpha}^{s}$. We will now prove a stronger variant of
  submodularity for our extended $z$. That is, we will show that for
  each two capacity functions $s$ and $t$ it holds that:
  \begin{equation}\label{eq:submodularity}
    s \leq t \Rightarrow z_{N}(s \cup \{a\}) -  z_{N}(s) \geq z_{N}(t \cup \{a\}) -  z_{N}(t).
  \end{equation}
  Our proof is by induction on $N$.  Clearly,
  Equation~\eqref{eq:submodularity} holds for $N' = \emptyset$. Now,
  assuming that Equation~\eqref{eq:submodularity} holds for every $N'
  \subset N$ we will prove its correctness for $N$.  Let $i$ denote an
  agent such that $\Phi_{\alpha}^{t \cup \{a\}}(i) = a$ (if there is
  no such agent then clearly the equation holds). Let $a_{s} =
  \Phi_{\alpha}^{s}(i)$ and $a_{t} = \Phi_{\alpha}^{t}(i)$. We have:
  \begin{align*}
    z_{N}(t \cup \{a\}) -  z_{N}(t) =   \alpha(\pos_i(a)) +  z_{N \setminus \{i\}}(t) - \alpha(\pos_i(a_{t})) - z_{N \setminus \{i\}}(t \setminus \{a_{t}\}). 
  \end{align*}
  We also have:
  \begin{align*}
    z_{N}(s \cup \{a\}) -  z_{N}(s) \geq \alpha(\pos_i(a)) + z_{N \setminus \{i\}}(s) - \alpha(\pos_i(a_{s})) - z_{N \setminus \{i\}}(s \setminus \{a_{s}\}).
  \end{align*}
  Since $\Phi_{\alpha}^{t}$ describes an optimal representation function under
  the capacity restrictions $t$, we have that:
  \begin{align*}
    \alpha(\pos_i(a_t)) + z_{N \setminus \{i\}}(t \setminus a_{t})
    \geq \alpha(\pos_i(a_s)) + z_{N \setminus \{i\}}(t \setminus \{a_{s}\}).
  \end{align*}
  Finally, from the inductive hypothesis for $N' = N \setminus \{i\}$ we have:
  \begin{align*}
    z_{N \setminus \{i\}}(s) - z_{N \setminus \{i\}}(s \setminus \{a_{s}\}) \geq
    z_{N \setminus \{i\}}(t) - z_{N \setminus \{i\}}(t \setminus \{a_{s}\}).
  \end{align*}
  By combining these inequalities we get:
  \begin{align*}
    z_{N}(s \cup \{a\}) -  z_{N}(s) & \geq \alpha(\pos_i(a)) + z_{N \setminus \{i\}}(s) - (\alpha(\pos_i(a_{s})) + z_{N \setminus \{i\}}(s \setminus \{a_{s}\})) \\
    & \geq \alpha(\pos_i(a)) - \alpha(\pos_i(a_{s})) +  z_{N \setminus \{i\}}(t) - z_{N \setminus \{i\}}(t \setminus \{a_{s}\}) \\
    & \geq \alpha(\pos_i(a)) + z_{N \setminus \{i\}}(t) - \alpha(\pos_i(a_t)) -  z_{N \setminus \{i\}}(t \setminus \{a_{t}\}) \\
    & = z_{N}(t \cup \{a\}) - z_{N}(t).
  \end{align*}
  This completes the proof.
\end{proof}

Formally speaking, Algorithm~GM is never worse than Algorithm~A. For
Borda satisfaction function, it inherits the approximation guarantees
from Algorithm~A, and for other cases Theorem~\ref{thm:gmMonroe}
guarantees approximation ratio $1-\frac{1}{e}$ (we do not know of any
guarantees for Algorithm~A for these cases). The comparison with
Algorithms~B and~C is not nearly as easy. Algorithm~GM is still likely
better than them for satisfaction functions significantly different
from Borda's, but for the Borda case our experiments show that
Algorithm~GM is much slower than Algorithms~B and~C and obtains almost
the same or slightly worse results (see
Section~\ref{sec:experiments}).




\subsection{Algorithm P (CC)}

The idea of our algorithm (presented in Figure~\ref{alg:greedy2}) is
to compute a certain value $x$ and to greedily compute an assignment
that (approximately) maximizes the number of agents assigned to one of
their top-$x$ alternatives.\footnote{This is very similar to the so-called
  MaxCover problem. Skowron and Faliszewski~\cite{sko-fal:t:max-cover}
  have discussed the connection of MaxCover to the winner determination problem 
  under the Chamberlin--Courant voting system (for approval-based
  satisfaction functions) and provided a number of FPT approximation
  schemes for it.}  If after this process some agent has no
alternative assigned, we assign him or her to his or her most
preferred alternative from those already picked.  Somewhat
surprisingly, it turns out that this greedy strategy achieves
high-quality results. (Recall that for nonnegative real numbers,
Lambert's W-function, $\w(x)$, is defined to be the solution of the
equation $x = \w(x)e^{\w(x)}$.)


\begin{figure}[t]
\begin{algorithm}[H]
   \small
   \SetAlCapFnt{\small}
   \KwNotation{We use the same notation as in Algorithm~C;
   $\hspace{21pt$} $\mathrm{num\_pos}_x(a) \leftarrow \|\{i \in [n] \setminus \Phi^{\leftarrow} : pos_i(a) \leq x \}\|$ (the number of not-yet assigned agents that rank alternative $a$ in one of their first $x$ positions)}
   $\Phi = \{\}$ \\
   $x = \lceil \frac{m\w(K)}{K} \rceil$ \\
   \For{$i\leftarrow 1$ \KwTo $K$}{
                        $a_{i} \leftarrow \mathrm{argmax}_{a \in A \setminus \Phi^{\rightarrow}} \mathrm{num\_pos}_x(a)$

      \ForEach{$j \in [n] \setminus \Phi^{\leftarrow}$}{
         \If{$pos_j(a_{i}) < x$}{
             $\Phi[j] \leftarrow a_{i}$ \\
         }
      }
   }
   \ForEach{$j \in A \setminus \Phi^{\leftarrow}$}{
       $a \leftarrow$ such server from $\Phi^{\rightarrow}$ that $\forall_{a' \in \Phi^{\rightarrow}} pos_{j}(a) \leq pos_{j}(a')$ \\
       $\Phi[j] \leftarrow a$ \\
   }
\end{algorithm}
\caption{The algorithm for $\alpha_\bordadec$-\textsc{SU-CC} (Algorithm~P).}\label{alg:greedy2}
\end{figure}

\begin{lemma}\label{lemma:greedyCC}
  Algorithm~P
  is a polynomial-time $(1 -
  \frac{2\w(K)}{K})$-approximation algorithm for
  $\alpha_\bordadec$-\textsc{SU-CC}.
\end{lemma}
\begin{proof}
  Let $x = \frac{m\w(K)}{K}$.  We will first give an inductive proof
  that, for each $i$, $0 \leq i \leq K$, after the $i$'th iteration of the outer loop at
  most $n(1- \frac{w(K)}{K})^{i}$ agents are unassigned. Based on this
  observation, we will derive the approximation ratio of our
  algorithm.

  For $i = 0$, the inductive hypothesis holds because $n(1-
  \frac{\w(K)}{K})^{0} = n$. For each $i$, let $n_{i}$ denote the
  number of unassigned agents after the $i$'th iteration. Thus, after
  the $i$'th iteration there are $n_i$ unassigned agents, each with
  $x$ unassigned alternatives among his or her top-$x$ ranked
  alternatives. As a result, at least one unassigned alternative is
  present in at least $\frac{n_{i}x}{m-i}$ of top-$x$ positions of
  unassigned agents. This means that after the $(i+1)$'st iteration
  the number of unassigned agents is:
  \begin{align*}
    n_{i+1} \leq n_{i} - \frac{n_{i}x}{m-i} \leq n_{i}\left(1 -
    \frac{x}{m}\right) = n_{i}\left(1 - \frac{\w(K)}{K}\right).
  \end{align*}
  If for a given $i$ the inductive hypothesis holds, that is, if $n
  _{i} \leq n\left(1- \frac{\w(K)}{K}\right)^{i}$, then:
  \begin{align*}
    n_{i+1} \leq n(1- \frac{\w(K)}{K})^{i}(1 -
    \frac{\w(K)}{K}) = n\left(1- \frac{\w(K)}{K}\right)^{i+1}.
  \end{align*}
  Thus the hypothesis holds and, as a result, we have that:
  \begin{align*}
    n_{k} \leq n\left(1- \frac{\w(K)}{K}\right)^{K} \leq
    n\left(\frac{1}{e}\right)^{\w(K)} = \frac{n\w(K)}{K}.
  \end{align*}
  Let $\Phi$ be the assignment computed by our algorithm.  To compare
  it against the optimal solution, it suffices to observe that the
  optimal solution has the value of satisfaction of at most $\OPT \leq (m-1)n$,
  that each agent selected during the first $K$ steps has satisfaction
  at least $m-x = m-\frac{m\w(K)}{K}$, and that the agents not
  assigned within the first $K$ steps have satisfaction no worse than
  $0$. Thus it holds that:
  \begin{align*}
    \frac{\ell_{1}^{\alpha_{\bordadec}}(\Phi)}{\OPT} & \geq \frac{(n - \frac{n\w(K)}{K})(m - \frac{m\w(K)}{K})}{(m-1)n} \\
    & \geq \left(1 - \frac{\w(K)}{K}\right)\left(1 - \frac{\w(K)}{K}\right) \geq 1 -
    \frac{2\w(K)}{K}.
  \end{align*}
  This completes the proof.
\end{proof}

Since for each $\epsilon > 0$ there is a value $K_\epsilon$ such that
for each $K > K_\epsilon$ it holds that $\frac{2\w(K)}{K} < \epsilon$,
and $\alpha_\bordadec$-\textsc{SU-CC} problem can be solved
optimally in polynomial time for each fixed constant $K$(see the work
of Betzler et al.~\cite{fullyProportionalRepr}), there is a
polynomial-time approximation scheme (PTAS) for
$\alpha_\bordadec$-\textsc{SU-CC} (i.e., a family of
algorithms such that for each fixed $r$, $0 < r < 1$, there is
a polynomial-time $r$-approximation algorithm for
$\alpha_\bordadec$-\textsc{SU-CC} in the family; note that in
PTASes we measure the running time by considering $r$ to be a
fixed constant).

\begin{theorem}\label{theorem:ptas}
  There is a PTAS for $\alpha_\bordadec$-\textsc{SU-CC}.
\end{theorem}

The idea used in Algorithm~P 
can also be used to address a generalized \textsc{SE-CC} problem. We
can consider the following relaxation of \textsc{SE-CC}: Instead of
requiring that each agent's satisfaction is lower-bounded by some
value, we ask that the satisfactions of a significant majority of the
agents are lower-bounded by a given value.  More formally, for a given
constant $\delta$, we introduce an additional quality metric:
\[ \ell_{\min}^{\delta,\alpha}(\Phi) = \mathrm{max}_{N' \subseteq N:
  \frac{||N|| - ||N'||}{||N||} \leq \delta}\mathrm{min}_{i \in
  N'}\alpha(pos_{i}(\Phi(i))).
\]
For a given $0 < \delta < 1$, by putting $x =
\frac{-m\ln(\delta)}{K}$, we get $(1 +
\frac{\ln(\delta)}{K})$-approximation algorithm for the
$\ell_{\min}^{\delta,\alpha}(\Phi)$ metric.

Finally, we show that Algorithm~P 
performs very well even if the voters cast truncated
ballots. Proposition~\ref{lemma:greedyCCTruncated} gives the relation
between the number of positions used by the algorithm and the
approximation ratio. In Figure~\ref{fig:cc_truncated} we show this
relation for some values of the parameters $m$ and $K$.

\begin{proposition}\label{lemma:greedyCCTruncated}
  Let $Q$ be the number of top positions in the agents' preference
  orders that are known by the algorithm ($Q \leq
  \frac{m\w(K)}{K}$).  Algorithm~P 
  that uses $x = Q$
  instead of $x = \lceil \frac{m\w(K)}{K} \rceil$ is a polynomial-time
  $\left(\frac{m-Q}{m-1}(1 - e^{-\frac{QK}{m}})\right)$-approximation
  algorithm for $\alpha_\bordadec$-\textsc{SU-CC}.
\end{proposition}
\begin{proof}
  Let $n_i$ denote the number of the agents not-yet-assigned until the
  $(i+1)$-th iteration of the algorithm. Using the same reasoning as
  in Lemma~\ref{lemma:greedyCC} we show that $n_i \leq n(1 -
  \frac{Q}{m})^{i}$. As before, our proof proceeds by induction on
  $i$. It is evident that the hypothesis is correct for $i=0$. Now,
  assuming that $n_i \leq n(1 - \frac{Q}{m})^{i}$, we assess $n_{i+1}$
  as follows:
\begin{align*}
  n_{i+1} \leq n_{i} - \frac{n_{i}Q}{m-i} \leq n_{i}\left(1
    -\frac{Q}{m}\right) \leq n\left(1 - \frac{Q}{m}\right)^{i+1}
  \textrm{.}
\end{align*}
This proves the hypothesis. Thus, we can bound $n_K$:
\begin{align*}
  n_{K} \leq n\left(1 - \frac{Q}{m}\right)^{K} \leq
  n\left(\frac{1}{e}\right)^{\frac{QK}{m}} \textrm{.}
\end{align*}
This means that the satisfaction of the assignment $\Phi$ returned by
our algorithm is at least:
\begin{align*}
  \ell_{1}^{\alpha_{\bordadec}}(\Phi) \geq (n - n_K)(m - Q) \geq
  n(m-Q)(1 - e^{-\frac{QK}{m}}) \textrm{.}
\end{align*}
In effect, it holds that:
\begin{align*}
  \frac{\ell_{1}^{\alpha_{\bordadec}}(\Phi)}{\OPT} \geq \frac{n(m-Q)(1
    - e^{-\frac{QK}{m}})}{n(m-1)} \geq \frac{m-Q}{m-1}\left(1 -
  e^{-\frac{QK}{m}}\right)\textrm{.}
\end{align*}
This completes the proof.
\end{proof}

\begin{figure*}[t!]
\begin{minipage}[h]{0.48\linewidth}
  \centering
  \includegraphics[width=\textwidth]{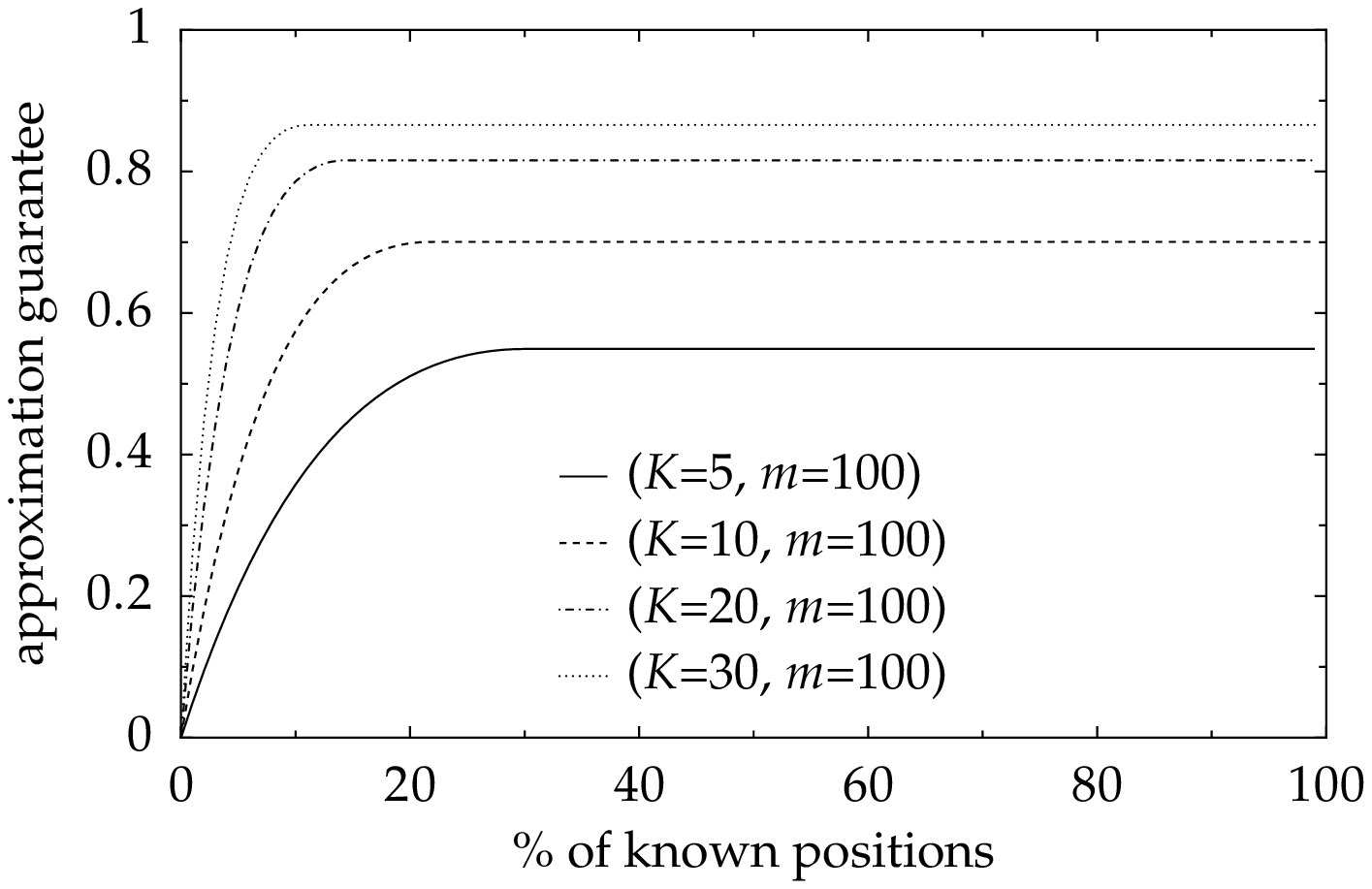}
\end{minipage}
\hspace{0.3cm}
\begin{minipage}[h]{0.48\linewidth}
  \centering
  \includegraphics[width=\textwidth]{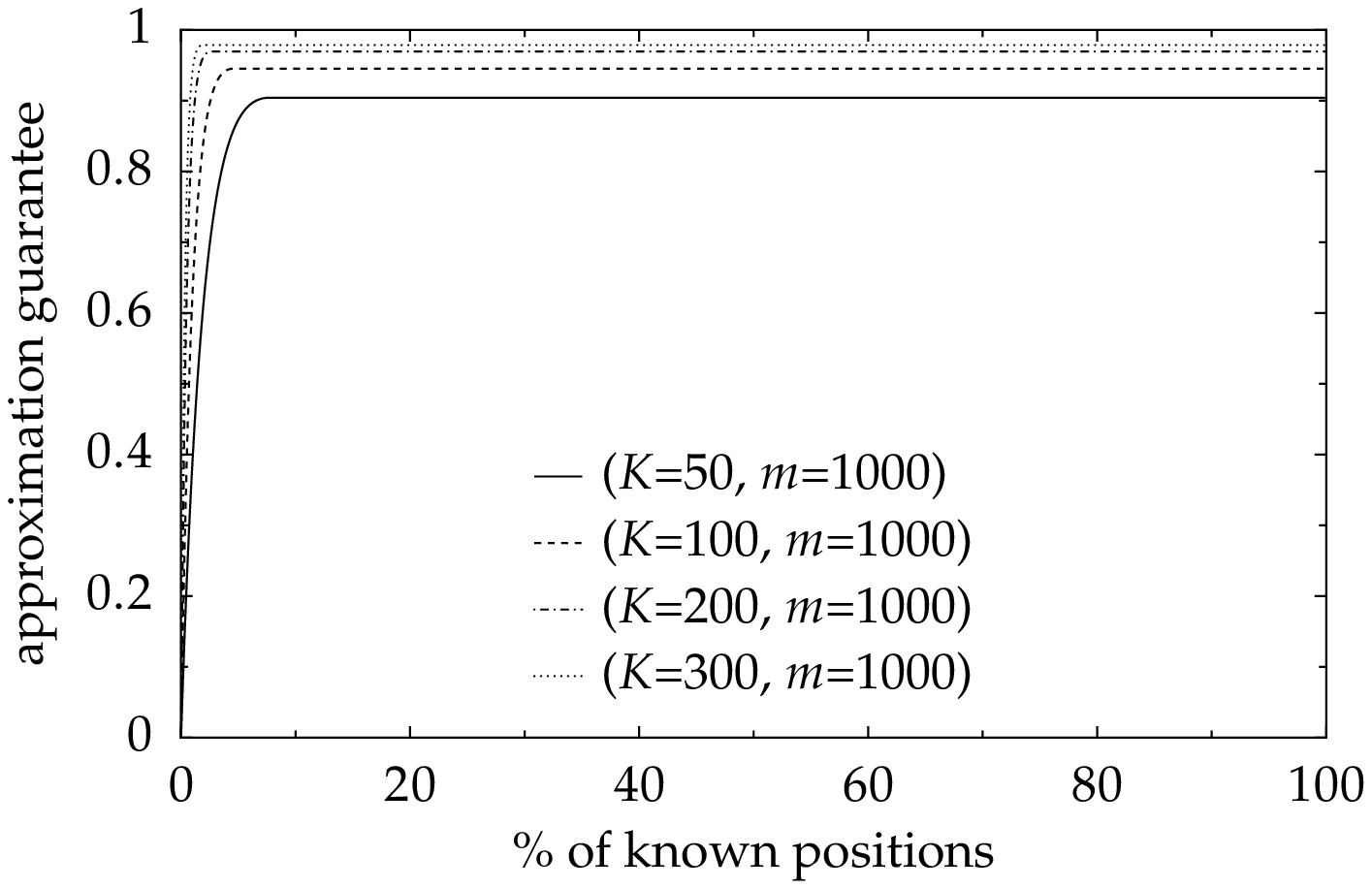}
\end{minipage}
\caption{The relation between the percentage of the known positions and the approximation ratio of Algorithm P for $\alpha_\bordadec$-\textsc{SU-CC}.}
\label{fig:cc_truncated}
\end{figure*}

For example, for Polish parliamentary elections ($K=460$, $m=6000$),
it suffices that each voter ranks only $0.5\%$ of his or her top
alternatives (that is, about $30$ alternatives) for the algorithm to
find a solution with guaranteed satisfaction at least $90\%$ of the
one (possibly infeasible) where every voter is assigned to his or her top alternative.

\subsection{ILP Formulation (Monroe, CC)}
To experimentally measure the quality of our approximation algorithms, we compare
the results against optimal solutions that we obtain using integer
linear programs (ILPs) that solve the Monroe and Chamberlin--Courant winner determination
problem. 
An ILP for the Monroe rule was provided by Potthoff and Brams~\cite{potthoff-brams}, Lu and Boutilier~\cite{budgetSocialChoice} adapted it also for the  Chamberlin--Courant rule with arbitraty PSF $\alpha$.
For the sake of completeness, below we recall the ILP whose optimal
solutions correspond to $\alpha$-SU-Monroe winner sets for the given
election (we also indicate which constraints to drop to obtain an ILP
for finding $\alpha$-SU-CC winner sets):
\begin{enumerate}
\item For each $i$, $1 \leq i \leq n$, and each $j$, $1 \leq j \leq m$
   we have a $0/1$ variable $a_{ij}$ indicating whether alternative $a_j$
   represents agent $i$. For each $j$, $1 \leq j \leq m$, we have a
   $0/1$ variable $x_j$ indicating whether alternative $a_j$ is included in
   the set of winners.
\item Our goal is to maximize the value $\sum_{i =
     1}^{n}\alpha(pos_{i}(a_j))a_{ij}$ subject to the following constraints:
    \begin{enumerate}
    \item For each $i$ and $j$, $1 \leq i \leq n, 1 \leq j \leq m$, $0
      \leq a_{ij} \leq x_j$ (alaternative $a_j$ can represent agent $i$
      only if $a_j$ belongs to the set of winners)
    \item For each $i$, $1 \leq i \leq n$, $\sum_{1 \leq j \leq m}
      a_{ij} = 1$ (every agent is represented by exactly one
      alternative).
    \item\label{item:bounding} For each $j$, $1 \leq j \leq m$, $x_j
      \lfloor \frac{n}{K}\rfloor \leq \sum_{1 \leq i \leq n} a_{ij}
      \leq x_j \lceil \frac{n}{K} \rceil$ (each alternative either does
      not represent anyone or represents between $\lfloor \frac{n}{K}
      \rfloor$ and $\lceil \frac{n}{K} \rceil$ agents; if we remove
      these constraints then we obtain an ILP for the Chamberlin-Courant
      rule).
    \item $\sum_{j=1}^n x_j \leq K$ (there are exactly $K$
      winners\footnote{For the Monroe framework inequality here is equivalent to equality. We use the inequality so that deleting
        constraints from item~(\ref{item:bounding}) leads to an ILP for the 
        Chamberlin-Courant rule.}).
    \end{enumerate}
 \end{enumerate}

We used the GLPK 4.47 package (GNU Linear Programming Kit, version 4.47)
to solve these ILPs, whenever it was possible to do so in reasonable
time.

\section{Empirical Evaluation of the Algorithms}\label{sec:experiments}

In this section we present the results of empirical evaluation of
algorithms from Section~\ref{sec:algorithms}. In the experiments we evaluated
versions of the randomized algorithms that use exactly 100 sampling steps.
In all cases but one, we
have used Borda PSF to measure voter satisfaction. In one case, with
six candidates, we have used DPSF defined through vector
$(3,3,3,2,1,0)$ (we made this choice due to the nature of the data set
used; see discussion later).

We have conducted four sets of experiments. First, we have tested all
our algorithms on relatively small elections (up to $10$ candidates,
up to $100$ agents). In this case we were able to compare the
 solutions provided by our algorithms with the optimal ones. (To obtain the optimal
solutions, we were using the ILP formulations and the GLPK's ILP solver.)
Thus we report the quality of our algorithms as the average of
fractions ${C}/{C_{\opt}}$, where $C$ is the satisfaction obtained by
a respective algorithm and $C_{\opt}$ is the satisfaction in the
optimal solution.
For each algorithm and data set, we also report the average fraction
${C}/{C_{\ideal}}$, where $C_{\ideal}$ is the satisfaction that the
voters would have obtained if each of them were matched to his or her
most preferred alternative. In our further experiments, where we
considered larger elections, we were not able to compute optimal
solutions, but fraction ${C}/{C_{\ideal}}$ gives a lower bound for
${C}/{C_{\opt}}$.  We report this value for small elections so that we
can see an example of the relation between ${C}/{C_{\opt}}$ and
${C}/{C_{\ideal}}$ and so that we can compare the results for small
elections with the results for the larger ones.  Further, for the case
of Borda PSF the ${C}/{C_\ideal}$ fraction has a very natural
interpretation: If its value for a given solution is $v$, then,
on the average, in this solution each voter is matched to an
alternative that he or she prefers to $(m-1)v$ alternatives.

In our second set of experiments, we have run our algorithms on large
elections (thousands of agents, hundreds of alternatives), coming
either from the NetFlix data set (see below) or generated by us using one of our
models. Here we reported the average fraction ${C}/{C_{\ideal}}$ only.
We have analyzed the quality of the solutions as a function of the
number of agents, the number of candidates, and the relative number of
winners (fraction $K/m$). (This last set of results is particularly
interesting because in addition to measuring the quality of our
algorithms, it allows one to asses the size of a committee one should
seek if a given average  satisfaction of agents is to be obtained).

In the third set of experiments, we have investigated the effect of
submitting truncated ballots (i.e., preference orders where only some
of the top alternatives are ranked). Specifically, we focused on the
relation between the fraction of ranked alternatives and the
approximation ratio of the algorithms.  We run our experiments on
relatively large instances describing agents' preferences; thus, here
as in the previous set of experiments, we used NetFlix data set and
the synthetic data sets. We report the quality of the algorithms as
the ratio ${C}/{C_{\ideal}}$.

In the fourth set of experiments we have measured running times of our
algorithms and of the ILP solver. Even though all our algorithms
(except for the ILP based ones) are polynomial-time, in practice some
of them are too slow to be useful.

\subsection{Experimental Data}\label{sec:DataSets}
For the evaluation of the algorithms we have considered both real-life
preference-aggregation data and synthetic data, generated according to
a number of election models. The experitments reported in this paper
predate the work of Mattei and Walsh~\cite{mat-wal:c:preflib} on
gathering a large collection of data sets with preference data, but we
mention that the conference version of this paper contributed several
data sets to their collection.

\subsubsection{Real-Life Data}\label{sec:real-data}

We have used real-life data regarding people's preference on sushi types,
movies, college courses, and competitors' performance in
figure-skating competitions.
One of the major problems regarding real-life preference data is that
either people express preferences over a very limited set of
alternatives, or their preference orders are partial. To address the
latter issue, for each such data set we complemented the partial
orders to be total orders using the technique of
Kamishima~\cite{Kamishima:Nantonac}.  (The idea is to complete each
preference order based on those reported preference orders that appear
to be similar.)

Some of our data sets contain a single profile, whereas the others
contain multiple profiles.  When preparing data for a given number $m$
of candidates and a given number $n$ of voters from a given data set,
we used the following method: We first uniformly at random chose a
profile within the data set, and then we randomly selected $n$ voters
and $m$ candidates. We used preference orders of these $n$ voters
restricted to these $m$ candidates.
\medskip



\noindent\textbf{Sushi Preferneces.}\quad
We used the set of preferences regarding sushi types collected by
Kamishima\cite{Kamishima:Nantonac}.\footnote{The sushi data set is
  available under the following url:
  \url{http://www.kamishima.net/sushi/}}  Kamishima has collected two
sets of preferences, which we call \textsc{S1} and \textsc{S2}. Data
set S1 contains complete rankings of $10$ alternatives collected from
$5000$ voters.  S2 contains partial rankings provided by $5000$ voters over a
set of $100$ alternatives (each vote ranks $10$ alternatives). We used
Kamishima~\cite{Kamishima:Nantonac} technique to obtain total
rankings.\medskip

\noindent\textbf{Movie Preferences.}\quad
Following Mattei et al.~\cite{Mattei:Netflix}, we have used the NetFlix data
set\footnote{http://www.netflixprize.com/} of movie preferences (we
call it \textsc{Mv}).  NetFlix data set contains ratings collected
from about $480$ thousand distinct users regarding $18$ thousand
movies. The users rated movies by giving them a score between $1$
(bad) and $5$ (good). The set contains about $100$ million ratings.
We have generated $50$ profiles using the following method: For each
profile we have randomly selected $300$ movies, picked $10000$ users
that ranked the highest number of the selected movies, and for each
user we have extended his or her ratings to a complete preference
order using the method of
Kamishima~\cite{Kamishima:Nantonac}.\medskip

\noindent\textbf{Course Preferences.}\quad
Each year the students at the AGH University choose courses that they
would like to attend. The students are offered a choice of six courses
of which they have to attend three.  Thus the students are asked to
give an unordered set of their three top-preferred courses and a
ranking of the remaining ones (in case too many students select a
course, those with the highest GPA are enrolled and the remaining ones
are moved to their less-preferred courses). In this data set, which we
call \textsc{Cr}, we have $120$ voters (students) and $6$ alternatives
(courses). However, due to the nature of the data, instead of using
Borda count PSF as the satisfaction measure, we have used the vector
$(3,3,3,2,1,0)$.  Currently this data set is available as part of
PrefLib~\cite{mat-wal:c:preflib}.\medskip

\noindent\textbf{Figure Skating.}\quad This data set, which we call \textsc{Sk}, contains
preferences of the judges over the performances in a figure-skating
competitions. The data set contains $48$ profiles, each describing a
single competition. Each profile contains preference orders of $9$
judges over about 20 participants. The competitions include European
skating championships, Olympic Games, World Junior, and World
Championships, all from 1998\footnote{This data set is available under
  the following url: \url{http://rangevoting.org/SkateData1998.txt}.}.
(Note that while in figure skating judges provide numerical scores,
this data set is preprocessed to contain preference orders.)

\subsubsection{Synthetic Data}

For our tests, we have also used profiles generated using three
well-known distributions of preference orders. \medskip

\noindent\textbf{Impartial Culture.}\quad Under the impartial culture
model of preferences (which we denote \textsc{IC}), for a given set
$A$ of alternatives, each voter's preference order is drawn uniformly
at random from the set of all possible total orders over $A$.  While
not very realistic, profiles generated using impartial culture model
are a standard testbed of election-related algorithms.  \medskip

\noindent\textbf{Polya-Eggenberger Urn Model.}\quad Following McCabe-Dansted and Slinko~\cite{mcc-sli:j:experiments} and Walsh~\cite{Walsh11},
we have used the Polya-Eggenberger urn model~\cite{bpublicchoice85}
(which we denote \textsc{Ur}). In this model we generate votes as
follows. We have a set $A$ of $m$ alternatives and an urn that
initially contains all $m!$ preference orders over $A$.  To generate a
vote, we simply randomly pick one from the urn (this is our generated
vote), and then---to simulate correlation between voters---we return
$a$ copies of this vote to the urn. When generating an election with
$m$ candidates using the urn model, we have set the parameter $a$ so
that $\frac{a}{m!} = 0.05$ (Both McCabe-Dansted and Slinko~\cite{mcc-sli:j:experiments} and Walsh~\cite{Walsh11} call this parameter
$b$; we mention that those authors use much higher values of $b$ but
we felt that too high a value of $b$ leads to a much too strong
correlation between votes).\medskip

\noindent\textbf{Generalized Mallow's Model.}\quad We refer to this
data set as \textsc{Ml}. Let $\succ$ and $\succ'$ be two preference
orders over some alternative set $A$. Kendal-Tau distance between
$\succ$ and $\succ'$, denoted $d_{K}(\succ,\succ')$, is defined as the
number of pairs of candidates $x, y \in A$ such that either $x \succ
y \land y \succ' x$ or $y \succ x \land x \succ' y$.

Under Mallow's distribution of preferences~\cite{mallowModel} we are
given two parameters: A \emph{center} preference order $\succ$ and a
number $\phi$ between $0$ and $1$. The model says that the probability
of generating preference order $\succ'$ is proportional to the value
$\phi^{d_{K}(\succ,\succ')}$.  To generate preference orders following
Mallow's distribution, we use the algorithm given by Lu and Boutilier
\cite{mallowImplementation2011}.

In our experiments, we have used a mixture of Mallow's models.  Let
$A$ be a set of alternatives and let $\ell$ be a positive integer. This
mixture model is parametrized by three vectors, $\Lambda =
(\lambda_{1}, \dots, \lambda_{\ell})$ (where each $\lambda_i$ is between $0$ and $1$, and $\sum_{i=1}^\ell\lambda_1=1$), $\Phi
= (\phi_{1}, \dots, \phi_{\ell})$ (where each $\phi_i$ is a number between $0$ and $1$), and $\Pi = (\succ_{1}, \ldots,
\succ_{\ell})$ (where each $\succ_i$ is a preference
order over $A$). To generate a vote, we pick a random integer $i$, $1
\leq i \leq \ell$ (each $i$ is chosen with probability $\lambda_i$), and
then generate the vote using Mallow's model with parameters
$(\succ_i,\phi_i)$.

For our experiments, we have used $a = 5$, and we have generated
vectors $\Lambda$, $\Phi$, and $\Pi$ uniformly at random.

\begin{table}[t]
\begin{center}

\begin{tabular}{|c|c|c|c|c|c||c|c|c|c|}
\cline{2-10}
\multicolumn{1}{c|}{} & \multicolumn{5}{|c||}{Monroe} & \multicolumn{4}{|c|}{CC} \\
\cline{2-10}
\multicolumn{1}{c|}{} & A & B & C & GM & R    & C & GM & P & R\\
\cline{1-10}
\textsc{S1} & $0.94$ & $0.99$        & $\approx 1.0$  & $0.99$         & $0.99$ & $1.0$ & $\approx 1.0$ & $0.99$ & $0.99$ \\
\textsc{S2} & $0.95$ & $0.99$        & $1.0$          & $\approx 1.0$  & $0.99$ & $1.0$ & $\approx 1.0$ & $0.98$ & $0.99$ \\
\textsc{Mv} & $0.96$ & $\approx 1.0$ & $1.0$          & $\approx 1.0$  & $0.98$ & $1.0$ & $\approx 1.0$ & $0.96$ & $\approx 1.0$ \\
\textsc{Cr} & $0.98$ & $0.99$        & $1.0$          & $\approx 1.0$  & $0.99$ & $1.0$ & $\approx 1.0$ & $1.0$  & $\approx 1.0$ \\
\textsc{Sk} & $0.99$ & $\approx 1.0$ & $1.0$          & $\approx 1.0$  & $0.94$ & $1.0$ & $\approx 1.0$ & $0.85$ & $0.99$ \\
\cline{1-10}
\textsc{IC} & $0.94$ & $0.99$        & $\approx 1.0$  & $0.99$         & $0.99$ & $1.0$ & $\approx 1.0$ & $0.99$ & $0.99$\\
\textsc{Ml} & $0.94$ & $0.99$        & $1.0$          & $0.99$         & $0.99$ & $1.0$ & $\approx 1.0$ & $0.99$ & $0.99$ \\
\textsc{Ur} & $0.95$ & $0.99$        & $\approx 1.0$  & $0.99$         & $0.99$ & $1.0$ & $0.99$        & $0.97$ & $0.99$ \\
\cline{1-10}
\end{tabular}
\caption{The average quality of the  algorithms compared with the optimal solution ({$C/C_{\opt}$}) for the small instances of data and for {$K=3$} (${K=2}$ for \textsc{Cr}); ${m=10}$ (${m=6}$ for \textsc{Cr}); ${n=100}$ (${n=9}$ for \textsc{Sk}).}
\label{table:qualityAlgs1}
\end{center}
\end{table}

\begin{table}[t]
\begin{center}
\begin{tabular}{|c|c|c|c|c|c||c|c|c|c|}
\cline{2-10}
\multicolumn{1}{c|}{} & \multicolumn{5}{|c||}{Monroe} & \multicolumn{4}{|c|}{CC} \\
\cline{2-10}
\multicolumn{1}{c|}{} & A & B & C & GM & R    & C & GM & P & R\\
\cline{1-10}
\textsc{S1} & $0.95$ & $\approx 1.0$ & $1.0$          & $0.99$         & $0.99$        & $1.0$ & $\approx 1.0$ & $0.97$ & $0.99$ \\
\textsc{S2} & $0.94$ & $0.99$        & $\approx 1.0$  & $0.99$         & $0.99$        & $1.0$ & $\approx 1.0$ & $0.98$ & $\approx 1.0$ \\
\textsc{Mv} & $0.95$ & $0.99$        & $1.0$          & $\approx 1.0$  & $0.98$        & $1.0$ & $\approx 1.0$ & $0.97$ & $\approx 1.0$ \\
\textsc{Cr} & $0.96$ & $\approx 1.0$ & $1.0$          & $\approx 1.0$  & $0.99$        & $1.0$ & $1.0$         & $1.0$  & $1.00$ \\
\textsc{Sk} & $0.99$ & $\approx 1.0$ & $1.0$          & $\approx 1.0$  & $0.88$        & $1.0$ & $1.0$         & $0.91$ & $\approx 1.0$ \\
\cline{1-10}
\textsc{IC} & $0.95$ & $0.99$        & $\approx 1.0$  & $0.99$         & $0.99$        & $1.0$ & $\approx 1.0$ & $0.99$ & $0.99$\\
\textsc{Ml} & $0.95$ & $0.99$        & $\approx 1.0$  & $0.99$         & $0.99$        & $1.0$ & $\approx 1.0$ & $0.98$ & $0.99$ \\
\textsc{Ur} & $0.96$ & $0.99$        & $\approx 1.0$  & $0.99$         & $\approx 1.0$ & $1.0$ & $\approx 1.0$ & $0.96$ & $0.99$ \\
\cline{1-10}

\end{tabular}

\caption{The average quality of the  algorithms compared with the optimal solution ({$C/C_{\opt}$}) for the small instances of data and for {$K=6$} (${K=4}$ for \textsc{Cr}); ${m=10}$ (${m=6}$ for \textsc{Cr}); ${n=100}$ (${n=9}$ for \textsc{Sk}).}
\label{table:qualityAlgs2}
\end{center}
\end{table}

\begin{table}[t]
\begin{center}
\begin{tabular}{|c|c|c|c|c|c||c|c|c|c|}
\cline{2-10}
\multicolumn{1}{c|}{} & \multicolumn{5}{|c||}{Monroe} & \multicolumn{4}{|c|}{CC} \\
\cline{2-10}
\multicolumn{1}{c|}{} & A & B & C & GM & R    & C & GM & P & R\\
\cline{1-10}
\textsc{S1} & $0.85$ & $0.89$ & $0.9$   & $0.89$ & $0.89$ & $0.92$ & $0.89$ & $0.91$ & $0.92$ \\
\textsc{S2} & $0.85$ & $0.89$ & $0.89$  & $0.89$ & $0.89$ & $0.93$ & $0.9$  & $0.91$ & $0.92$ \\
\textsc{Mv} & $0.88$ & $0.92$ & $0.92$  & $0.92$ & $0.91$ & $0.97$ & $0.92$ & $0.93$ & $0.97$ \\
\textsc{Cr} & $0.94$ & $0.97$ & $0.96$  & $0.96$ & $0.96$ & $0.97$ & $0.97$ & $0.97$ & $0.97$ \\
\textsc{Sk} & $0.96$ & $0.96$ & $0.97$  & $0.97$ & $0.91$ & $1.0$  & $0.97$ & $0.82$ & $0.99$ \\
\cline{1-10}
\textsc{IC} & $0.8$  & $0.84$ & $0.85$  & $0.84$ & $0.84$ & $0.85$ & $0.83$ & $0.84$ & $0.85$\\
\textsc{Ml} & $0.83$ & $0.88$ & $0.88$  & $0.9$  & $0.88$ & $0.92$ & $0.90$ & $0.89$ & $0.94$ \\
\textsc{Ur} & $0.8$  & $0.85$ & $0.86$  & $0.87$ & $0.85$ & $0.9$  & $0.87$ & $0.87$ & $0.89$ \\
\cline{1-10}
\end{tabular}

\caption{The average quality of the  algorithms compared with the simple lower bound ({$C/C_{\ideal}$}) for the small instances of data and for {$K=3$} (${K=2}$ for \textsc{Cr}); ${m=10}$ (${m=6}$ for \textsc{Cr}); ${n=100}$ (${n=9}$ for \textsc{Sk}).}

\label{table:qualityAlgs3}
\end{center}
\vspace{-0.5cm}
\end{table}

\begin{table}[t]
\begin{center}
\begin{tabular}{|c|c|c|c|c|c||c|c|c|c|}
\cline{2-10}
\multicolumn{1}{c|}{} & \multicolumn{5}{|c||}{Monroe} & \multicolumn{4}{|c|}{CC} \\
\cline{2-10}
\multicolumn{1}{c|}{} & A & B & C & GM & R    & C & GM & P & R\\
\cline{1-10}
S1 & $0.91$ & $0.96$ & $0.96$  & $0.95$ & $0.95$ & $0.98$ & $0.98$ & $0.96$ & $0.98$ \\
S2 & $0.88$ & $0.93$ & $0.93$  & $0.93$ & $0.93$ & $0.98$ & $0.98$ & $0.96$ & $0.98$ \\
Mv & $0.85$ & $0.89$ & $0.89$  & $0.89$ & $0.88$ & $0.99$ & $0.99$ & $0.97$ & $0.99$ \\
Cr & $0.95$ & $0.98$ & $0.99$  & $0.99$ & $0.98$ & $1.0$  & $1.0$  & $1.0$  & $1.0$ \\
Sk & $0.91$ & $0.92$ & $0.92$  & $0.92$ & $0.81$ & $1.0$  & $1.0$  & $0.91$ & $\approx 1.0$ \\
\cline{1-10}
IC & $0.91$ & $0.95$ & $0.95$  & $0.94$ & $0.95$ & $0.96$ & $0.96$ & $0.95$ & $0.95$\\
Ml & $0.89$ & $0.94$ & $0.94$  & $0.94$ & $0.93$ & $0.97$ & $0.98$ & $0.95$ & $0.98$ \\
Ur & $0.91$ & $0.95$ & $0.95$  & $0.94$ & $0.95$ & $0.98$ & $0.98$ & $0.94$ & $0.97$ \\
\cline{1-10}

\end{tabular}

\caption{The average quality of the  algorithms compared with the simple lower bound ({$C/C_{\ideal}$}) for the small instances of data and for $K=6$ (${K=4}$ for \textsc{Cr}); ${m=10}$ (${m=6}$ for \textsc{Cr}); ${n=100}$ (${n=9}$ for \textsc{Sk}).}

\label{table:qualityAlgs4}
\end{center}
\end{table}

\subsection{Evaluation on Small Instances}

We now present the results of our experiments on small
elections. 
%
%
%
For each data set, we generated elections with the number of agents
$n=100$ ($n=9$ for data set \textsc{Sk} because there are only $9$
voters there) and with the number of alternatives $m=10$ ($m=6$ for
data set \textsc{Cr} because there are only $6$ alternatives there)
using the method described in Section~\ref{sec:real-data} for the
real-life data sets, and in the natural obvious way for synthetic
data.  For each algorithm and for each data set we ran $500$
experiments on different instances for $K=3$ (for the \textsc{Cr} data
set we used $K=2$) and $500$ experiments for $K=6$ (for \textsc{Cr} we
set $K=4$). For Algorithm $C$ (both for Monroe and for CC) we set the
parameter $d$, describing the number of assignment functions computed
in parallel, to $15$. The results (average fractions ${C}/{C_\opt}$
and ${C}/{C_\ideal}$) for $K=3$ are given in
Tables~\ref{table:qualityAlgs1}~and~\ref{table:qualityAlgs3}; the
results for $K=6$ are given in
Tables~\ref{table:qualityAlgs2}~and~\ref{table:qualityAlgs4} (they are
almost identical as for $K=3$). For each experiment in this section we
also computed the standard deviation; it was always on the order of
$0.01$.
The results lead to the following conclusions:

\begin{enumerate}
\item For the case of Monroe, already Algorithm~A obtains very good
  results, but nonetheless Algorithms~B and C improve noticeably upon
  Algorithm~A. In particular, Algorithm~C (for $d=15$) obtains the
  highest satisfaction on all data sets and in almost all cases was
  able to find an optimal solution.
\item Both for Monroe and for CC, Algorithm~R gives slightly worse
  solutions than Algorithm~C.
\item The results do not seem to depend on the data sets used in the
  experiments (the only exception is Algorithm~R for the Monroe system
  on data set \textsc{Sk}; however \textsc{Sk} has only 9 voters so it
  can be viewed as a border case).
\end{enumerate}

\subsection{Evaluation on Larger Instances}

\begin{figure*}[t!h!]
\begin{minipage}[h]{0.48\linewidth}
  \centering
  \includegraphics[width=\textwidth]{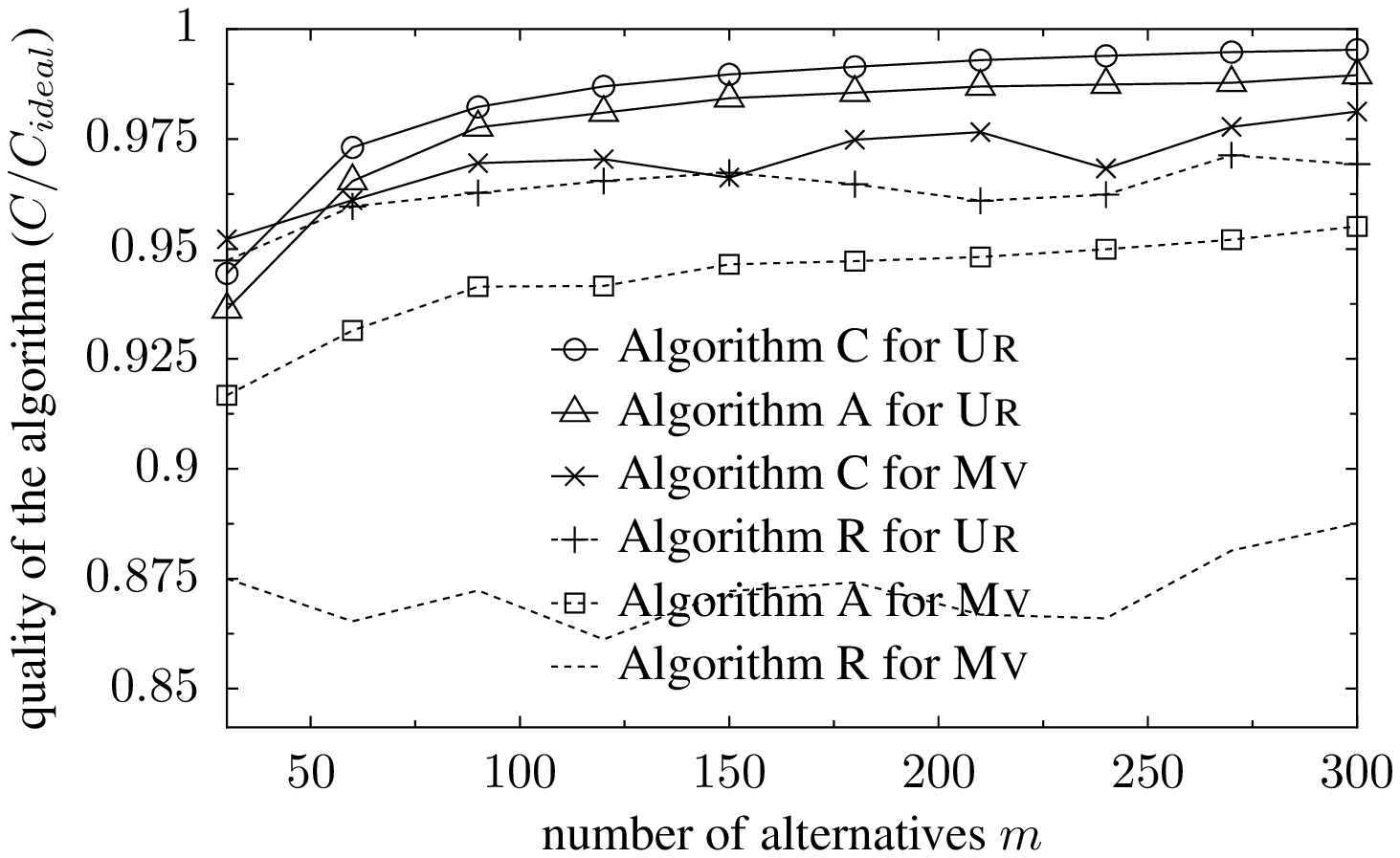}
\end{minipage}
\hspace{0.3cm}
\begin{minipage}[h]{0.48\linewidth}
  \centering
  \includegraphics[width=\textwidth]{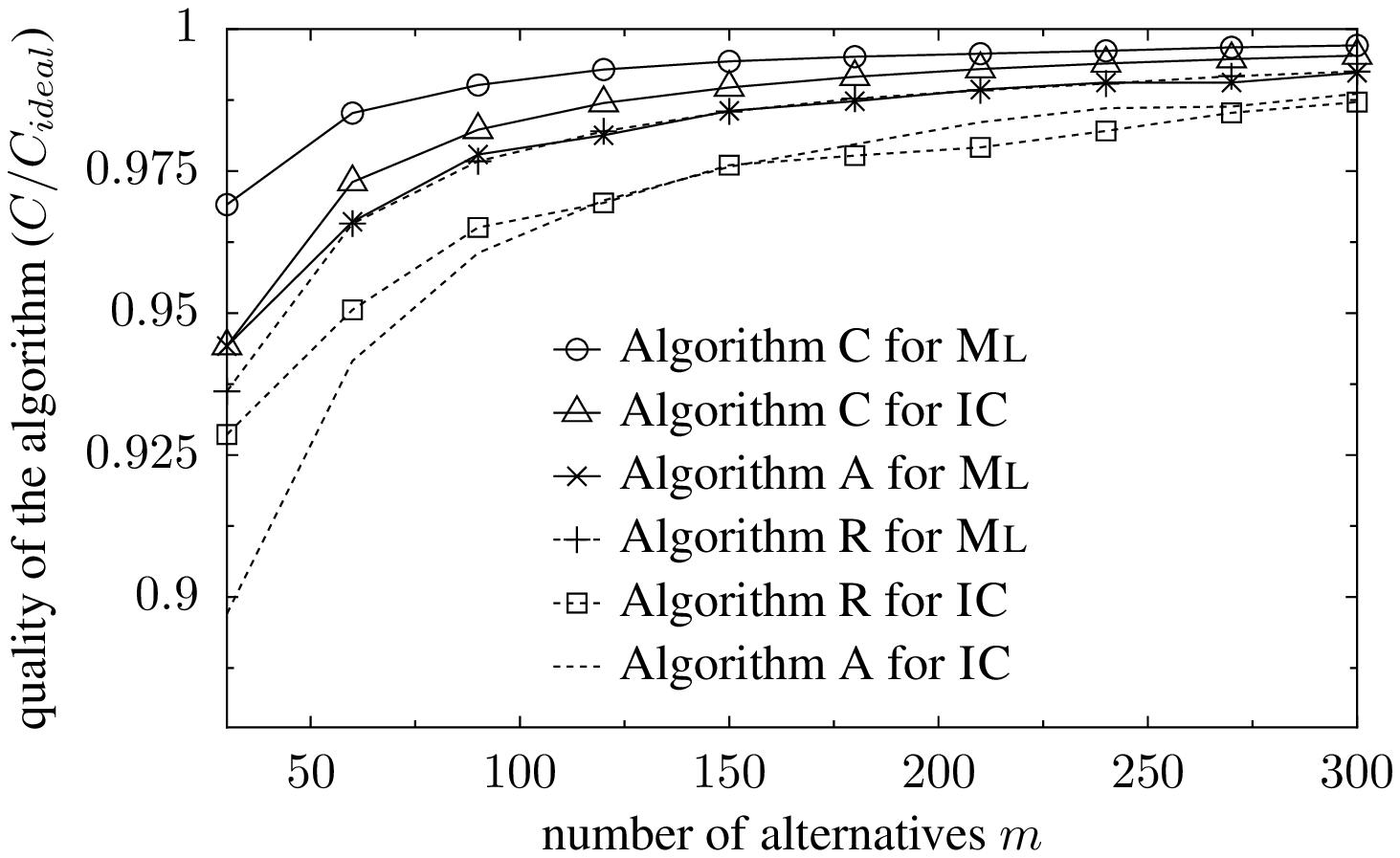}
\end{minipage}
  \caption{The relation between the number of alternatives
    {$m$} and the quality of the algorithms
    {$C/C_{\ideal}$} for the Monroe system; ${K/m = 0.3}$; {$n = 1000$}.}
  \label{fig:changing_m_monroe}
  \vspace{0.4cm}

\begin{minipage}[h]{0.48\linewidth}
  \centering
  \includegraphics[width=\textwidth]{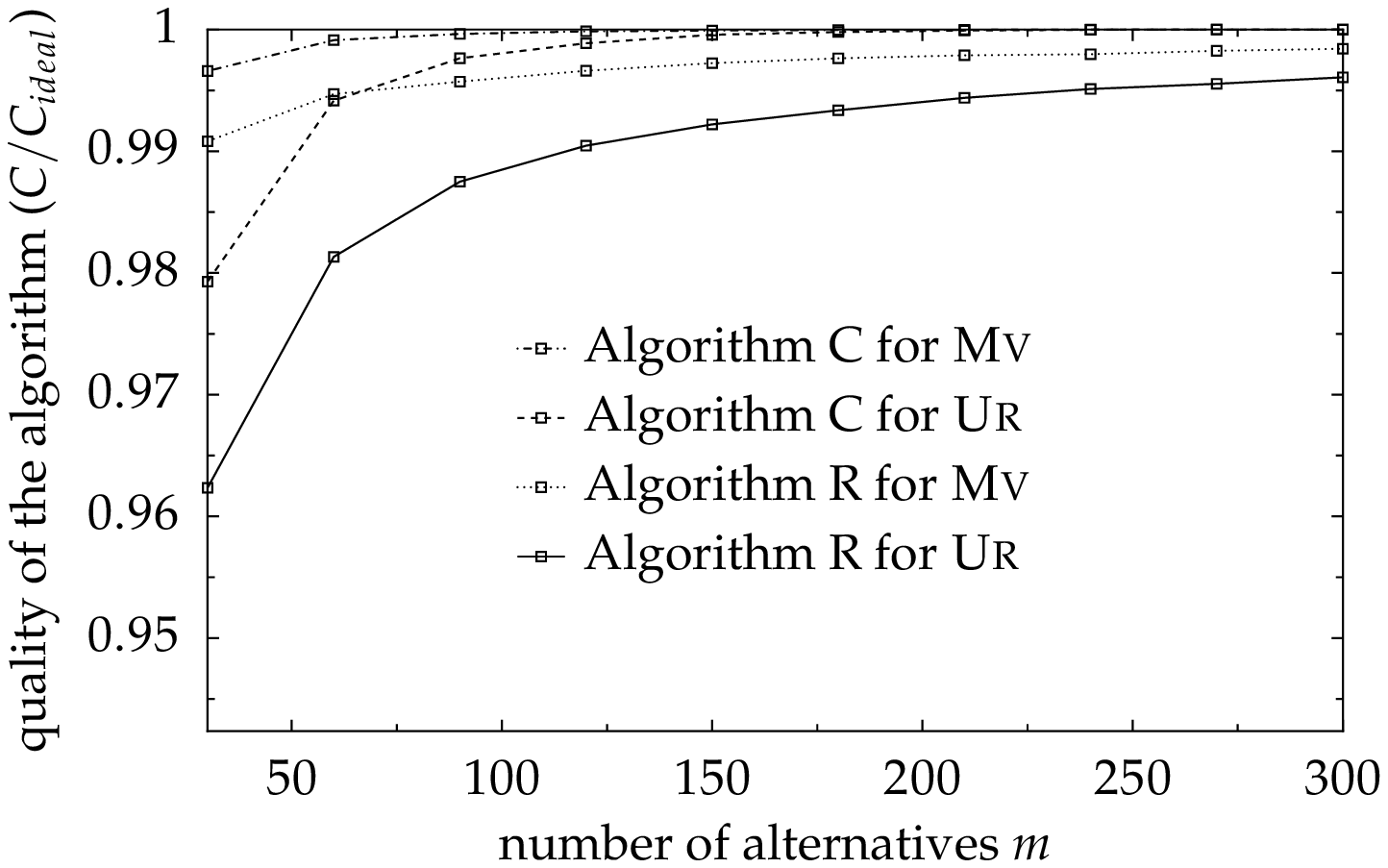}
\end{minipage}
\hspace{0.3cm}
\begin{minipage}[h]{0.48\linewidth}
  \centering
  \includegraphics[width=\textwidth]{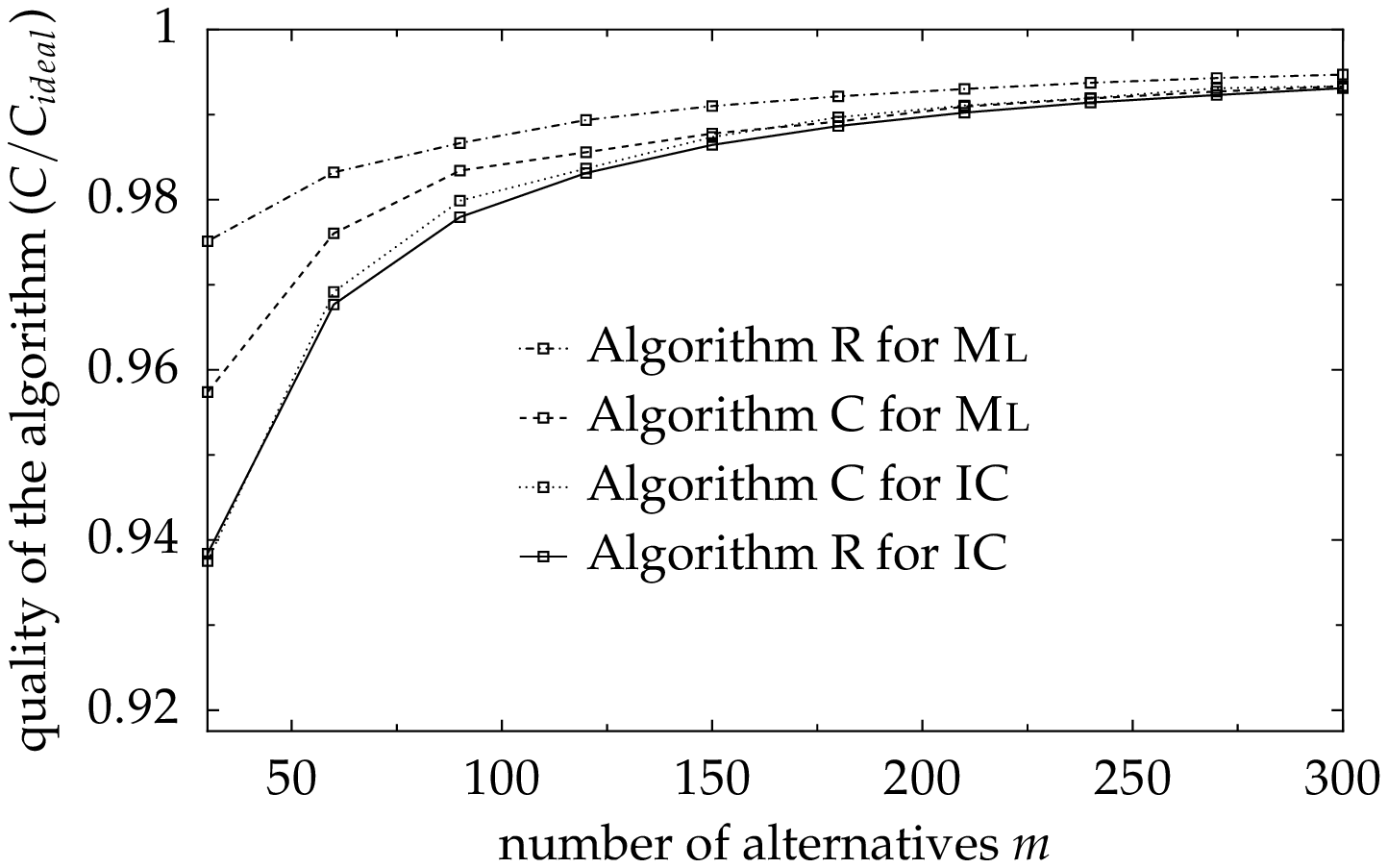}
\end{minipage}
  \caption{The relation between the number of alternatives
    {$m$} and the quality of the algorithms
    {$C/C_{\ideal}$} for the Chamberlin--Courant system;  ${K/m = 0.3}$; {$n = 1000$}.}
  \label{fig:changing_m_cc}
\end{figure*}

\begin{figure*}[t!h!]
\begin{minipage}[h]{0.48\linewidth}
  \centering
  \includegraphics[width=\textwidth]{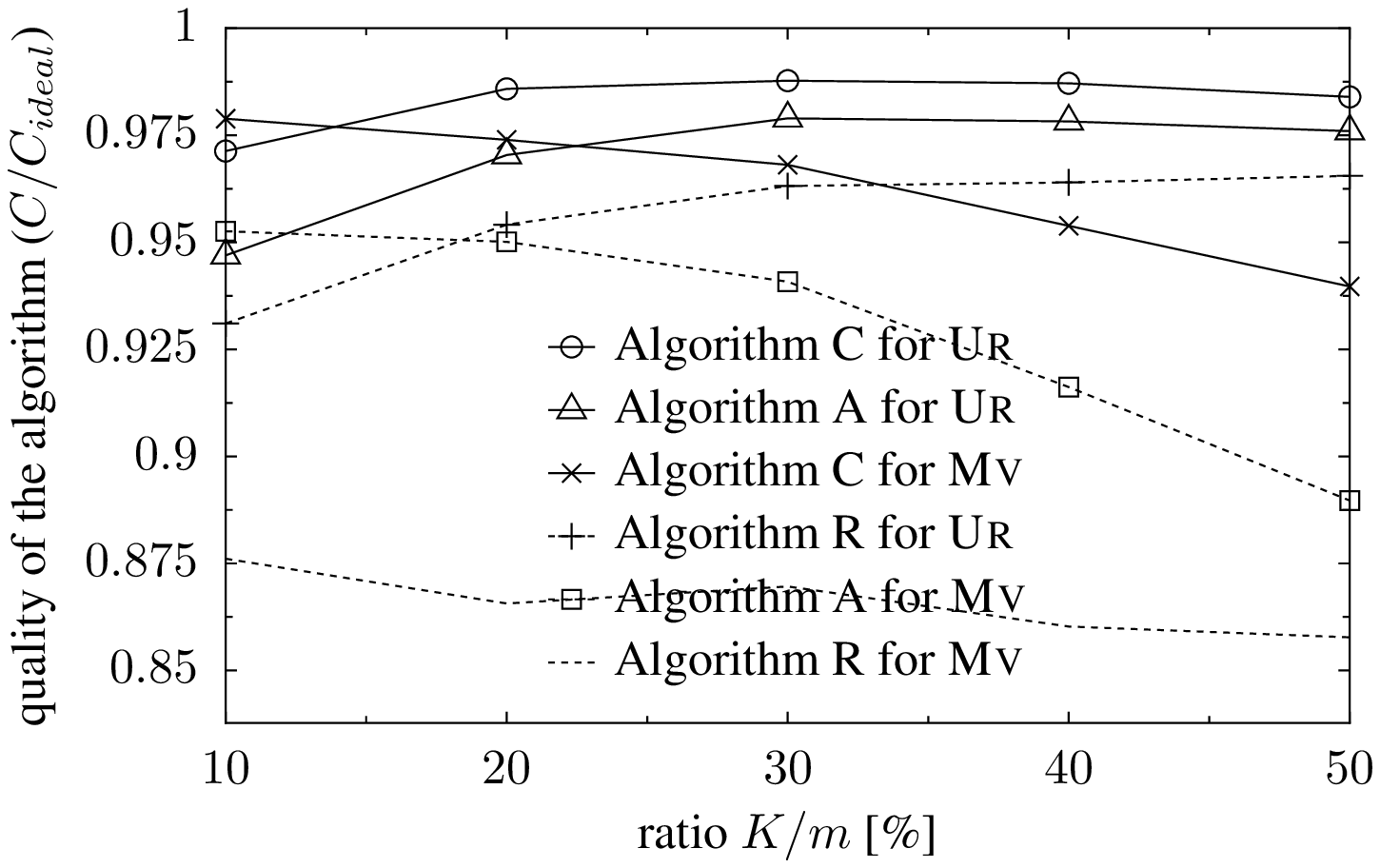}
\end{minipage}
\hspace{0.3cm}
\begin{minipage}[h]{0.48\linewidth}
  \centering
  \includegraphics[width=\textwidth]{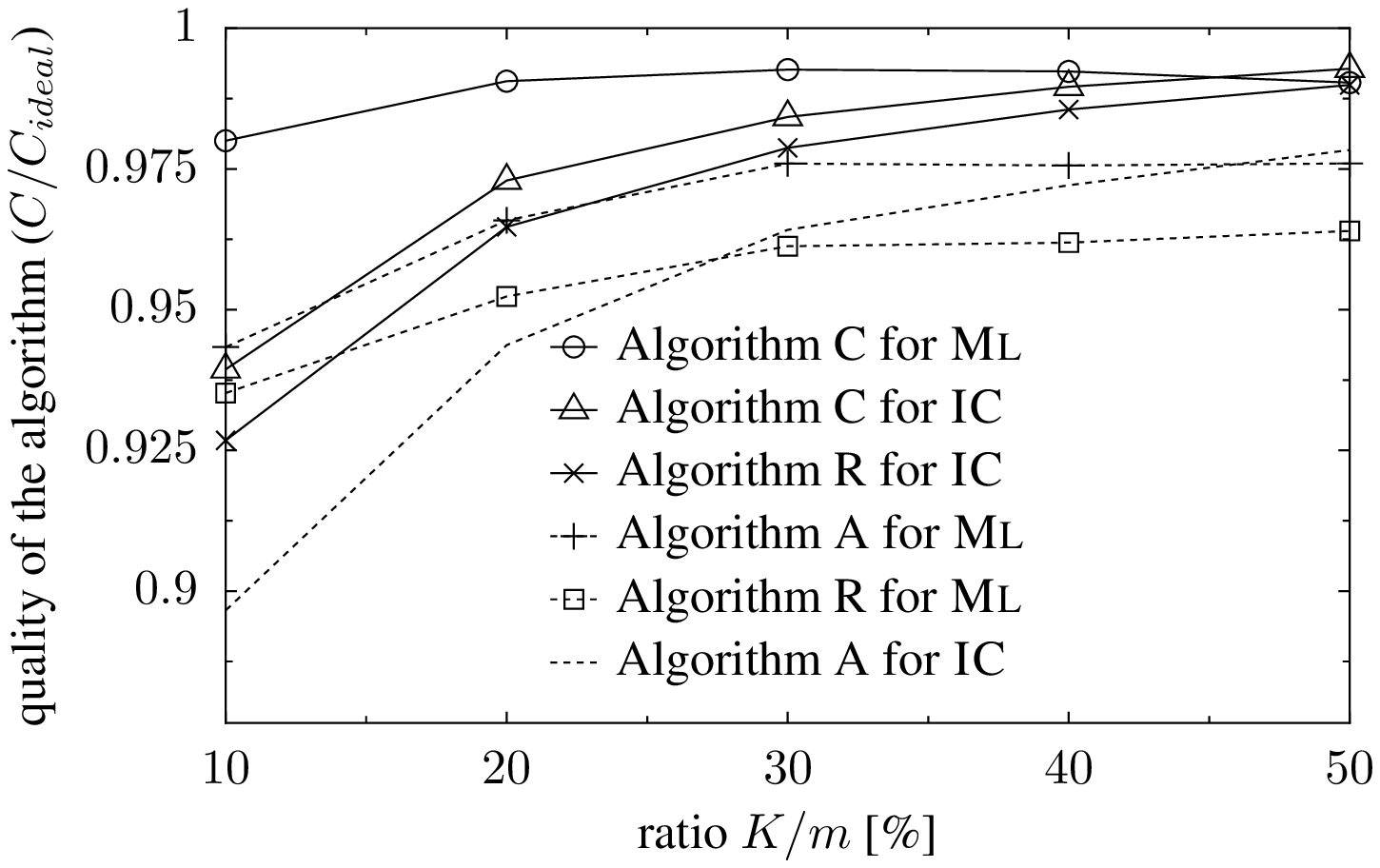}
\end{minipage}
  \caption{The relation between the ratio {$K/m$} and the
    quality of the algorithms {$C/C_{\ideal}$} for the Monroe
    system; ${m = 100}$; ${n = 1000}$.}
  \label{fig:changing_km_monroe}
  \vspace{0.4cm}

\begin{minipage}[h]{0.48\linewidth}
  \centering
  \includegraphics[width=\textwidth]{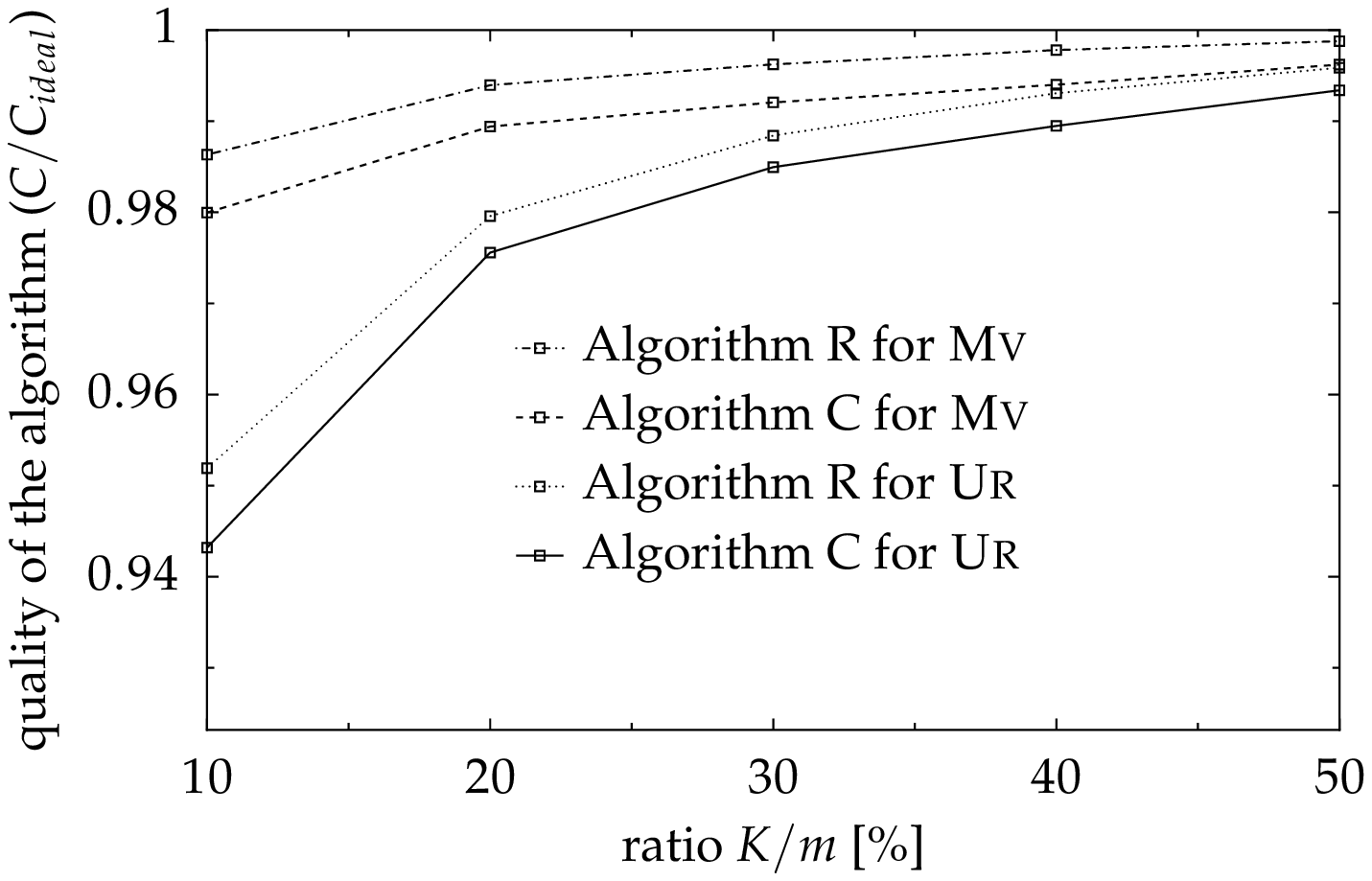}
\end{minipage}
\hspace{0.3cm}
\begin{minipage}[h]{0.48\linewidth}
  \centering
  \includegraphics[width=\textwidth]{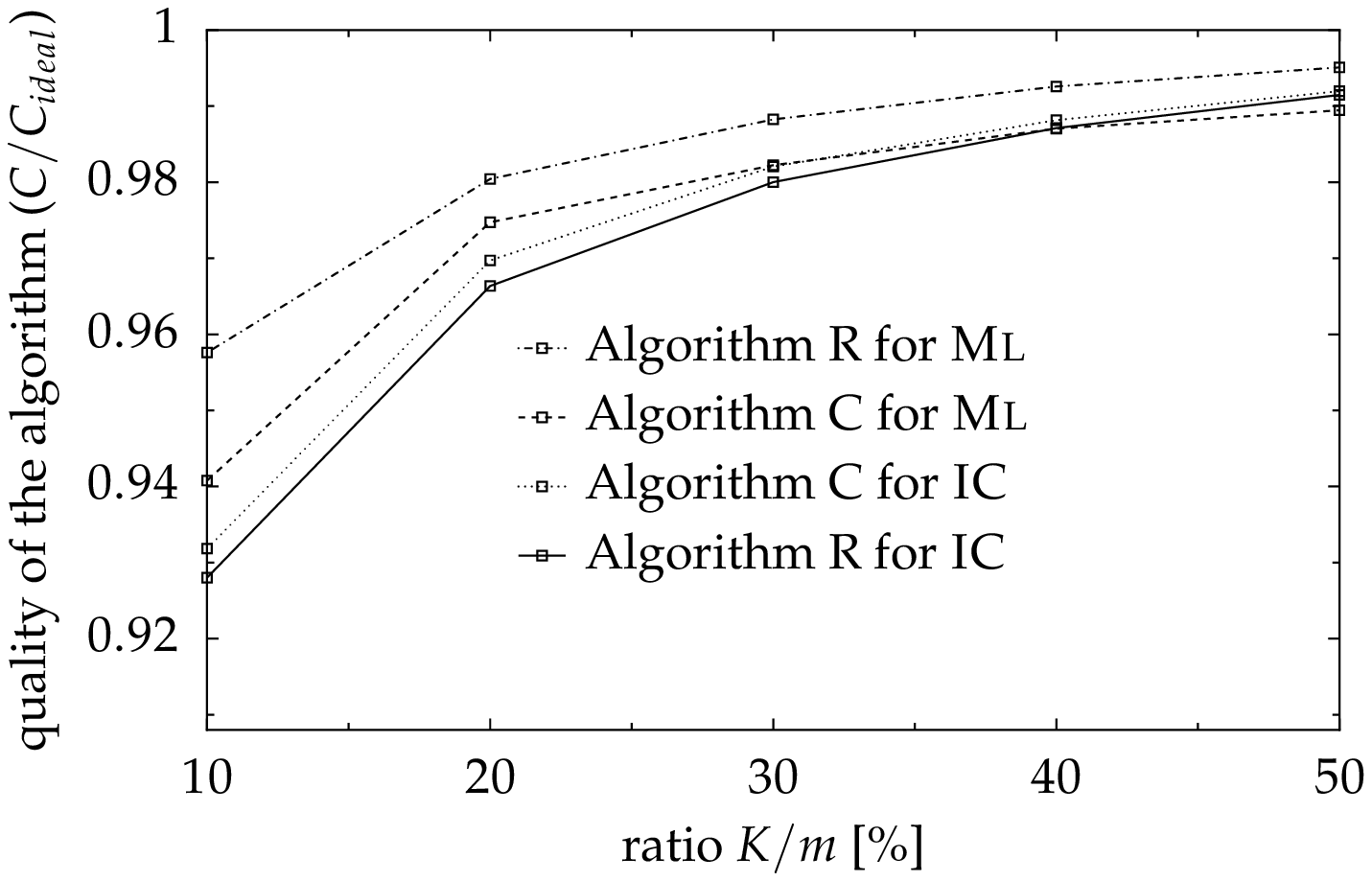}
\end{minipage}
  \caption{The relation between the ratio {$K/m$} and the
    quality of the algorithms {$C/C_{\ideal}$} for the
    Chamberlin--Courant system; ${m = 100}$; ${n = 1000}$.}
  \label{fig:changing_km_cc}
\end{figure*}

For experiments on larger instances we needed data sets with at
least $n = 10 000$ agents. Thus we used the NetFlix data set and
synthetic data. (Additionally, we run the subset of experiments (for
$n \leq 5000$) also for the \textsc{S2} data set.) For the Monroe rule we
present results for Algorithm~A, Algorithm~C, and Algorithm~R, and for
the Chamberlin--Courant rule we present results for Algorithm~C and
Algorithm~R. We limit the set of algorithms for the sake of the clarity of
the presentation. For Monroe we chose Algorithm~A because it is the
simplest and the fastest one, Algorithm~C because it is the best
generalization of Algorithm~A that we were able to run in reasonable
time, and Algorithm~R to compare a randomized algorithm to
deterministic ones. For the Chamberlin--Courant rule we chose Algorithm~C
because it is, intuitively, the best one, and we chose Algorithm~R for
the same reason as in the case of Monroe.

First, for each data set and for each algorithm we fixed the value of
$m$ and $K$ and for each $n$ ranging from $1000$ to $10000$ with the
step of $1000$ we run $50$ experiments. We repeated this procedure for
4 different combinations of $m$ and $K$: ($m = 10$, $K = 3$), ($m =
10$, $K = 6$), ($m = 100$, $K = 30$) and ($m = 100$, $K = 60$). We
measured the statistical correlation between the number of voters and
the quality of the algorithms $C/C_{\ideal}$. The ANOVA test in most
cases showed that there is no such correlation. The only exception was
\textsc{S2} data set, for which we obtained an almost negligible
correlation. For example, for ($m = 10, K = 3$) Algorithm $C$ under
data set \textsc{S2} for Monroe's system for $n = 5000$ gave
$C/C_{\ideal} = 0.88$, while for $n = 100$ (in the previous section)
we got $C/C_{\ideal} = 0.89$. Thus we conclude that in practice the
number of agents has almost no influence on the quality of the results
provided by our algorithms.

Next, we fixed the number of voters $n = 1000$ and the ratio $K/m =
0.3$, and for each $m$ ranging from $30$ to $300$ with the step of
$30$ (naturally, as $m$ changed, so did $K$ to maintain the ratio
$K/m$), we run 50 experiments. We repeated this procedure for $K/m =
0.6$. The relation between $m$ and $C/C_{\ideal}$ for \textsc{Mv} and
\textsc{Ur}, under both the Monroe rule and the Chamberlin--Courant rule,
is given in Figures~\ref{fig:changing_m_monroe}
and~\ref{fig:changing_m_cc} (the results for $K/m = 0.6$ look similar).

Finally, we fixed $n = 1000$ and $m = 100$, and for each $K/m$ ranging
from $0.1$ and $0.5$ with the step of $0.1$ we run $50$ experiments. The
relation between the ratio $K/m$ and the quality $C/C_{\ideal}$ is
presented in Figures~\ref{fig:changing_km_monroe} and~~\ref{fig:changing_km_cc}. 

For the case of Chamberlin--Courant system, increasing the size of the
committee to be elected improves overall agents' satisfaction. Indeed, since there are no
constraints on the number of agents matched to a given alternative, a
larger committee means more opportunities to satisfy the agents.  For the
Monroe rule, a larger committee may lead to a lower total satisfaction.  This
happens if many agents like a particular alternative a lot, but only
some of them can be matched to this alternative and others have to be
matched to their less preferred ones.  Nonetheless, we see that
Algorithm~C achieves $C/C_{\ideal} = 0.925$ even for $K/m = 0.5$ for
the NetFlix data set.

Our conclusions from these experiments are the following.  For the
Monroe rule, even Algorithm~A achieves very good results. However,
Algorithm~C consistently achieves better  (indeed, almost perfect) ones.  
For the Chamberlin--Courant rule the randomized algorithm on some datasets performs better than the deterministic ones. However, even in such cases, the improvement over the Algorithm C is small.

\subsection{Truncated ballots}\label{sec:truncated}

\begin{figure*}[t!h!]
\begin{minipage}[h]{0.32\linewidth}
  \centering
  \includegraphics[width=\textwidth]{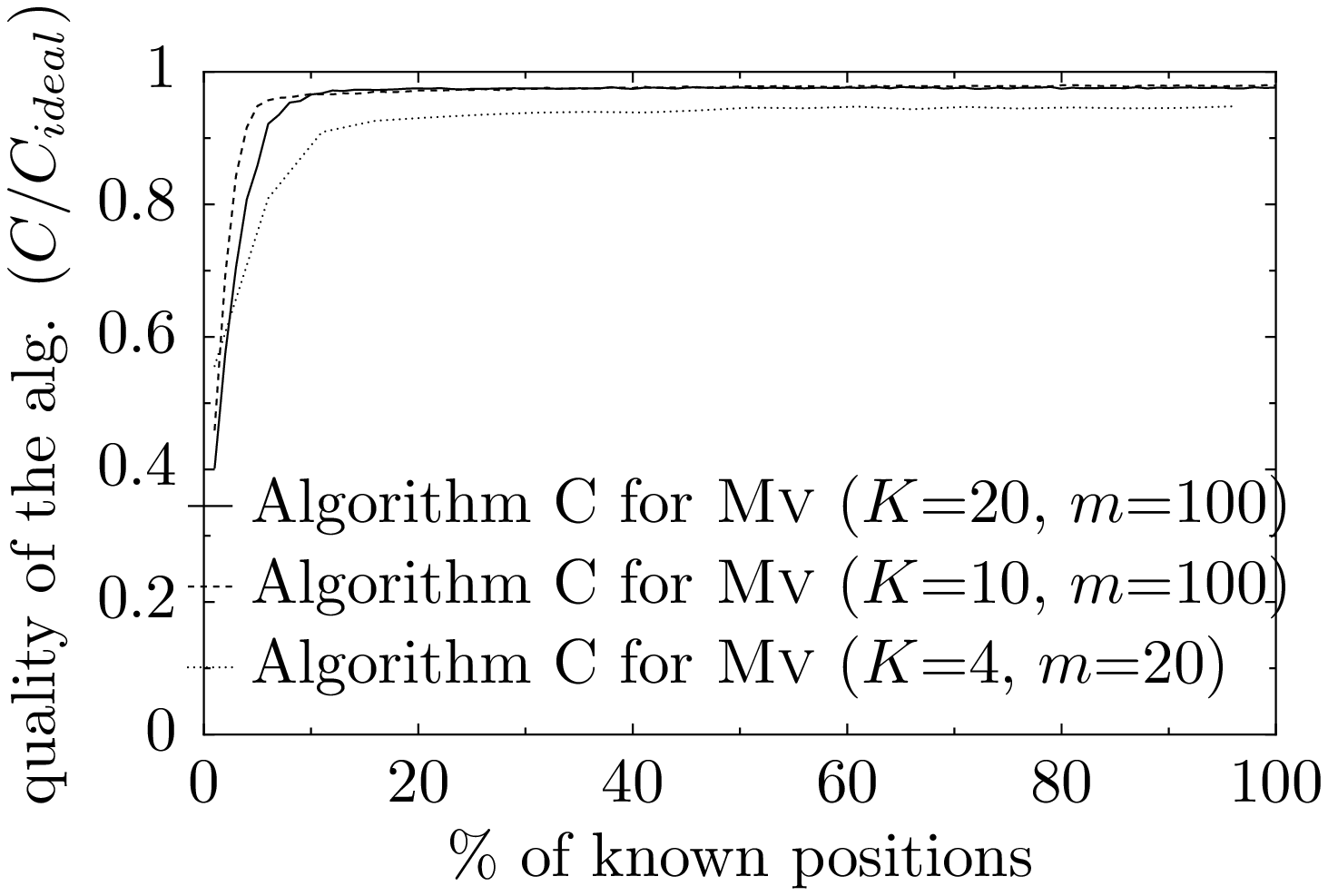}
\end{minipage}
\hspace{0.1cm}
\begin{minipage}[h]{0.32\linewidth}
  \centering
  \includegraphics[width=\textwidth]{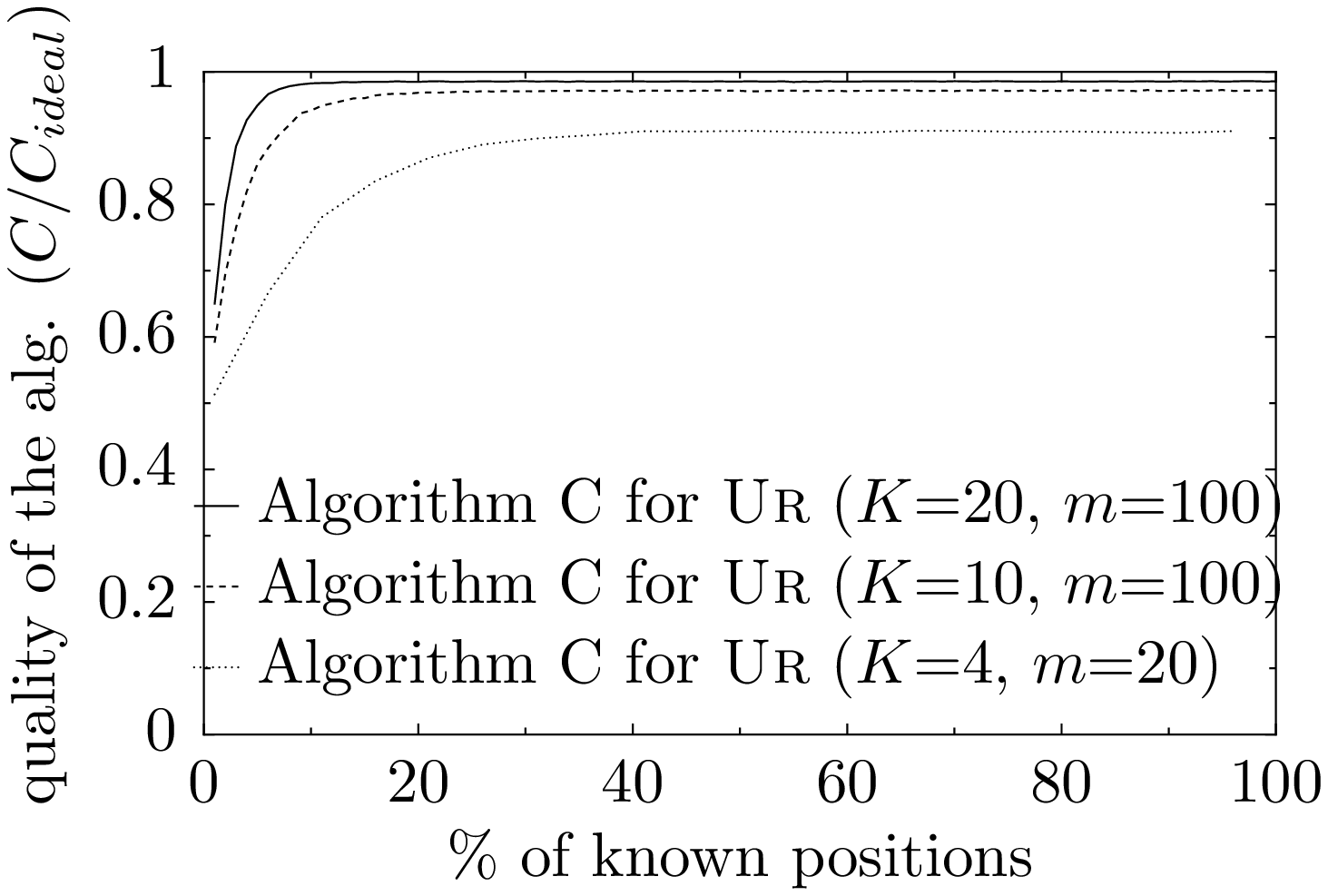}
\end{minipage}
\hspace{0.1cm}
\begin{minipage}[h]{0.32\linewidth}
  \centering
  \includegraphics[width=\textwidth]{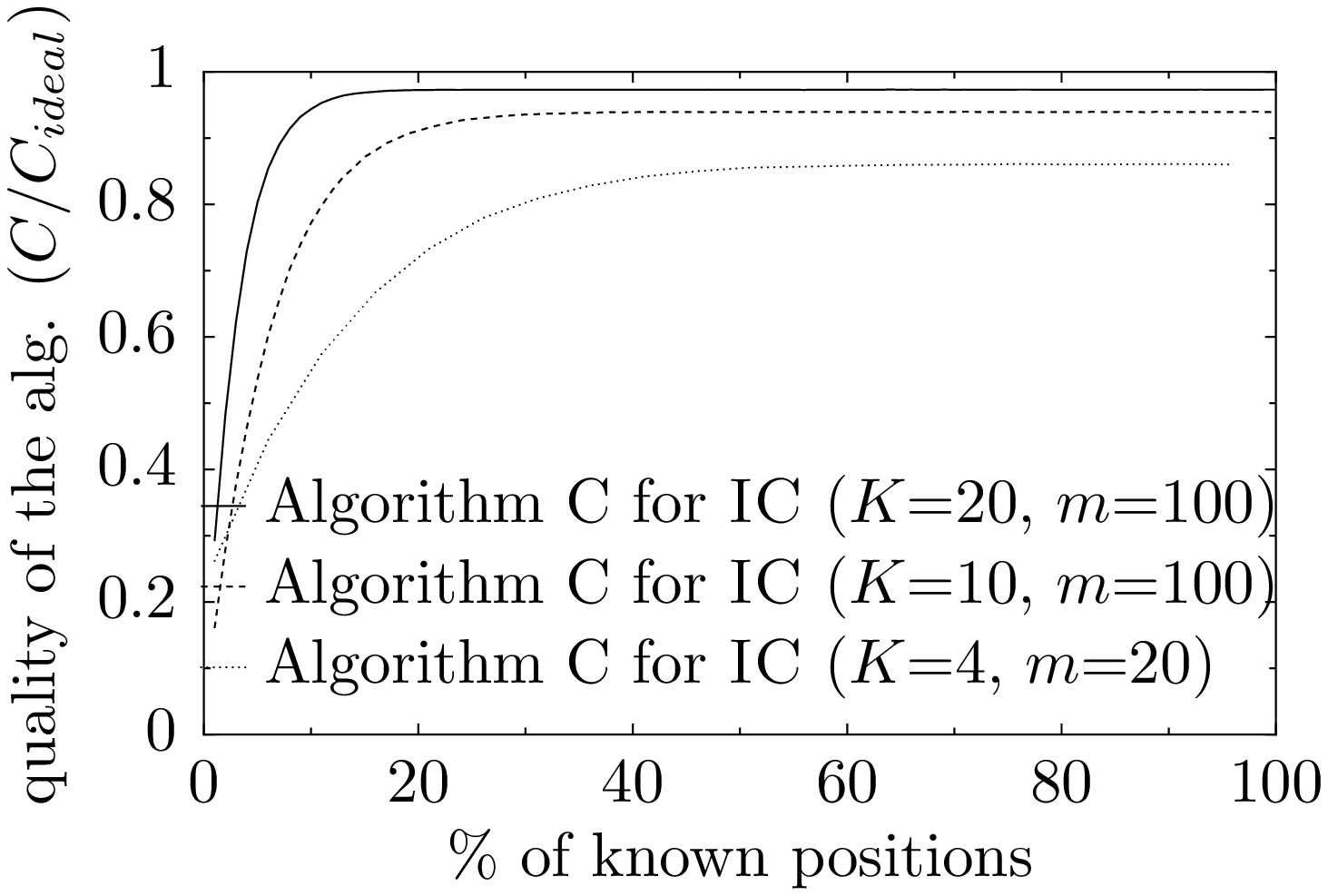}
\end{minipage}
  \vspace{0.2cm}

\begin{minipage}[h]{0.32\linewidth}
  \centering
  \includegraphics[width=\textwidth]{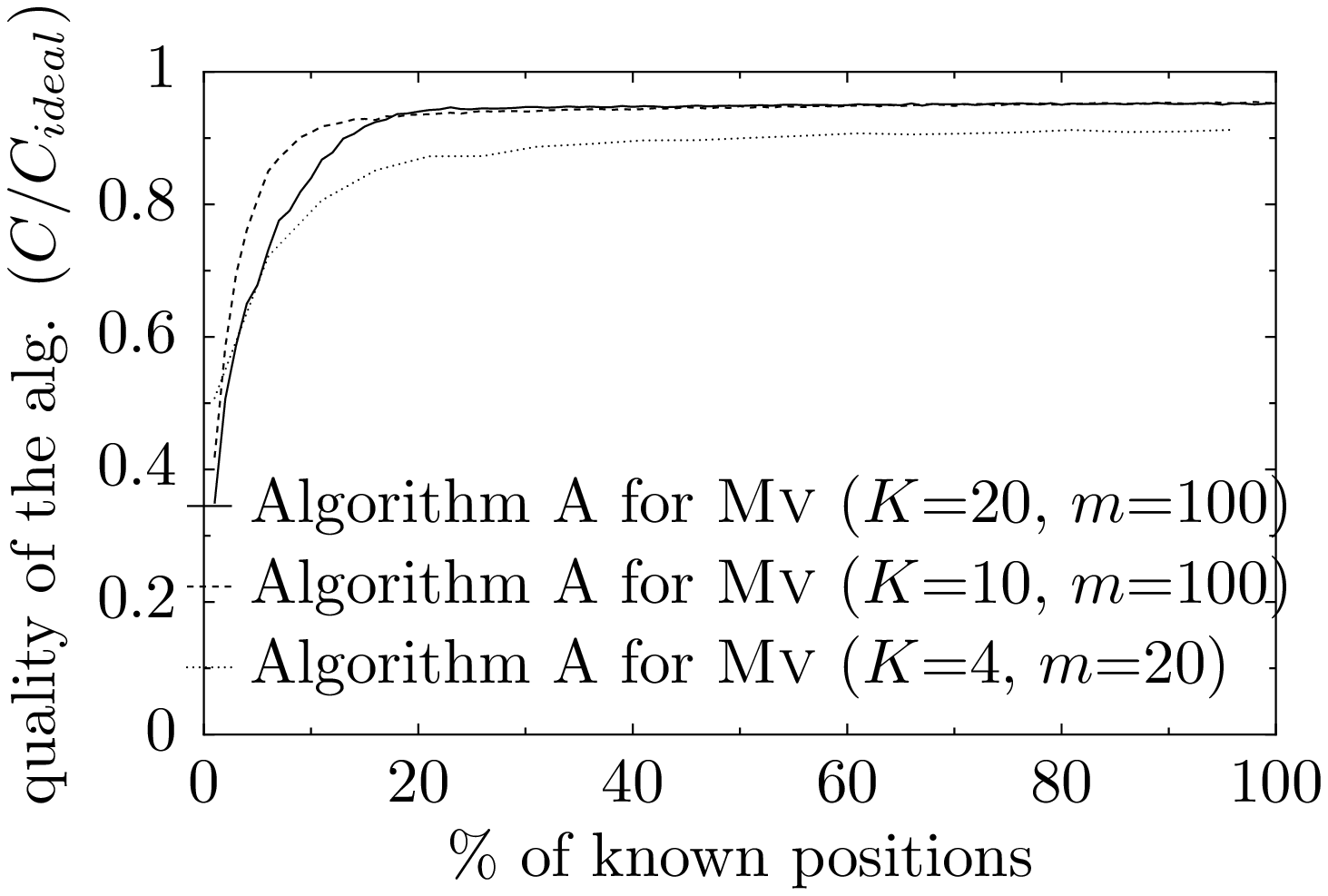}
\end{minipage}
\hspace{0.1cm}
\begin{minipage}[h]{0.32\linewidth}
  \centering
  \includegraphics[width=\textwidth]{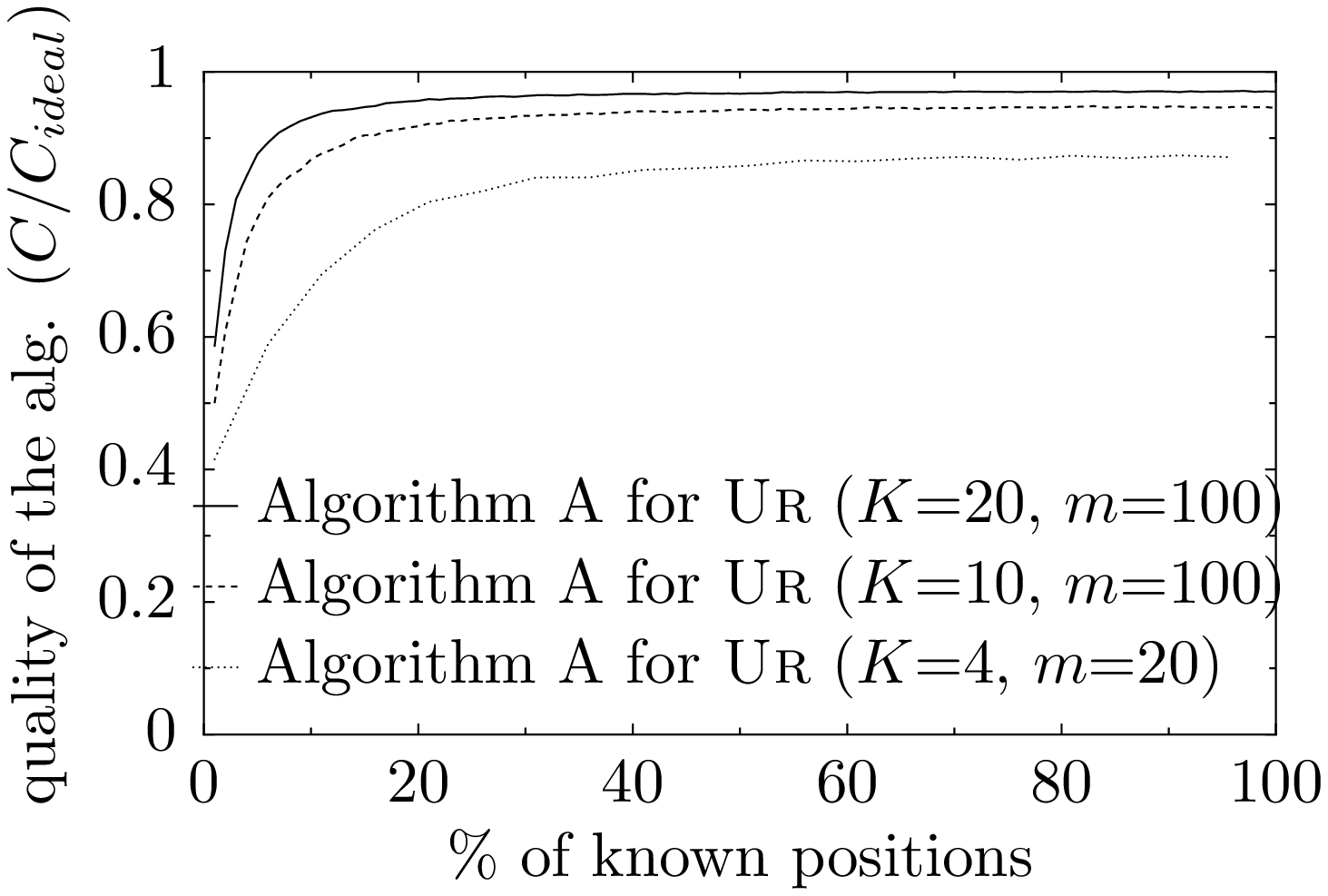}
\end{minipage}
\hspace{0.1cm}
\begin{minipage}[h]{0.32\linewidth}
  \centering
  \includegraphics[width=\textwidth]{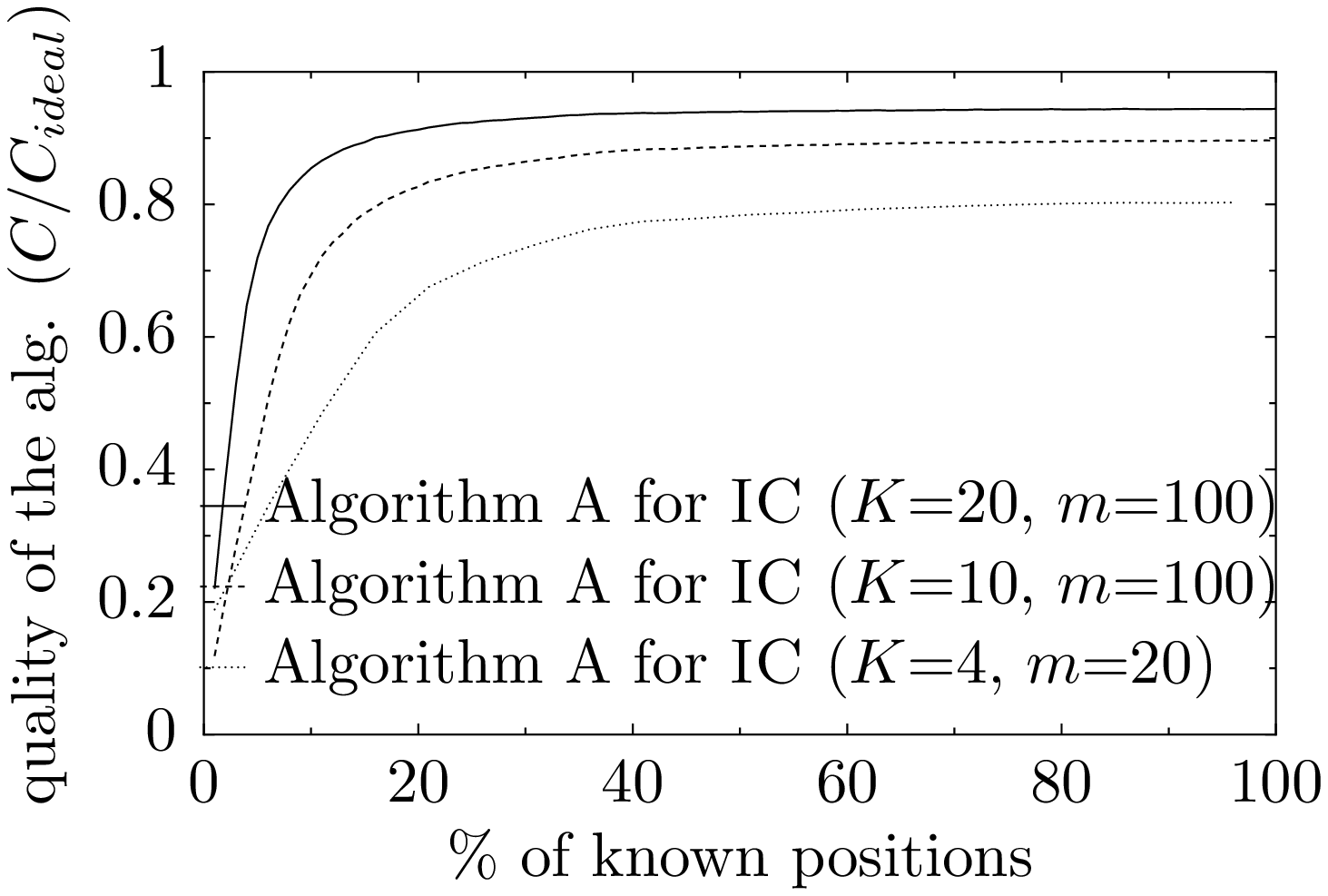}
\end{minipage}
  \vspace{0.2cm}

\begin{minipage}[h]{0.32\linewidth}
  \centering
  \includegraphics[width=\textwidth]{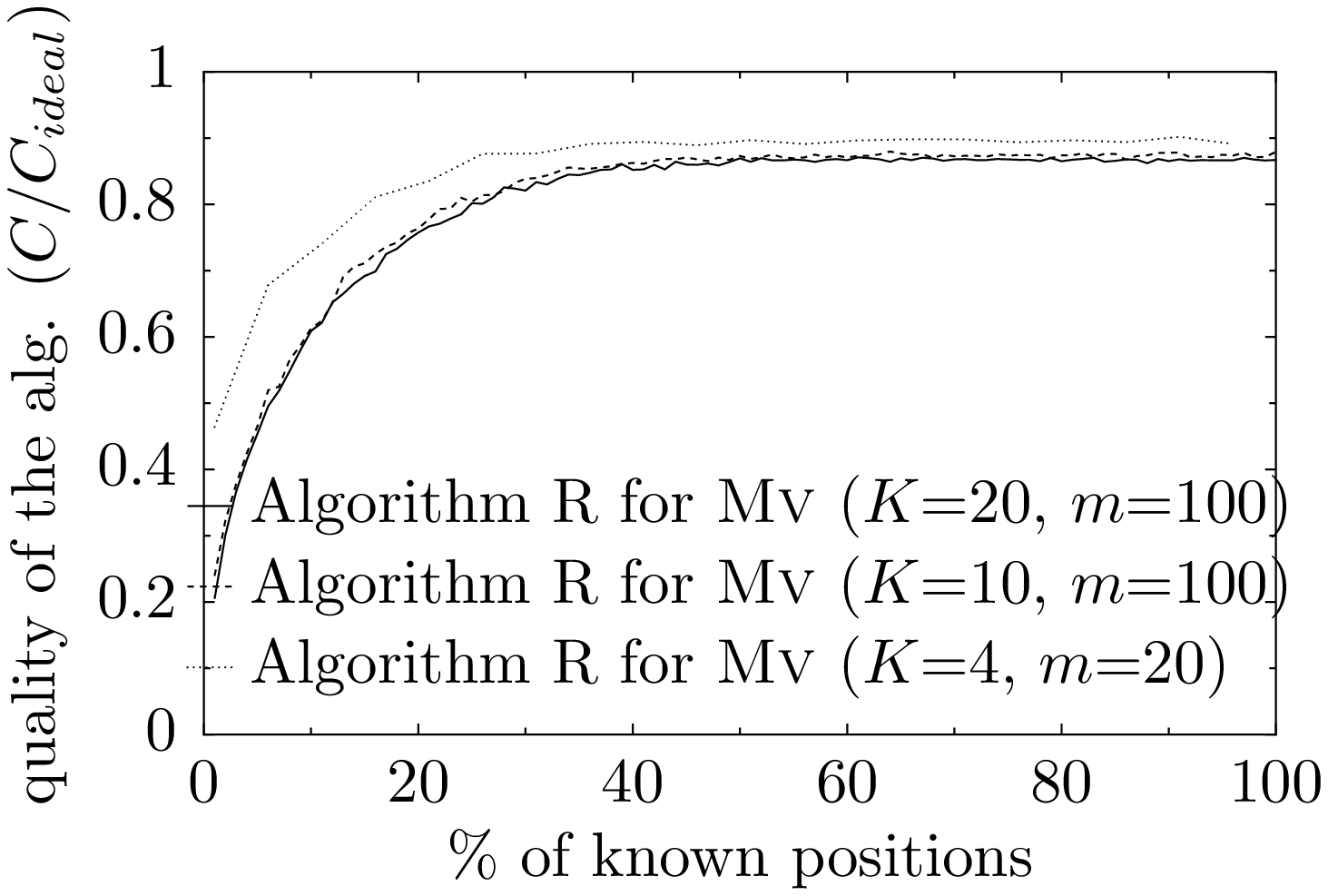}
\end{minipage}
\hspace{0.1cm}
\begin{minipage}[h]{0.32\linewidth}
  \centering
  \includegraphics[width=\textwidth]{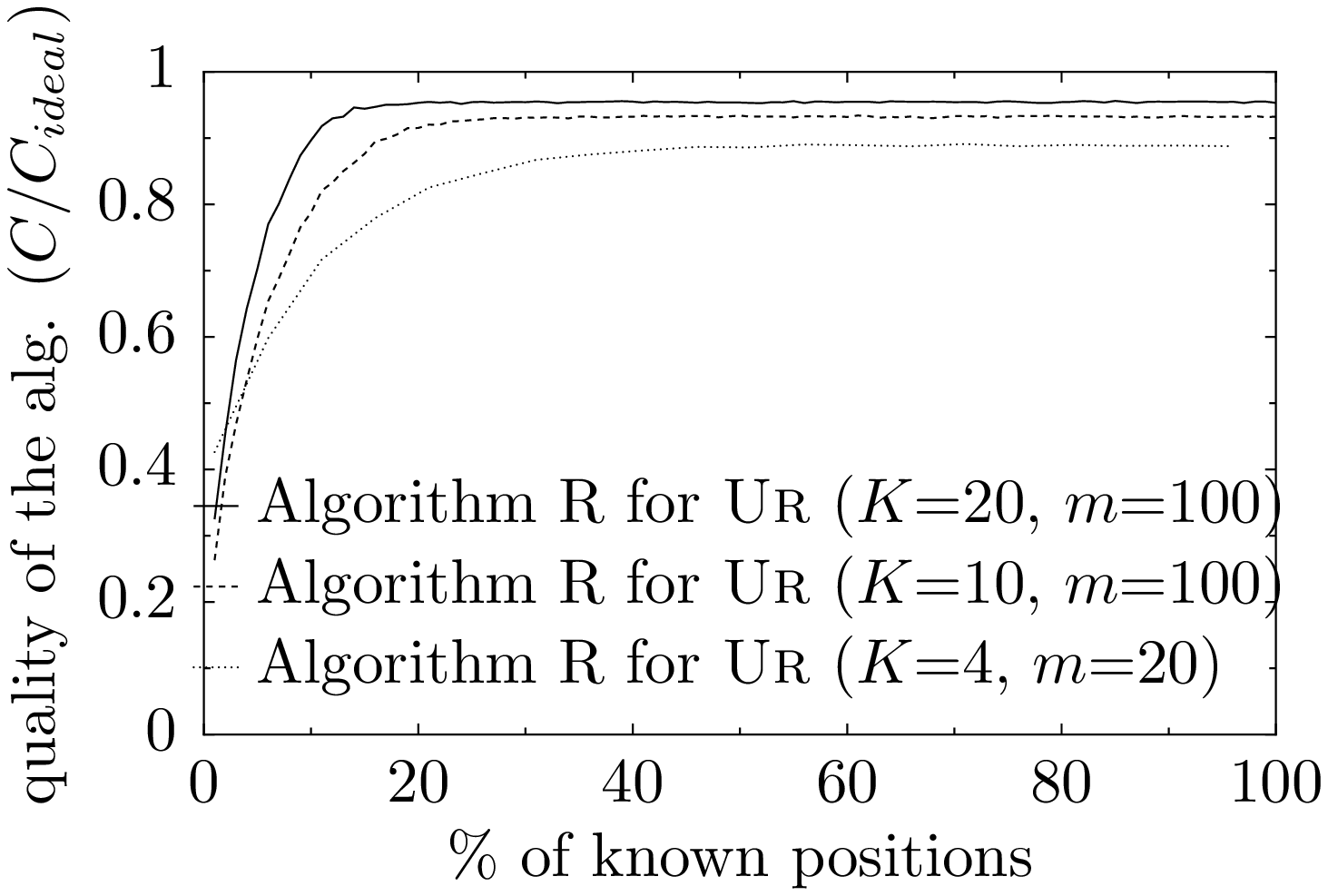}
\end{minipage}
\hspace{0.1cm}
\begin{minipage}[h]{0.32\linewidth}
  \centering
  \includegraphics[width=\textwidth]{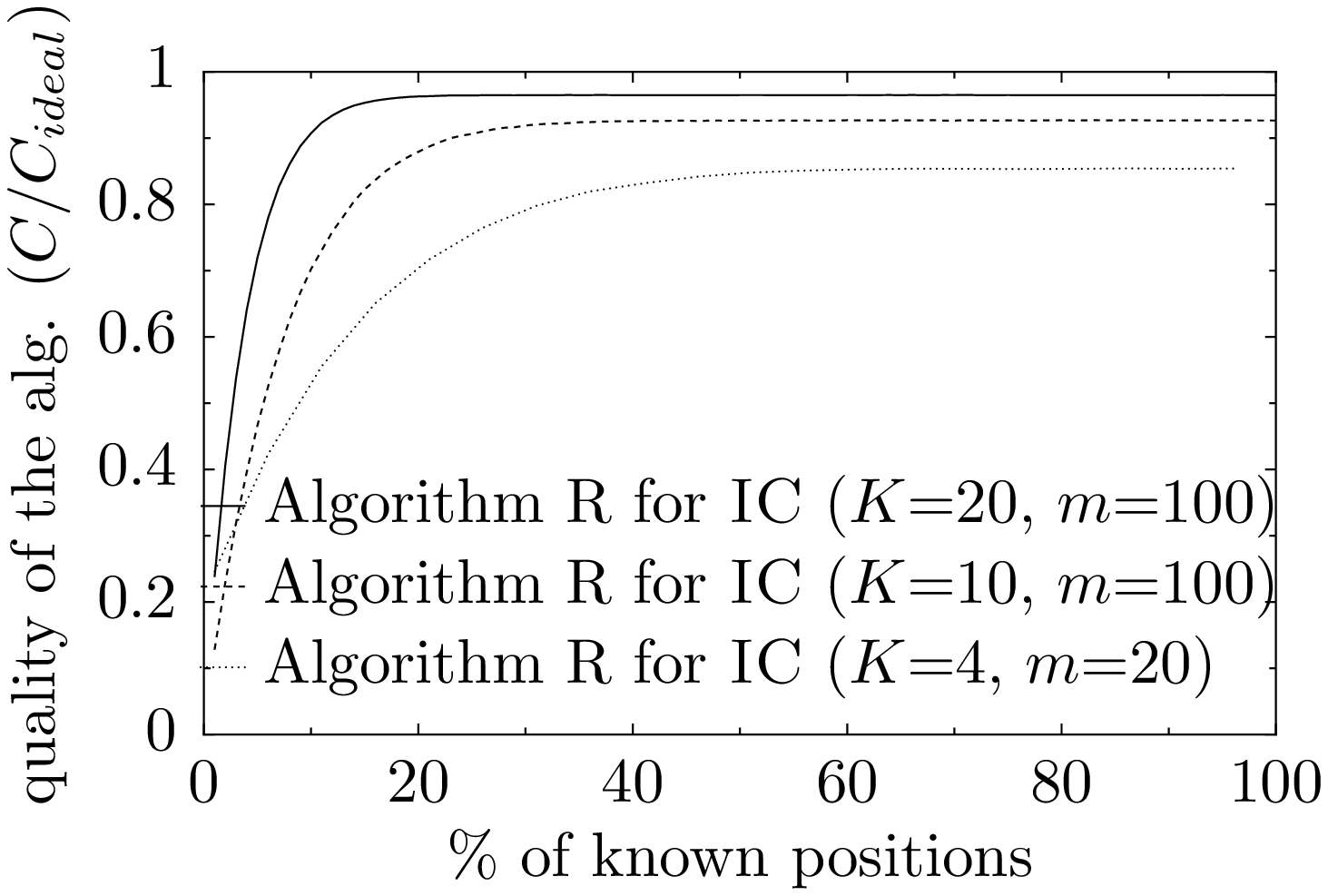}
\end{minipage}
  \caption{The relation between the percentage of known positions $P/m$ [\%] and the quality of the algorithm {$C/C_{\ideal}$} for Algorithms~C,~A,~and~R for Monroe's system. Each row of the plots describes one algorithm; each column describes one data set; ${n = 1000}$. (Results for the Mallows model are similar to those for the urn model and are omitted for clarity.)}
  \label{fig:changing_borda_monroe}
\end{figure*}

\begin{figure*}[t!h!]
\begin{minipage}[h]{0.32\linewidth}
  \centering
  \includegraphics[width=\textwidth]{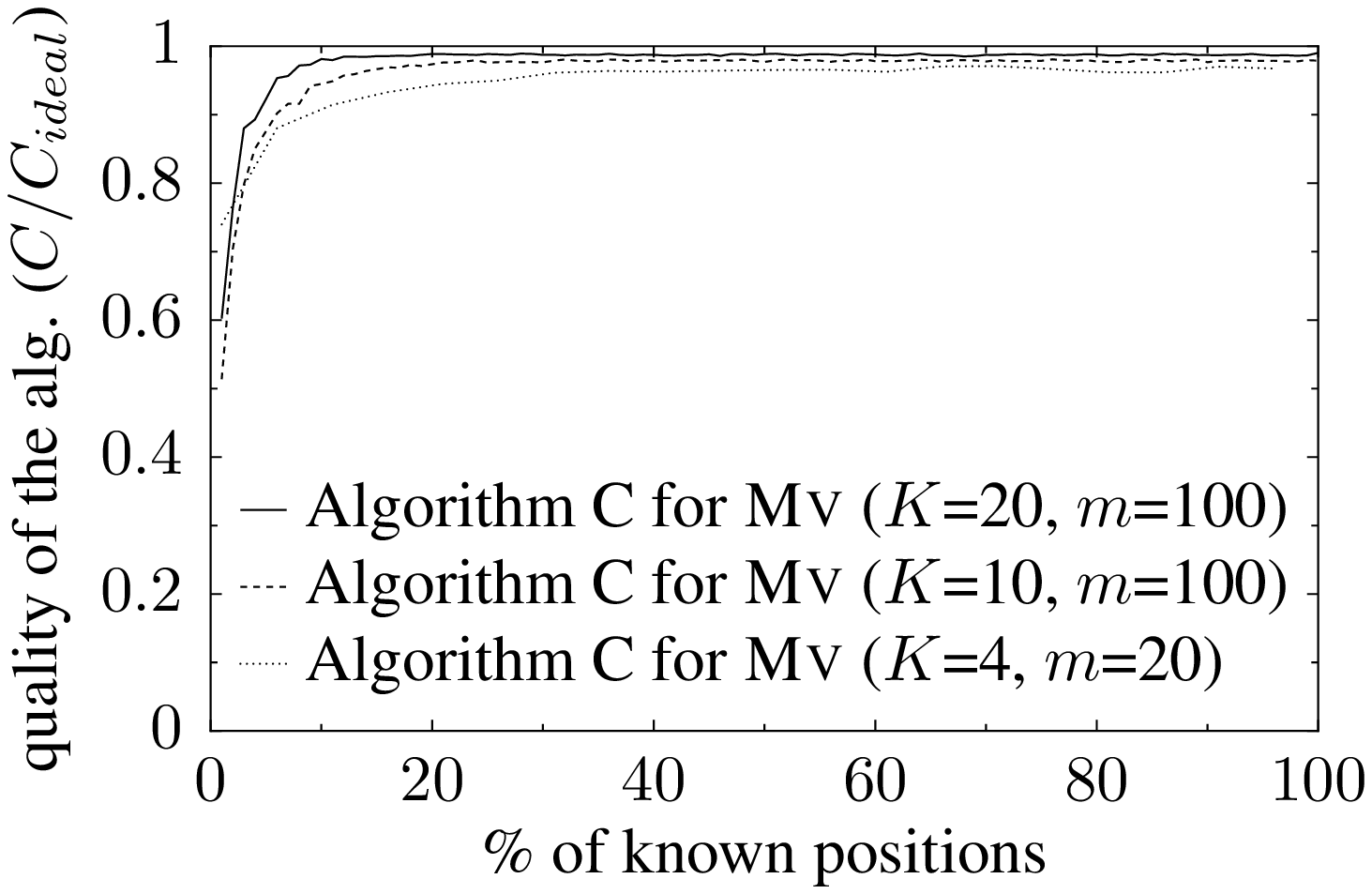}
\end{minipage}
\hspace{0.1cm}
\begin{minipage}[h]{0.32\linewidth}
  \centering
  \includegraphics[width=\textwidth]{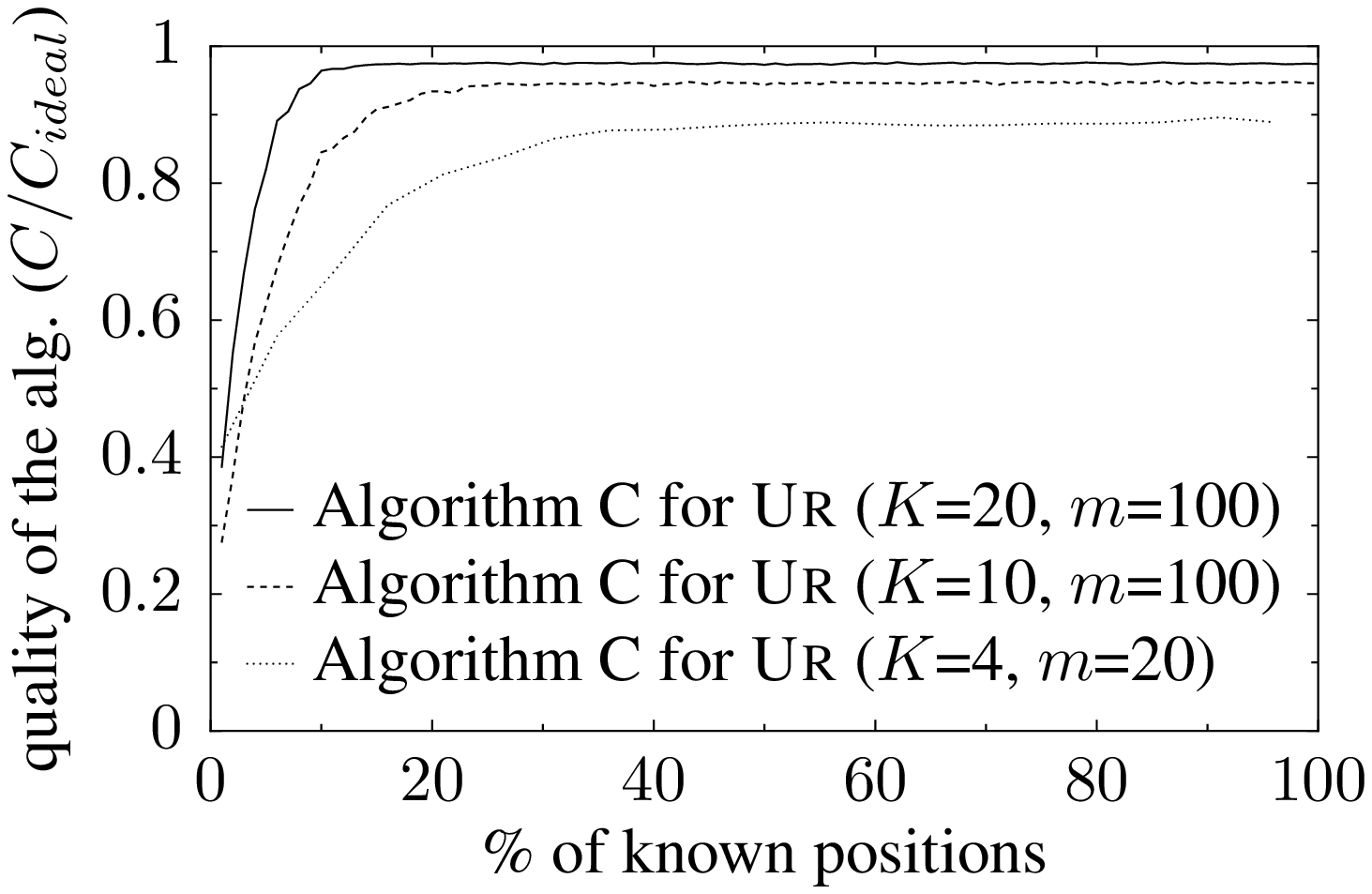}
\end{minipage}
\hspace{0.1cm}
\begin{minipage}[h]{0.32\linewidth}
  \centering
  \includegraphics[width=\textwidth]{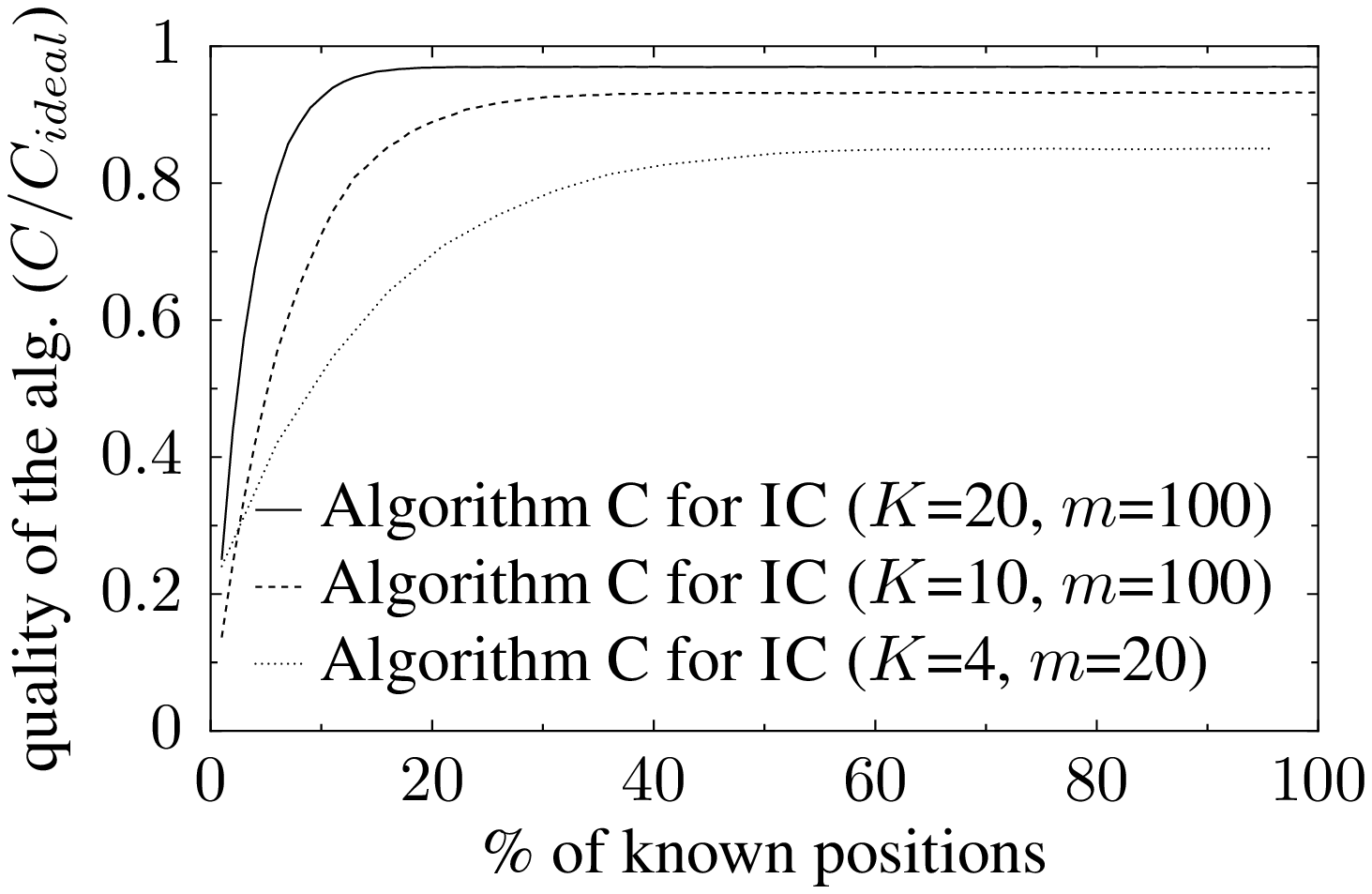}
\end{minipage}
  \vspace{0.2cm}

\begin{minipage}[h]{0.32\linewidth}
  \centering
  \includegraphics[width=\textwidth]{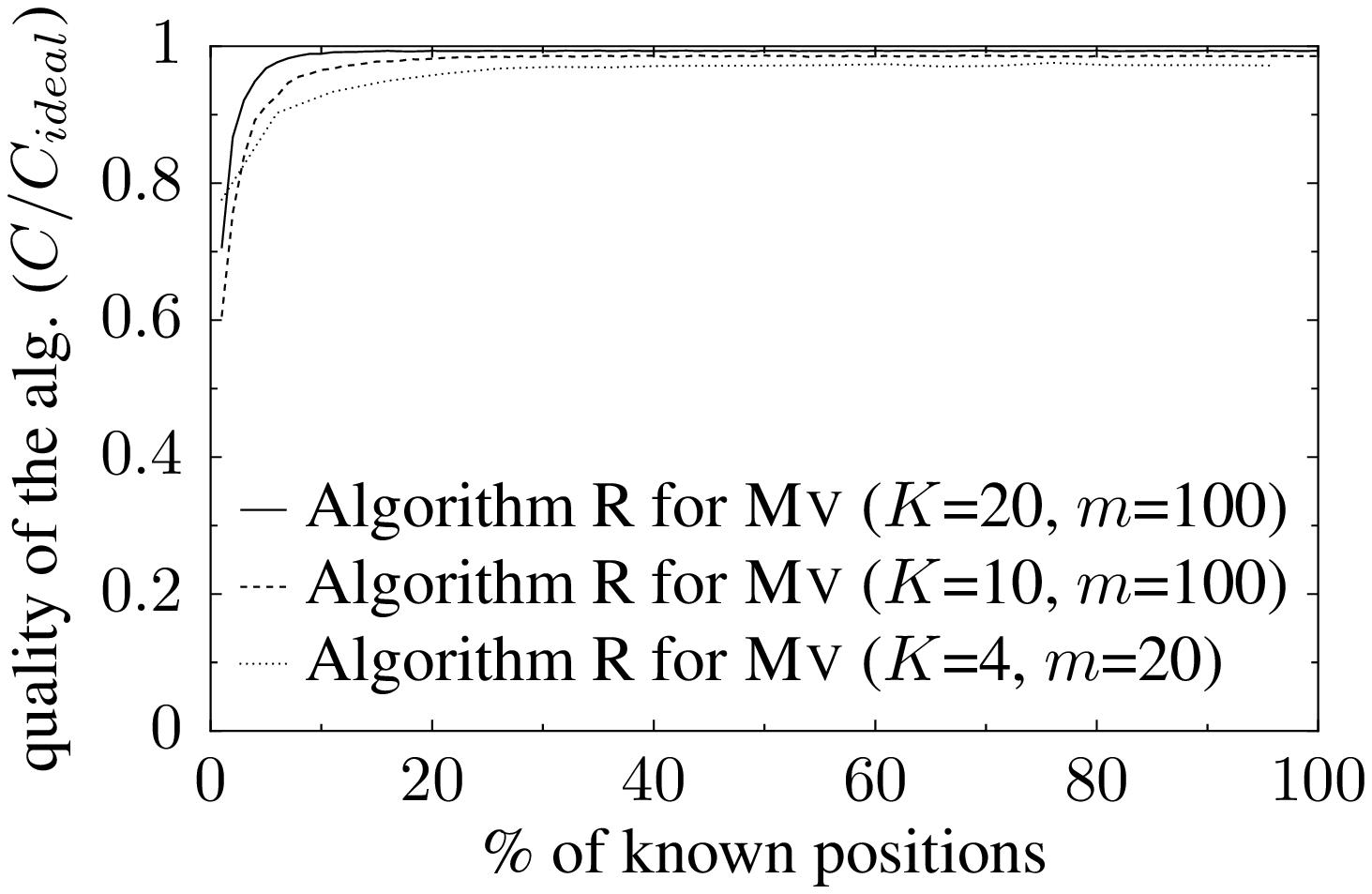}
\end{minipage}
\hspace{0.1cm}
\begin{minipage}[h]{0.32\linewidth}
  \centering
  \includegraphics[width=\textwidth]{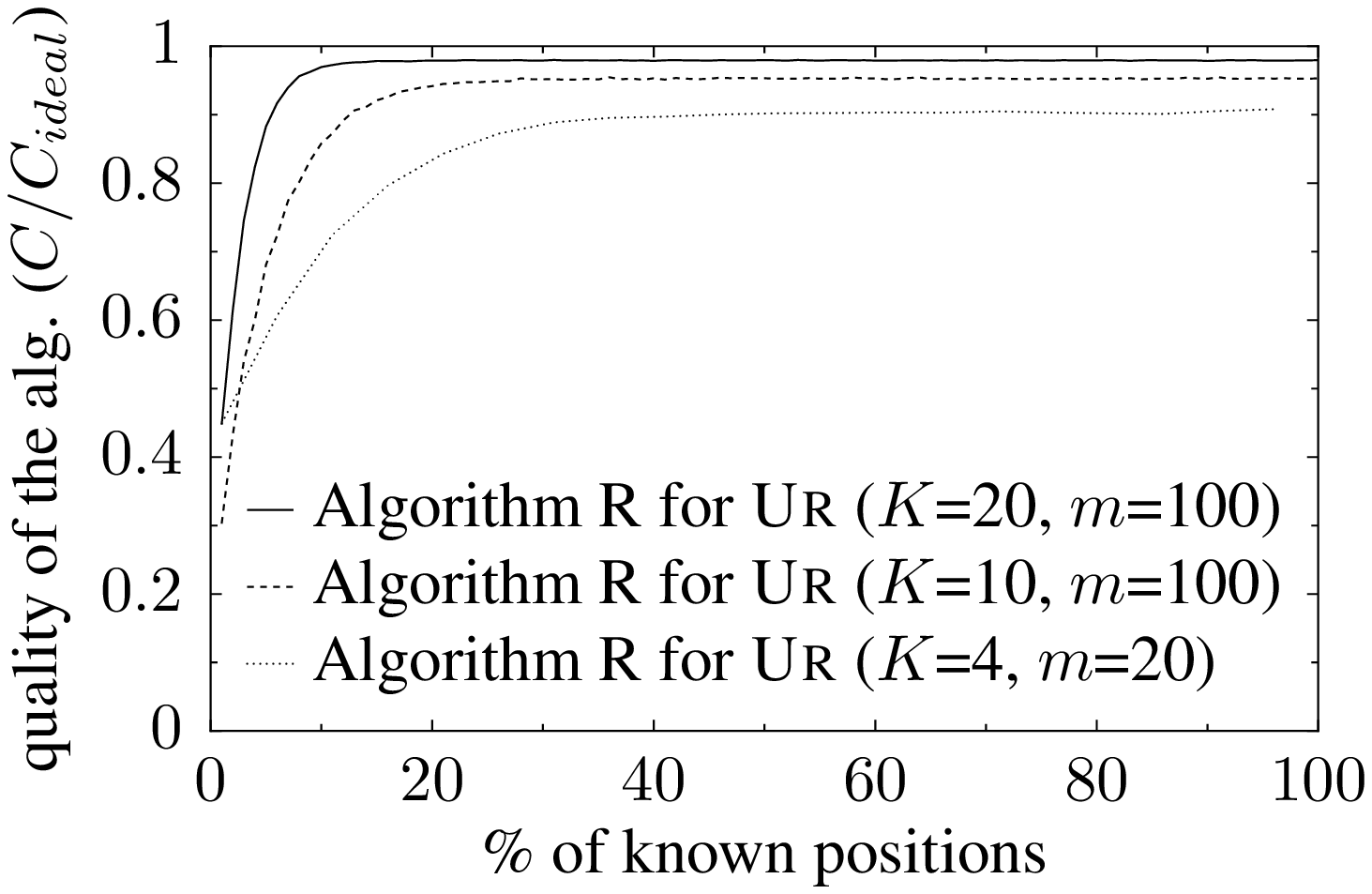}
\end{minipage}
\hspace{0.1cm}
\begin{minipage}[h]{0.32\linewidth}
  \centering
  \includegraphics[width=\textwidth]{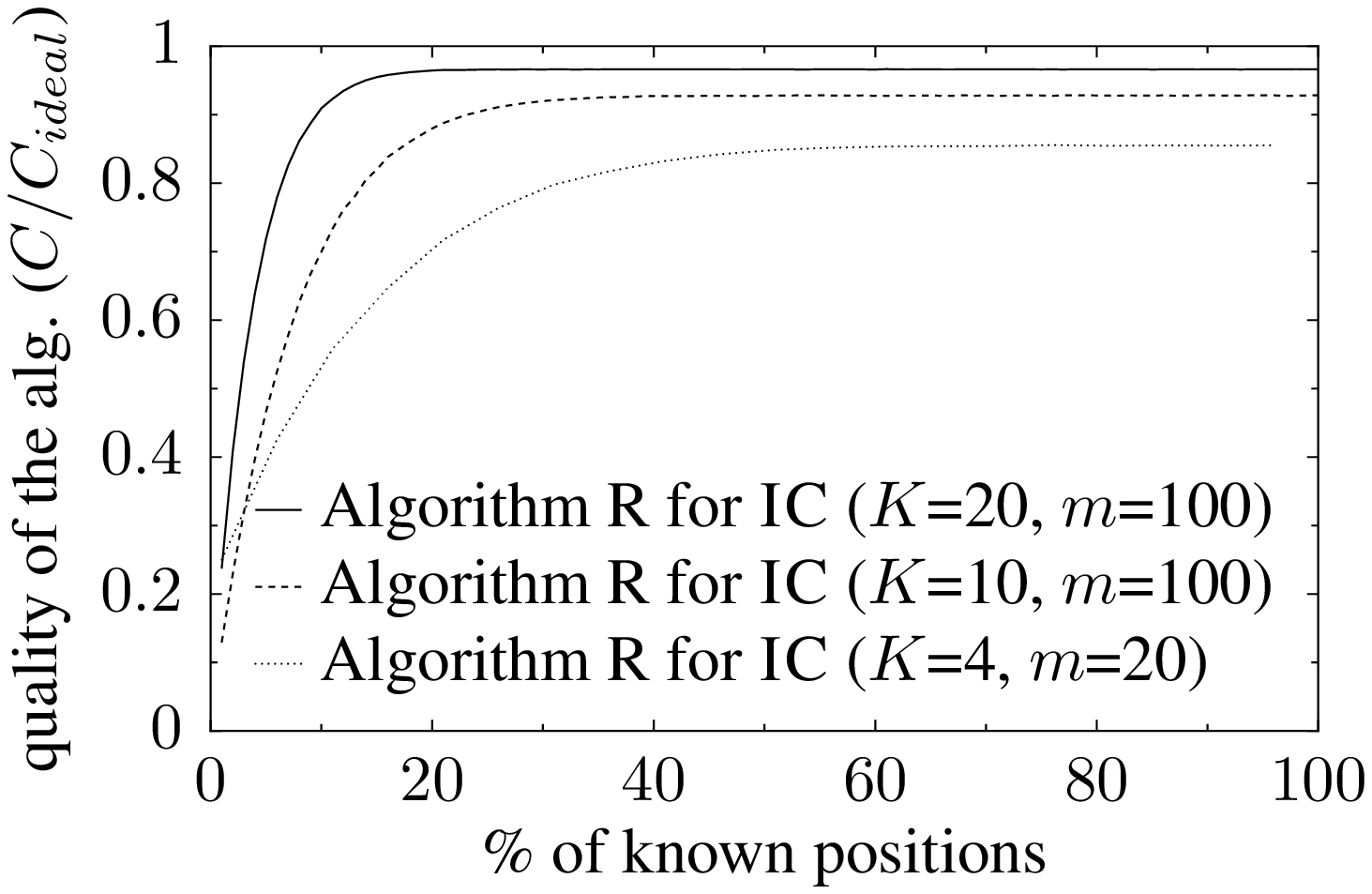}
\end{minipage}
  \caption{The relation between the percentage of known positions $P/m$ [\%] and the quality of the algorithm {$C/C_{\ideal}$} for Algorithms~C~and~R for the Chamberlin--Courant system. Each row of the plots describes one algorithm; each column describes one data set; ${n = 1000}$. (Results for the Mallows model are similar to those for the urn model and are omitted for clarity.)}
  \label{fig:changing_borda_cc}
\end{figure*}

The purpose of our third set of experiments was to see how our
algorithm behave in practical settings with truncated ballotrs.  We
conducted this part of evaluation on relatively large instances,
including $n=1000$ agents and up to $m=100$ alternatives. Thus, in
this set of experiments, we used the same sets of data as in the
previous subsection: the Netflix data set and the synthetic
distributions. Similarly, we evaluated the same algorithms:
Algorithm~A,~C,~and~R for the case of Monroe's system, and
Algorithm~C,~and~R for the case of the Chamberlin--Courant system.

For each data set and for each algorithm we run experiments for 3
independent settings with different values of the parameters
describing the elections: (1) $m=100$, $K=20$, (2) $m=100$, $K=10$,
and (3) $m=20$, $K=4$. For each setting we run the experiments for the
values of $P$ (the number of known positions) varying between 1 and
$m$.

For each algorithm, data set, setting and each value of $P$ we run 50
independent experiments in the following way. From a data set we
sampled a sub-profile of the appropriate size $n \times m$. We
truncated this profile to the $P$ first positions. We run the
algorithm for the truncated profile and calculated the quality ratio
${C}/{C_{\ideal}}$. When calculating ${C}/{C_{\ideal}}$ we assumed the
worst case scenario, i.e., that the satisfaction of the agent from an
alternative outside of his/her first $P$ positions is equal to 0. In
other words, we used the positional scoring function described by the
following vector: $\langle m-1, m-2, \dots, m-P, 0, \dots 0
\rangle$. Next, we averaged the values of ${C}/{C_{\ideal}}$ over all
50 experiments.

The relation between the percentage of the known positions in the
preference profile and the average quality of the algorithm for the
Monroe and Chamberlin--Courant systems are plotted in
Figures~\ref{fig:changing_borda_monroe}~and~\ref{fig:changing_borda_cc},
respectively. We omit the plots for Mallow's model, as in this
case we obtained almost identical results as for the Urn model. We
have the following conclusions.
\begin{enumerate}
\item All the algorithms require only small number of the top
  positions to achieve their best quality. Here, the deterministic
  algorithms are superior.
\item The small elections with synthetic distributions appear to be
  the worst case scenario---in such case we require the knowledge of
  about 40\% of the top positions to obtain the highest approximation
  ratios of the algorithms. In the case of the NetFlix data set, even
  on small instances the deterministic algorithms require only about
  8\% of the top positions to get their best quality (however the
  quality is already high for 3-5\% of the top positions). For the
  larger number of the alternatives, the algorithms do not require
  more than 3\% of the top positions to reach their top results.
\item Algorithm C does not only give the best quality but it is also
  most immune to the lack of knowledge. These results are more evident
  for the case of the Monroe system.
\end{enumerate}

\subsection{Running time}

\begin{figure}[b!]
  \centering
  \includegraphics[width=0.6\textwidth]{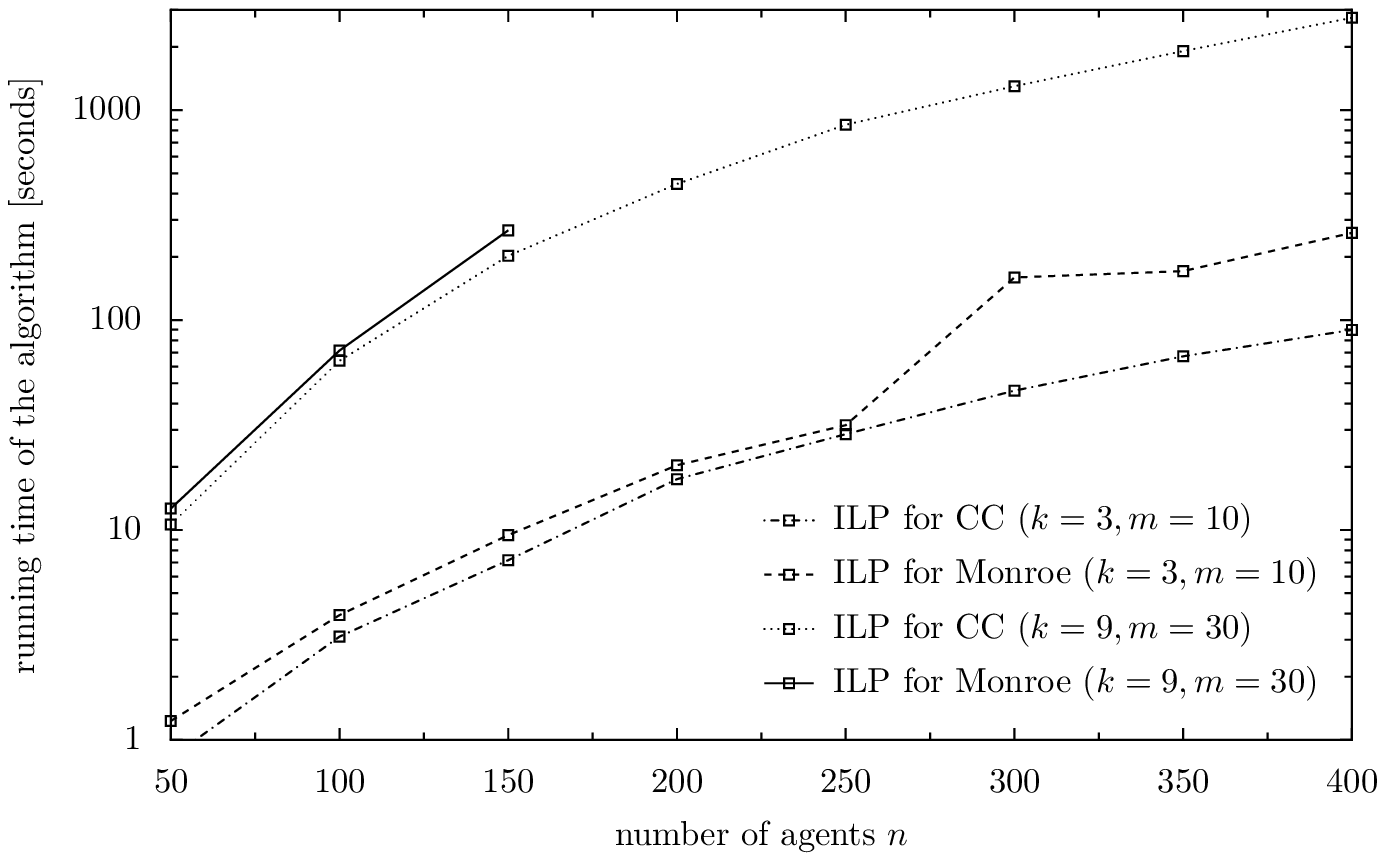}
  \caption{The running time of the standard ILP solver for the
    Monroe and for the Chamberlin--Courant systems.  For
    Monroe's system, for {$K=9, m = 30$}, and for
    {$n \geq 200$} none of the single algorithm execution
    finished within 1 day.}
  \label{fig:ilp_runtime}
\end{figure}

\begin{table}[th!]
\begin{center}
\begin{tabular}{|c|c|c|c|c||c|c|c|}
\cline{3-8}
\multicolumn{2}{c|}{} & \multicolumn{3}{|c||}{$m=10$, $K=3$} & \multicolumn{3}{|c|}{$m=10$, $K=6$} \\
\cline{2-8}
\multicolumn{1}{c|}{} & $n=$ & 2000 & 6000 & 10000 & 2000 & 6000 & 10000\\
\cline{1-8}
\multirow{5}{*}{\begin{sideways} Monroe \end{sideways}}
& A & $0.01$ & $0.03$ & $0.05$  & $0.01$ & $0.04$ & $0.07$ \\
& B & $0.08$ & $0.9$ & $2.3$  & $0.2$ & $1.4$ & $3.6$ \\
& C & $1.1$ & $8$ & $22$  & $2.1$ & $16$ & $37$ \\
& GM & $0.8$ & $7.3$ & $20$ & $1.9$ & $13$ & $52$ \\
& R & $7.6$ & $50$ & $180$  & $6.5$ & $52$ & $140$ \\
\cline{1-8}
\multirow{4}{*}{\begin{sideways} CC \end{sideways}}
& C & $0.02$ & $0.07$ & $0.12$ & $0.05$ & $0.14$ & $0.26$ \\
& GM & $0.003$ & $0.009$ & $0.015$ & $0.003$ & $0.01$ & $0.018$ \\
& P & $0.009$ & $0.032$ & $0.05$ & $0.008$ & $0.02$ & $0.05$ \\
& R & $0.014$ & $0.04$ & $0.065$ & $0.02$ & $0.06$ & $0.11$ \\
\cline{1-8}
\multicolumn{2}{c}{} \\

\cline{3-8}
\multicolumn{2}{c|}{} & \multicolumn{3}{|c||}{$m=100$, $K=30$} & \multicolumn{3}{|c|}{$m=100$, $K=60$} \\
\cline{2-8}
\multicolumn{1}{c|}{} & $n=$ & 2000 & 6000 & 10000 & 2000 & 6000 & 10000\\
\cline{1-8}
\multirow{5}{*}{\begin{sideways} Monroe \end{sideways}}
& A & $0.5$ & $1.6$ & $2.8$  & $0.9$ & $2.8$ & $4.9$ \\
& B & $0.8$ & $4$ & $9.5$  & $1.7$ & $8$ & $18$ \\
& C & $38$ & $140$ & $299$  & $64$ & $221$ & $419$ \\
& GM & $343$ & $2172$ & $5313$ & $929$ & $5107$ & $13420$ \\
& R & $41$ & $329$ & $830$  & $88$ & $608$ & $1661$ \\
\cline{1-8}
\multirow{4}{*}{\begin{sideways} CC \end{sideways}}
& C & $4.3$ & $11$ & $19$ & $7.5$ & $19$ & $31$ \\
& GM & $0.06$ & $0.2$ & $0.4$ & $0.09$ & $0.3$ & $0.7$ \\
& P & $0.03$ & $0.1$ & $0.26$ & $0.03$ & $0.1$ & $0.2$ \\
& R & $0.06$ & $0.24$ & $0.45$ & $0.1$ & $0.4$ & $0.8$ \\
\cline{1-8}
\end{tabular}
\end{center}

\caption{Example running times of the algorithms [in seconds].}
\label{table:runningTimes}
\end{table}

In our final set of experiments, we have measured running times of our
algorithms on the data set \textsc{Mv}. We have used a machine with
Intel Pentium Dual T2310 1.46GHz processor and 1.5GB of RAM.  In
Figure~\ref{fig:ilp_runtime} we show the running times of the GLPK ILP
solver for the Monroe and for Chamberlin--Courant rules. These running
times are already large for small instances and they are increasing
exponentially with the number of voters. For the Monroe rule, even for
$K=9, m = 30, n=100$ some of the experiments timed out after 1 hour,
and for $K=9, m = 30, n=200$ none of the experiments finished within
one day.  Thus we conclude that the real application of the ILP-based
algorithm is very limited. 

Example running times of the other algorithms for some combinations of
$n$, $m$, and $K$ are presented in Table~\ref{table:runningTimes}.
For the case of CC, essentially all the algorithms are very fast and
the quality of computed solutions is the main criterion in choosing
among them. For the case of Monroe, the situation is more complicated.
While for small elections all the algorithms are practical, for
elections with thousands of voters, using Algorithm~GM becomes
problematic. Indeed, even Algorithm~C can be seen as a bit too slow if
one expects immediate results. On the other hand, Algorithms~A and~B
seem perfectly practical and, as we have seen in the previous
experiments, give high-quality results.

\section{Summary}
\label{sec:conclusions}

We have defined a certain resource allocation problem and have shown
that it generalizes the problem of finding winners for the  multiwinner voting rules of Monroe and of
Chamberlin and Courant. Since it is known that the winners for these voting rules are
hard to compute~\cite{complexityProportionalRepr,budgetSocialChoice,fullyProportionalRepr,sko-yu-fal-elk:c:single-crossing-monroe-cc,sko-fal:t:max-cover},
we focused on finding approximate solutions. We have shown that if we
try to optimize agents' dissatisfaction, then our problems are hard to
approximate up to any constant factor. The same holds for the case
where we focus on the satisfaction of the least satisfied agent. However, for
the case of optimizing total satisfaction, we suggest good
approximation algorithms.  
In particular, for the Monroe system we suggest a randomized
algorithm that for the Borda score achieves an approximation ratio
arbitrarily close to $0.715$ (and much better in many real-life
settings), and ($1 - \frac{1}{e}$)-approximation algorithm for
arbitrary positional scoring function.  For the Chamberlin-Courant
system, we have shown a polynomial-time approximation scheme (PTAS).

\begin{table}[t!]
\centering
\footnotesize
\begin{tabular}{p{0.2cm}|p{1.6cm}|p{1.6cm}|p{1.6cm}|p{1.6cm}|p{1.6cm}|p{2.0cm}|}
\cline{2-7}
& \multicolumn{2}{|c|}{Monroe} & \multicolumn{2}{|c|}{Chamberlin and Courant}  &\multicolumn{2}{|c|}{General model}\\ \cline{2-7}
& dissat. & satisfaction & dissat. & satisfaction & dissat. & satisfaction \\ \cline{1-7}
\multicolumn{1}{|c|}{\begin{rotate}{270}util.\end{rotate}}
& Inapprox. $\newline$ Theorem~\ref{theorem:noApprox1}
& Good approx.
& Inapprox. $\newline$ Theorem~\ref{theorem:noApprox3}
& Good approx.
& Inapprox. $\newline$ Theorem~\ref{theorem:noApprox1} Theorem~\ref{theorem:noApprox3}
& Open problem \\ \cline{1-7}
\multicolumn{1}{|c|}{\begin{rotate}{270}egal.\end{rotate}}
& Inapprox. $\newline$ Theorem~\ref{theorem:noApprox2}
& Inapprox. $\newline$ Theorem~\ref{theorem:noApprox5}
& Inapprox. $\newline$ Theorem~\ref{theorem:noApprox4}
& Inapprox. $\newline$ Theorem~\ref{theorem:noApprox6}
& Inapprox. $\newline$ Theorem~\ref{theorem:noApprox2} Theorem~\ref{theorem:noApprox4}
& Inapprox. $\newline$ Theorem~\ref{theorem:noApprox5} Theorem~\ref{theorem:noApprox6}
\\ \cline{1-7}
\end{tabular}
\caption{Summary of approximability results for the Monroe and Chamberlin-Courant multiwinner voting systems and for the general resource allocation problem.}
\label{table:summary}
\end{table}

\begin{table}[t!]
\small
\begin{center}
  \setlength{\tabcolsep}{3pt}
  \begin{tabular}{|p{0.45cm}|c|p{4.2cm}|c|c|}
  \hline
  & Algorithm & Approximation & Runtime & Reference \\
  \hline
  \multicolumn{1}{|c|}{\begin{rotate}{270}\vspace{0.2cm}\hspace{0.55cm}Monroe\end{rotate}}
  & A    & $1 - \frac{K-1}{2(m-1)} - \frac{H_K}{K}$ & $Kmn$ & Lemma~\ref{lemma:greedy}  \\
  & B    & as in Algorithm A & $Kmn$$+$$O(\Phi^{S})$ & Lemma~\ref{lemma:greedy}\\
  & C    & as in Algorithm A & $dKmn$$+$$dO(\Phi^{S})$ & Lemma~\ref{lemma:greedy}\\
  & GM   & as in Alg. A for Borda PSF; $1-\frac{1}{e}$ for others   & $KmO(\Phi^{S})$ & Theorem~\ref{thm:gmMonroe} \\
  & R    & $\frac{1}{2}(1 + \frac{K}{m} - \frac{K^2m - K^{3}}{m^3-m^2})$ & $\frac{|\log (1 -\lambda)|}{K\epsilon^2}O(\Phi^{S})$ & Lemma~\ref{lemma:randMonroe} \\
  & AR   & $0.715$  & $\max(\textrm{A}, \textrm{R})$ & Theorem~\ref{thm:combMonroe} \\
  \hline
  \multicolumn{1}{|c|}{\begin{rotate}{270}\hspace{0.55cm}CC\end{rotate}}
  & \multicolumn{3}{|c|}{PTAS} & Theorem~\ref{theorem:ptas} \\
  \cline{2-5}
  & P    & $1 - \frac{2\w(K)}{K}$ & $nm\w(K)$ & Lemma~\ref{lemma:greedyCC} \\
  & GM   & $1-\frac{1}{e}$ & $Kmn$ & Lu and Boutilier~\cite{budgetSocialChoice} \\
  & C    & as in Algorithm GM & $dKm(n$$+$$\log dm)$ & Lu and Boutilier~\cite{budgetSocialChoice}  \\
  & R    & $(1-\frac{1}{K+1})(1 + \frac{1}{m})$ & $\frac{|\log (1 -\lambda)|}{\epsilon^2}n$ & Oren~\cite{ore:p:cc} \\
  \hline  
  \end{tabular}
  \caption{\label{tab:algs}A summary of the algorithms 
studied in this paper. 
The top of the table regards algorithms for Monroe's rule
    and the bottom for the Chamberlin--Courant rule. In column
    ``Approximation'' we give currently known approximation ratio for the
    algorithm under Borda PSF, on profiles with $m$
    candidates and where the goal is to select a committee of size
    $K$. Here, $O(\Phi^{S}) = O(n^2(K +
      \mathrm{log}n))$ is the complexity of finding a partial
    representation function with the algorithm of Betzler et
    al.~\cite{fullyProportionalRepr}. $\w(\cdot)$ denotes Lambert's W-Function.}
\end{center}
\vspace{-0.75cm}
\end{table}

In Table~\ref{table:summary} we present the summary of our
(in)approximability results.  In Table~\ref{tab:algs} we present
specific results regarding our approximation algorithms for the
utilitarian satisfaction-based framework.  In particular, the table
clearly shows that for the case of Monroe, Algorithms B and C are not
much slower than Algorithm A but offer a chance of improved
peformance. Algorithm GM is intuitively even more appealing, but
achieves this at the cost of high time complexity. For the case of
Chamberlin-Courant rule, theoretical results suggest using
Algorithm P (however, see below).

We have provided experimental evaluation of the algorithms for
computing the winner sets both for the Monroe and Chamberlin--Courant rules . While finding solutions for these rules is computationally
hard in the worst case, it turned out that in practice we can obtain
very high quality solutions using very simple algorithms. Indeed, both
for the Monroe and Chamberlin-Courant rules we recommend
using Algorithm~C (or Algorithm~A on very large Monroe elections).
Our experimental evaluation confirms that the algorithms work very
well in case of truncated ballots.  We believe that our results mean
that (approximations of) the Monroe and Chamberlin--Courant rules can
be used in practice.

Our work leads to a number of further research directions. First, it would be
very interesting to find a better upper bound on the quality of 
solutions for the (satisfaction-based) Monroe and Chamberlin--Courant
systems (with Borda PSF) than the simple $n(m-1)$ bound that we use
(where $n$ is the number of voters and $m$ is the number of
candidates). We use a different approach in our randomized algorithm,
but it would be much more interesting to find a deterministic
algorithm that beats the approximation ratios of our algorithms. One
of the ways of seeking such a bound would be to consider Monroe's rule
with ``exponential'' Borda PSF, that is, with PSF of the form, e.g.,
$(2^{m-1}, 2^{m-2}, \ldots, 1)$. For such PSF our  approach in the proof of
Lemma~\ref{lemma:greedy} would not give satisfactory results and so one would
be forced to seek different attacks.  In a similar vein, it would be
interesting to find out if there is a PTAS for Monroe's system.

In our work, we have focused on PSFs that are strictly
increasing/decreasing.  It would also be interesting to study PSFs
which increase/decrease but not strictly, that is allowing some equalities. We have started to work in this direction by
considering the so-called $t$-approval PSF's $\alpha_t$, which (in the
satisfaction-based variant) are defined as follows: $\alpha_t(i) = 1$
if $i \leq t$ and otherwise $\alpha_t(i) = 0$. Results for this case
for the Chamberlin--Courant rule are presented in the paper of Skowron
and Faliszewski~\cite{sko-fal:t:max-cover}.

On a more practical side, it would be interesting to develop our study
of truncated ballots. Our results show that we can obtain very high
approximation ratios even when voters rank only relatively few of
their top candidates. For example, to achieve 90\% approximation ratio
for the satisfaction-based Monroe system in Polish parliamentary
election ($K=460, m=6000$), each voter should rank about $8.7\%$ of
his or her most-preferred candidates. However, this is still over
$500$ candidates. It is unrealistic to expect that the voters  would be
willing to rank this many candidates. Thus, how should one organize
Monroe-based elections in practice, to balance the amount of effort
required from the voters and the quality of the results?

Finally, going back to our general resource allocation problem, we
note that we do not have any positive results for it (the negative
results, of course, carry over from the more restrictive settings). Is
it possible to obtain some good approximation algorithm for the
 resource allocation problem (in the utilitarian satisfaction-based setting) in full generality?

\medskip
\noindent\textbf{Acknowledgements} 
This paper is based on two extended abstracts, presented at IJCAI-2013
and AAMAS-2012.  We would like to thank the reviewers from these two
conferences for very useful feedback.  The authors were supported in
part by Poland’s National Science Center grants
UMO-2012/06/M/ST1/00358, DEC-2011/03/B/ST6/01393, and by AGH
University of Science and Technology grant 11.11.120.865. Piotr
Skowron was also supported by EU’s Human Capital Program "National PhD
Programme in Mathematical Sciences" carried out at the University of
Warsaw

.

\bibliographystyle{plain}
\bibliography{main}

\end{document}